%% file: arxiv_update.tex
\documentclass[twoside]{article}

\usepackage[accepted]{arxiv}

\usepackage[utf8]{inputenc} 
\usepackage[T1]{fontenc}    
\usepackage{hyperref}       
\usepackage{url}            
\usepackage{booktabs}       
\usepackage{amsfonts}       
\usepackage{nicefrac}       
\usepackage[protrusion=true,expansion=false,activate=true,final,kerning=true,spacing=true]{microtype}

\usepackage{xcolor}         
\usepackage[english]{babel}
\usepackage{natbib}
\usepackage[utf8]{inputenc}
\usepackage[T1]{fontenc}
\usepackage{subcaption}
\usepackage{booktabs}
\usepackage{multirow}
\usepackage{url}
\usepackage{soul}
\usepackage{tikz}
\usetikzlibrary{calc}
\usetikzlibrary{decorations.pathmorphing}
\usetikzlibrary{positioning}
\usetikzlibrary{shapes}
\usetikzlibrary{decorations.pathreplacing}
\usetikzlibrary{arrows.meta,backgrounds,automata}
\usetikzlibrary{positioning,shapes,matrix,calc}
\tikzstyle{node}=[circle, draw, fill=gray!80!white,thick,scale=1.2]
\tikzstyle{edge}=[draw=black, thick,-]
\usetikzlibrary{hobby,arrows,decorations.pathmorphing,backgrounds,positioning,fit,petri,calc}
\pgfdeclarelayer{background}
\pgfsetlayers{background,main}
\usepackage{todonotes}

\usepackage{tikz}
\usepackage[edges]{forest}
\usetikzlibrary{arrows.meta,shadows.blur}
\usepackage{rotating}

\usepackage{capt-of}
\usepackage{multirow}
\usepackage{changepage,threeparttable}

\definecolor{purple}{RGB}{27, 158, 119}
\definecolor{blue}{RGB}{217, 95, 2}
\definecolor{orange}{RGB}{117, 112, 179}
\definecolor{gray}{RGB}{239,240,241}
\definecolor{pink}{RGB}{254,15,127}
\definecolor{green}{RGB}{231, 41, 138}
\definecolor{darkgreen}{rgb}{0, 0.5, 0}

\usepackage{paralist}
\usepackage[stable]{footmisc}
\usepackage{adjustbox}

\usepackage{ifthen}
\newcommand{\CC}[1][]{$\text{C\hspace{-.25ex}}^{_{_{_{++}}}}
\ifthenelse{\equal{#1}{}}{}{\text{\hspace{-.625ex}#1}}$}

\usepackage{bm}
\usepackage{bbm}
\usepackage{amsmath}
\usepackage{amssymb}
\usepackage{amsthm}
\usepackage{amsfonts}
\usepackage{thmtools}	
\usepackage{xfrac}
\usepackage{mleftright}
\usepackage{stmaryrd}
\usepackage{nicefrac}
\usepackage{algorithm}
\usepackage{algorithmicx}
\usepackage[noend]{algpseudocode}

\usepackage{dirtytalk}

\usepackage{thm-restate}

\let\originalleft\left
\let\originalright\right
\renewcommand{\left}{\mathopen{}\mathclose\bgroup\originalleft}
\renewcommand{\right}{\aftergroup\egroup\originalright}

\usepackage{pifont}

\usepackage{enumitem}
\setlist[enumerate]{itemsep=0.05ex, topsep=0.1\topsep}
\setlist[description]{itemsep=0.05ex, topsep=0.1\topsep}
\setlist[itemize]{itemsep=0.05ex, topsep=0.1\topsep}

\makeatletter
\def\thmt@refnamewithcomma #1#2#3,#4,#5\@nil{%
\@xa\def\csname\thmt@envname #1utorefname\endcsname{#3}%
\ifcsname #2refname\endcsname
\csname #2refname\expandafter\endcsname\expandafter{\thmt@envname}{#3}{#4}%
\fi
}
\makeatother

\usepackage[capitalise,noabbrev]{cleveref}   
\theoremstyle{definition}
\newtheorem{theorem}{Theorem}

\newtheorem{proposition}[theorem]{Proposition}

\newtheorem{lemma}[theorem]{Lemma}
\newtheorem{corollary}[theorem]{Corollary}
\newtheorem{definition}[theorem]{Definition}

\newtheorem*{definition*}{Definition}
\usepackage{thm-restate}
\usepackage[mathic=true]{mathtools}
\usepackage{fixmath}
\usepackage{siunitx}

\usepackage[edges]{forest}
\usetikzlibrary{arrows.meta,shadows.blur}
\usepackage{rotating}

\usepackage{ellipsis}

\usepackage{anyfontsize}

\usepackage[auth-lg]{authblk}

\input{macros}

\usepackage[auth-lg]{authblk}

\begin{document}

%

%

\twocolumn[

\aistatstitle{Understanding Generalization in Node and Link Prediction}

\aistatsauthor{%
Antonis Vasileiou \And
Timo Stoll \And
Christopher Morris%
}
\aistatsaddress{%
\\
Computer Science Department, RWTH Aachen University, Germany \\
\texttt{\{antonis.vasileiou, timo.stoll\}@log.rwth-aachen.de}%
}

]

\begin{abstract}
Message-passing graph neural networks (MPNNs) are widely applied to node and link prediction across scientific and industrial domains, motivating diverse architectural innovations. Nevertheless, their generalization beyond training data remains insufficiently characterized. Existing analyses often rely on unrealistic i.i.d.\@ assumptions, overlooking correlations between nodes and links that are inherently induced by the graph structure. Additionally, most works typically fix aggregation schemes and employ restrictive loss functions. We present a unified theoretical framework that explicitly accounts for these correlations when analyzing MPNN generalization in inductive and transductive prediction tasks. Our framework further incorporates architectural parameters, loss functions, and structural properties, and extends naturally to general classification problems. Empirical evaluations substantiate our theoretical findings and demonstrate the critical impact of structural correlations on MPNN generalization.
\end{abstract}

\section{Introduction}\label{sec:introduction}

Graphs model interactions in the life, natural, and formal sciences, such as atomistic systems~\citep{Duv+2023,Zha+2023c} or social networks~\citep{Eas+2010,Lov+2012}, motivating machine learning methods for graph-structured data. Neural networks tailored to such data, mainly \new{message-passing graph neural networks} (MPNNs)~\citep{Gil+2017,Sca+2009}, have gained wide attention, showing strong results in drug design~\citep{Won+2023}, social network analysis~\citep{Bor+2024}, weather forecasting~\citep{Lam+2023}, and combinatorial optimization~\citep{Cap+2021,Gas+2019,Sca+2024,Qia+2023}.

MPNNs support node-, link-, and graph-level prediction~\citep{Cha+2020}. While their generalization in graph-level tasks is well studied~\citep{Fra+2024,Mor+2023,scarselli2018vapnik,Lev+2023,Vas+2024,Vas+2024b}, node- and link-level generalization remains underexplored~\citep{Mor+2024b}. Existing node-level studies~\citep{Gar+2020,Ess+2021,Lia+2021,scarselli2018vapnik,Tan+2023,Ver+2019} often assume i.i.d.\@ samples, ignoring correlations from graph structure, and rely on restrictive losses (e.g., margin or $0$-$1$ losses), which are impractical for training. Likewise, despite advances in link-prediction models~\citep{Ali+2022,Ye+2022,Zha+2021}, their generalization is unstudied. In addition, most analyses target either the \new{inductive} or \new{transductive} regime (see~\cref{sec:learningongraphs}). That is, in the inductive regime, models are trained on multiple labeled graphs and predict on unseen ones, whereas the transductive regime assumes a single graph with partially observed labels. A unified understanding of MPNN generalization across these regimes is still lacking; see~\cref{app_sec:rel_work} for a detailed discussion of related work.

\textbf{Present work} We introduce a unified framework for analyzing MPNNs' generalization ability in node- and link-level prediction, extending recent covering number bounds~\citep{Vas+2024}. Unlike~\citet{Vas+2024}, we handle non-i.i.d.\@ samples in these regimes, requiring non-trivial extensions. Concretely,
\begin{enumerate}[leftmargin=*]
\item we present a unified framework, named \emph{generalized MPNNs} (\cref{sec:gen_MPNNs}), that encompasses a broad class of architectures, including modern models~\citep{Zha+2021,Zhu+2021}, for both node- and link-prediction tasks. 
\item We define pseudometrics for MPNNs, named \emph{unrolling distances} (\cref{sec:unrolling_distances}), which satisfy Lipschitz continuity, thereby capturing the underlying graph structure.
\item We establish robustness-based generalization bounds that apply to arbitrary inductive and transductive learning tasks (\cref{{thm:Xu_Mannor_noniid}}, \cref{thm:Xu_Mannor_transductive}). Combined, these bounds yield explicit generalization guarantees for node- and link-prediction (\cref{thm:binaryclassificationdatadepend}, \cref{thm:binaryclassificationtransductive}), while accounting for sample dependencies and standard losses, such as cross-entropy.
\item Empirically, we show our theory aligns with practice, yielding a sharper understanding of when MPNNs generalize in node- and link-level tasks.
\end{enumerate}

\emph{Overall, our results offer a unified framework for understanding the MPNNs' generalization abilities for node- and link-level prediction tasks, accounting for different regimes, non-i.i.d.\@ samples, architectural choices, and graph structure.}



\section{Background}
\label{sec:background}
In the following, we introduce the MPNN architecture employed in this work, together with the two statistical settings (inductive and transductive) on which our generalization results are established. Throughout, we adopt standard notation for graphs, (pseudo-)metric spaces, and Lipschitz continuity; see~\cref{app_sec:background}.

\subsection{Message-passing neural networks}
\label{sec:MPNNs}
A well-known and widely used class of graph-based models is the family of message-passing neural networks (MPNNs)~\citep{Gil+2017,Sca+2009}. These architectures learn vector representations for each node by iteratively aggregating information from their neighbors.  Following~\citet{Gil+2017}, let $G$ be a node-featured graph with initial node features $\hb_{v}^\tup{0} \in \Rb^{d_0}$, $d_0 \in \Nb$, for $v\in V(G)$. An \new{MPNN architecture} consists of a composition of $L$ neural network layers for some $L>0$. In each \new{layer}, $t \in [L]$,  we compute a $d_{t}$-dimensional node feature $\hb ^\tup{t}_{G}(v)$ via
\begin{equation*}
	\UPD^\tup{t}\Bigl(\hb_{G}(v)^\tup{t-1},\AGG^\tup{t} \bigl(\oms \hb_{G}(u)^\tup{t-1}
	\mid u\in N(v) \cms \bigr)\Bigr),
\end{equation*}
$d_t \in \Nb$, for $v\in V(G)$, where $\UPD^\tup{t}$ and $\AGG^\tup{t}$ may be  parameterized functions, e.g., neural networks.\footnote{Strictly speaking, \citet{Gil+2017} consider a slightly more general setting in which node features are computed by $\hb_G(v)^\tup{t} \coloneq
		\UPD^\tup{t}\Bigl(\hb_G(v)^\tup{t-1},\AGG^\tup{t} \bigl(\oms (\hb_G(v)^\tup{t-1},\hb_G(u)^\tup{t-1},\ell_G(v,u))
		\mid u\in N(v) \cms \bigr)\Bigr)$,
	where $\ell_G(v,u)$ denotes the edge label of the edge $(v,u)$.}

\textbf{Sum aggregation} For our analysis, we focus on a simplified yet expressive MPNN architecture, matching $\wlone$ expressivity~\citep{Morris2019}, that uses sum aggregation. We note, however, that your analysis lifts straightforwardly to other aggregation functions, e.g., mean aggregation. Given a node-featured graph $(G, a_G)$, we initialize node features as $\hb_{G}(v)^\tup{0} = a_G(v)$ for all $v \in V(G)$, and we 
\begin{equation}
	\label{def:sum_mpnnsgraphs}
    \varphi_{t}\Bigl(\vec{W}_{t}^{(1)} \vec{h}^{(t-1)}_G(v) + \vec{W}_{t}^{(2)} \sum_{u \in N(v)} \vec{h}^{(t-1)}_G(u) \Bigr),
\end{equation}

for $v \in V(G)$, where $\varphi_{t} \colon \Rb^{d_{t-1}} \to \Rb^{d_t}$ is an $L_{\varphi_t}$-Lipschitz continuous function with respect to the metric induced by the $2$-norm, for $d_t \in \Nb$ and $t \in [L]$. The matrices $\vec{W}_{t}^{(1)}, \vec{W}_{t}^{(2)} \in \Rb^{d_{t-1} \times d_t}$ are assumed to have a $2$-norm bounded by some constant $B > 0$.

\textbf{MPNNs for link prediction}
Given a graph, MPNNs for link prediction aim at computing the probability of a link between two target nodes. Older architectures such as RGCN~\citep{Kipf2018RGCN}, CompGCN~\citep{Vashishth2020CompGCN}, and GEM-GCN~\citep{Yu2020GEMGCN} first compute node-level representations and combine them to compute a node-pair representation for each potential link. However, \citet{Zha+2021} demonstrated that such approaches are insufficient for link prediction, as they fail to capture node interactions. Hence, recently, a large set of more expressive link prediction architectures emerged, which are based on applying the MPNN to a suitably transformed graph that depends on the candidate link rather than to the original graph. An overview of state-of-the-art MPNN-based link prediction models, namely, SEAL~\citep{Zha+2018}, C-MPNNs~\citep{huang2023theory}, and NCNs~\citep{Wan+2024}, is provided in \cref{app_sec:MPNN_link_pred}.

\subsection{Inductive and transductive learning}\label{sec:learningongraphs}

This section presents the two learning frameworks underlying our generalization analysis, \new{inductive} and \new{transductive learning}.

\textbf{Inductive setting} Let $\cX$ denote the input space and $\cY$ the label space, and define $\cZ \coloneqq \cX \times \cY$. A learning algorithm $\cA$ receives as input a training sample $\cS = \{(x_i, y_i)\}_{i=1}^N \subset \cZ$ drawn i.i.d.\@ from an unknown distribution $\mu$ on $\cZ$, and outputs a hypothesis $h = \cA_\cS \in \cH$, where $\cH$ is a hypothesis class. Given a bounded loss function $\ell \colon \cH \times \cX \times \cY \to \Rb^+$, the \new{expected} and the \new{empirical risk} are defined as
\begin{align*}
&\ell_{\text{exp}}(\cA_\cS) \coloneqq \Eb_{(x,y)\sim\mu} \left[ \ell(\cA_\cS, x, y) \right], \quad \text{ and } \quad \\
&\ell_{\text{emp}}(\cA_\cS) \coloneqq \frac{1}{N} \sum_{(x,y)\in\cS} \ell(\cA_\cS, x, y),
\end{align*}
respectively.
The generalization error is then defined as the absolute difference
\begin{equation*}
\left| \ell_{\text{emp}}(\cA_\cS) - \ell_{\text{exp}}(\cA_\cS) \right|.
\end{equation*}

\textbf{Transductive setting} In contrast to the inductive case, the transductive setting does not assume a distribution over $\cZ$. Instead, we are given a fixed dataset $\cZ = \{(x_i, y_i)\}_{i=1}^{m+u}$ of $m$ labeled and $u$ unlabeled examples. A transductive learning algorithm receives the full input set $\{x_i\}_{i=1}^{m+u}$ and the labels $\{y_i\}_{i=1}^{m}$ of a randomly selected training subset. The goal is to minimize the average loss over the remaining $u$ test points. Following \citet[Setting 1]{DBLP:books/sp/Vapnik06} for the model sampling process we let $\pi$ be a random permutation over $[m+u]$, and we define,
\begin{align*}
&R_m \coloneqq \frac{1}{m} \sum_{i=1}^{m} \ell(\cA_\pi, x_{\pi(i)}, y_{\pi(i)}), \quad \text{ and } \\
&R_u \coloneqq \frac{1}{u} \sum_{i=m+1}^{m+u} \ell(\cA_\pi, x_{\pi(i)}, y_{\pi(i)}).
\end{align*}
The permutation $\pi$ is treated as a random variable, uniformly distributed over all permutations of $[m+u]$, and models the process in which the $m$ training samples are drawn uniformly at random \emph{without} replacement from the $m+u$ elements of $\cZ$.
The generalization error is given by
\begin{equation*}
\left| R_m - R_u \right|.
\end{equation*}

Further formal details, including the permutation model, symmetry assumptions, and loss function, are provided in \cref{app_sec:learning_formalism}. 
In the remainder of the paper, we focus on generalization bounds for node and link prediction under binary classification.

\section{Generalized MPNNs and unrolling distances}
\label{sec:generalized_mpnns_and_pms}

\begin{figure*}
    \centering
    \scalebox{1.0}{\input{rho.tex}}
    \caption{An illustration of the padding process described in \cref{sec:unrolling_distances}. Red nodes indicate the padded (added) nodes. 
    }
    \label{fig:padding_process}
\end{figure*}

As discussed in~\cref{sec:background}, in link prediction tasks, MPNNs are often applied to a transformed version of the input graph to encode task-specific structural information. To study the generalization properties of MPNNs in a unified manner for both node- and link-level prediction, we introduce \new{generalized MPNNs}. Along with these, we define suitable distances aligned with such architectures between nodes, edges, or, more generally, tuples of nodes within a graph. These distances are constructed from appropriate computation trees and play a central role in establishing our main generalization theorems in \cref{sec:robustness_under_dependency_for_graphs}.

\subsection{Generalized MPNNs}
\label{sec:gen_MPNNs}
Intuitively, generalized MPNNs compute vector representations for target entities (such as nodes, links, or tuples of nodes) depending on the prediction task. They proceed in three steps. First, the input graph is transformed via a transformation function, and a standard MPNN (as defined in \cref{def:sum_mpnnsgraphs}) is applied to compute node embeddings on the transformed graph. Secondly, a selection function chooses a set of nodes from the transformed graph for each target entity. Thirdly, for each target entity, vector representations of the selected nodes (from the second step) are aggregated via a pooling function to yield the final representation of the target entity.

Formally, we first specify the type of vectorial representation the MPNN aims to compute, e.g., node, link, or graph vectorial representations. Formally, let $d \in \Nb$, and $\cG_{d}^{\Rb}$ be the set of all graphs with node features in $\Rb^d$. For a subset $\cG' \subseteq \cG_{d}^{\Rb}$, we define a \new{graph-representation task} on $\cG'$ as $\cG' \otimes \mathfrak{R}$, where
\begin{equation*}\label{def:representation_task}
    \cG' \otimes \mathfrak{R} \subseteq \{ (G,S) \mid G \in \cG', S \subseteq V(G) \} .
\end{equation*}
Intuitively, the set $\cG' \otimes \mathfrak{R}$ consists of pairs of graphs (determined by $\cG'$) and the subsets of nodes of these graphs (determined by $\mathfrak{R}$) for which we aim to compute vectorial representations. For example, if the goal is to build a model for node classification on graphs in $\cG_{d}^{\Rb}$, we consider the \new{node-representation task},

\begin{equation*}
    \cG_{d}^{\Rb} \otimes V \coloneqq \{ (G,\{u\}) \mid G \in \cG_{d}^{\Rb}, u \in V(G) \},
\end{equation*}
whereas the \new{link-representation task}, 
\begin{equation*}
    \cG_{d}^{\Rb} \otimes E \coloneqq  \{ (G,e) \mid G \in \cG_{d}^{\Rb}, e \subseteq V(G), |e|=2 \}.
\end{equation*}

Based on this, we now define \new{generalized MPNNs} that encompass all the above tasks. Let $\cG' \subseteq \cG_{d}^{\Rb}$, given a graph-representation task $\cG' \otimes \mathfrak{R}$ and $L \in \Nb$, a generalized MPNN with $L$ layers, denoted as a $(\cT, \cV, \Psi)$-\MPNN$(L)$, computes vector representations for elements in $\cG' \otimes \mathfrak{R}$ by applying an MPNN on a transformed graph. A generalized MPNN consists of three main components.
\begin{enumerate}[leftmargin=*]
    \item A \new{transformation function} $\cT \colon \cG' \to \cG$. 
    \item A \new{selection function} $\cV$ determines, for a pair $(G,S) \in \cG' \otimes \mathfrak{R}$, which nodes $A$ in the transformed graph $\cT(G)$ are used in the MPNN computation for $(G,S)$. That is, 
    \begin{equation*}
    \cV \colon (G,S) \mapsto A \subset V(\cT(G)).
    \end{equation*}
    \item A \new{pooling function} computes the final representations for the elements in $\cG' \otimes \mathfrak{R}$. That is, we first use an MPNN for $L$ layers for computing vectorial representations for the nodes in the transformed graph using the selection function $\cV(G,S)$, leading to the set,
    \begin{equation*}
    F(G,S) = \{ (\hb_{\cT(G)}^{(L)}(v), v) \mid v \in \cV(G,S) \}.
    \end{equation*}
    Subsequently, we use the \new{pooling function},
    \begin{equation*}
    \Psi \colon F(G,S) \mapsto \Psi(F(G,S)) \in \Rb^{d}, \text{for } (G,S) \in \cG' \otimes \mathfrak{R},
    \end{equation*}
    to aggregate these vectorial representations and compute the final representations for the elements in $\cG' \otimes \mathfrak{R}$.
\end{enumerate}
The final outcome is denoted by $\hb^{(L)}_{\cT,\cV,\Psi}(G,S)$, i.e., 
\begin{equation*}
\hb^{(L)}_{\cT,\cV,\Psi}(G,S) = \Psi(F(G,S)), \quad \text{for } (G,S) \in \cG'\otimes \mathfrak{R}.
\end{equation*}

Below, we describe how this generalized MPNN framework covers most well-known models in graph-, node-, and link-level prediction tasks.


\textbf{Node and link prediction with generalized MPNNs}  Similarly, for node-level prediction with MPNNs described in \cref{sec:MPNNs}, $\cG' \otimes \mathfrak{R} = \{(G,\{u\}) \mid G \in \cG_{d}^{\Rb}, u \in V(G)\}$, $\cT(G)=G$, $\cV(G,u) = \{u\}$, and $\Psi(F(G,\{u\})) = \hb_{G}^{(L)}(u)$. In \cref{app_sec:link_mpnns_as_generalizedMPNNs} we describe how common link prediction architectures, such as SEAL, C-MPNNs, and NCNs, fit into our framework. Sum-pooling is one of the most commonly used pooling functions $\Psi$. In \cref{app_sec:sub_sum_pooling}, we define a generalized version called \new{sub-sum} property, see~\cref{def:sub-sum_pooling}, and show that all architectures in our analysis use pooling functions that satisfy this property; see~\cref{prop:sub_sum_propety}.

\subsection{Unrolling distances}
\label{sec:unrolling_distances}
In the following, we derive distance functions that capture the computation performed by generalized MPNNs. Formally, let $\cG' \subset \cG_{d}^{\Rb}$ and let $(\cG', \mathfrak{R})$ be a graph-representation task as previously defined. For technical reasons, we assume that for all $G \in \cG'$ and all $u \in V(G)$, the node feature satisfies $a_G(u) \neq \vec{0}_{\Rb^d}$. The goal of this section is to define a family of distances on $(\cG', \mathfrak{R})$ that are tailored to the structure of $(\cT, \cV, \Psi)$-MPNNs. These distances are parameterized by the transformation function $\cT$ and the selection function $\cV$, and are designed to satisfy the Lipschitz property, which is essential for our generalization analysis in the next section.

The definition of the distance is based on the notion of computation trees (also called unrolling trees), which are commonly used to describe the message-passing mechanism or characterize the $\wlone$ algorithm. For further background, see \cref{app_sec:1WL}.

Given $(G, S) \in \cG' \otimes \mathfrak{R}$, a transformation function $\cT \colon \cG' \to \cG_{d}^{\Rb}$, a selection function $\cV \colon (G, S) \mapsto A \subset V(\cT(G))$, and a number of layers $L \in \Nb$, we define the multiset of unrolling trees as
\begin{equation*}
	\cF^{(L)}(G, S) = \oms \UNR{\cT(G), u, L} \mid u \in \cV(G, S) \cms.
\end{equation*}

Following \citet{chuang2022tree,Vas+2024}, we aim to define a distance between two input pairs $(G_1, S_1), (G_2, S_2) \in \cG' \otimes \mathfrak{R}$ that measures the misalignment between the multisets of unrolling trees $\cF^{(L)}(G_1, S_1)$ and $\cF^{(L)}(G_2, S_2)$. To handle multisets and unrollings of different sizes, we first define a padding function $\rho$, illustrated in \cref{fig:padding_process} and formally described in \cref{app_sec:unrolling_distances_extension}.

We now introduce the unrolling distance on the domain consisting of all graph-subset pairs, where the graph comes from $\cG_{d}^{\Rb}$ and the subset 
is any collection of its nodes. We denote this set by
$\cG_{d}^{\Rb} \otimes \mathfrak{C} \coloneq \{ (G, S) \mid G \in \cG_{d}^{\Rb},\ S \subseteq V(G) \}.$

\begin{definition}[Unrolling Distance]
\label{def:gen_unrolling_distances}
Let $\cT \colon \cG' \to \cG_{d}^{\Rb}$ be a transformation function and $\cV \colon (G, S) \mapsto A \subset V(\cT(G))$ a selection function. For any $(G, S) \in \cG' \otimes \mathfrak{C}$, the $(\cT, \cV)$-\new{Unrolling Distance} of depth $L$, denoted by $\mathrm{UD}^{(L)}_{\cT, \cV}$, is defined on $\cG_{d}^{\Rb} \otimes \mathfrak{C}$ as follows. Let $(G_1, S_1), (G_2, S_2) \in \cG' \otimes \mathfrak{C}$ and let $\rho(\cF^{(L)}(G_1, S_1), \cF^{(L)}(G_2, S_2)) = (M_1, M_2)$. Then,
\begin{equation}
\label{eq:unrolling_distance}
\begin{aligned}
&\mathrm{UD}^{(L)}_{\cT, \cV}((G_1, S_1), (G_2, S_2))
=\\
&\min_{\varphi \in S(F_1, F_2)} 
\sum_{x \in V(F_1)} \norm{a_{F_1}(x) - a_{F_2}(\varphi(x))}_2,
\end{aligned}
\end{equation}
where $F_i$ is the forest formed by the disjoint union of all trees in $M_i$ for $i = 1, 2$, and
$S(G, H)$ is the set of all \new{edge-preserving bijections} from $V(G)$ to $V(H)$. Here, an edge-preserving bijection preserves the graph structure (i.e., a graph isomorphism)
but not necessarily the node features, and that for a node-featured graph $G$, $a_G(u)$ denotes the feature of node $u \in V(G)$. 
\end{definition}
In \cref{app_sec:unrolling_distances_extension}, we show that the unrolling distance $\mathrm{UD}^{(L)}_{\cT, \cV}$ is a well-defined pseudo-metric (\cref{lemma:unrol_dist_pseudo-metricproperty}). We also introduce a weighted extension of this distance (\cref{def:weighted_node-forest-distance}) and prove that if the pooling function of a generalized MPNN satisfies the sub-sum property, then it is Lipschitz continuous concerning its corresponding unrolling distance (\cref{prop:Lipschitz_generalized_MPNNs}).

\section{Robustness generalization under dependency}
\label{sec:robustness_under_dependency}

As discussed in~\cref{sec:background}, in node- and link-level prediction tasks, samples cannot, in general, be assumed to be i.i.d. We therefore extend the covering number-based generalization framework of \citet{Xu+2021}, initially developed for the inductive setting under the i.i.d.\@ assumption, to handle dependent data.  In this work, we address both the inductive and the transductive settings. For the inductive case, we introduce explicit assumptions about the dependency structure of the samples, since the bounds in \cref{thm:Xu_Mannor_noniid} depend directly on these dependencies. In the transductive case, however, the dependency structure is determined by the setup itself. 

We begin by introducing a slightly modified notion of robustness, which we refer to as \emph{uniform robustness}, that differs from that of \citet{Xu+2012} and recover known results from~\citet{Xu+2012}.
\begin{definition}[Uniform robustness]
\label{def:robustness}
We say that a learning algorithm $\cA$ on the space $\cZ \coloneq \cX \times \cY$ with hypothesis class $\cH$ is \emph{$(K,\varepsilon)$-uniformly robust} if there exist constants $K \in \Nb$ and $\varepsilon > 0$ such that, there exists a partition $\{C_i\}_{i=1}^{K}$ of $\cZ$ satisfying the following property: For all $i \in [K]$, and for any sample $\cS$,
\begin{align*}
z_1, z_2 \in C_i \implies | \ell(\cA_{\cS}, z_1) - \ell(\cA_{\cS}, z_2)| \leq \varepsilon.
\end{align*}
\end{definition}

In the spirit of \citet[Theorem 14]{Xu+2012}, the following theorem shows that Lipschitz continuity of the loss function concerning a pseudo-metric $d$ is sufficient for a hypothesis class to satisfy the robustness property. It also provides robustness parameters regarding the covering number of the space with respect to $d$ and the associated radius. The proof follows directly from \citet[Theorem 14]{Xu+2012}.

\begin{theorem}
\label{thm:lipschitzimpliesrobustness}
Let $\cA$ be a learning algorithm on $\cZ$ for a hypothesis class $\cH$, and let $\ell \colon \cH \times \cZ \to \Rb$ be a loss function. Suppose $d$ is a pseudo-metric on $\cZ$. If $\ell(h,\cdot)$ is $C$-Lipschitz with respect to $d$, i.e., for all samples $\cS$ and $z_1,z_2 \in \cZ$,
\begin{equation*}
\left| \ell(\cA_{\cS},z_1) - \ell(\cA_{\cS},z_2) \right| \leq C \cdot d(z_1,z_2), 
\end{equation*}
then $\cA$ is $\left( \cN(\cZ, d, \frac{\varepsilon}{2}), C\cdot\varepsilon \right)$-uniformly robust for all $\varepsilon > 0$.
\end{theorem}

The following theorem establishes a connection between the robustness property and generalization performance under an i.i.d.\@ sampling process.

\begin{theorem}[\citet{Xu+2012}, Theorem 3]
\label{thm:xumannorthm3}
	If $\cS$ consists of $N$ i.i.d.\@ samples, $\varepsilon>0$ and $\cA$ is $(K,\varepsilon)$-(uniformly) robust, then for all $\delta > 0$, with probability at least $1 - \delta$,
\begin{equation*}
\left| \ell_{\mathrm{exp}}(\cA_{\cS}) - \ell_{\mathrm{emp}}(\cA_{\cS}) \right| \leq \varepsilon + M\sqrt{\frac{2K\log2 + 2\log(\frac{1}{\delta})}{N}}.
\end{equation*} 
\end{theorem}

In what follows, we aim to relax the i.i.d.\@ assumption in the above result by (i) introducing possible dependencies between samples, leading to generalization analysis in the inductive case, and (ii) deriving a generalization analysis in the transductive case.

\textbf{Inductive robustness generalization analysis}
\label{sec:inductive-robustness}
Independence plays a central role in deriving generalization bounds, particularly in the proof of \cref{thm:xumannorthm3}, which relies on applying the Bretagnolle--Huber--Carol inequality to a multinomial vector drawn from an i.i.d.\@ sample~\citep[Proposition A.6.6]{weakconvergence}. Since this inequality ultimately depends on Hoeffding’s inequality, the i.i.d.\@ assumption is essential.

To relax this assumption, we follow \citet{DBLP:journals/rsa/Janson04}, who introduced \emph{dependency graphs}, graph structures in which nodes represent random variables and edges indicate potential dependencies, to extend classical concentration bounds to the dependent setting. Informally, a set of random variables is said to be 
$G$-dependent if the dependencies among the variables can be represented by a graph $G$. We formally define dependency graphs in \cref{app_sec:dependency_graphs}. The following lemma provides a generalized Bretagnolle--Huber--Carol inequality for graph-dependent random variables.

\begin{lemma}
\label{lemma:Bretagnolle_Huber_extension}
Let $(\Omega, \cF, P)$ be a probability space, let $n \in \Nb$, and let $G \in \cG_n$. Suppose $X_i \colon \Omega \to A \subset \Rb$ are $G$-dependent, identically distributed random variables with distribution $\mu$, for all $i \in [n]$. Assume $A$ can be partitioned into $K \in \Nb$ disjoint subsets $\{C_j\}_{j=1}^K$, and define
\begin{equation*}
Z_j = \sum_{i=1}^{n} \mathbf{1}_{\{X_i \in C_j\}}, \quad \text{for } j \in [K].
\end{equation*}
Then, for all $t > 0$,
\begin{equation*}
P \left( \sum_{j=1}^{K} \left| Z_j - n \mu(C_j) \right| \geq 2t \right) 
\leq 2^{K+1} \exp\left( \frac{-2t^2}{\chi(G)n} \right),
\end{equation*}
where $\chi(G)$ denotes the chromatic number of the dependency graph $G$.
\end{lemma}

The above lemma allows us to extend \cref{thm:xumannorthm3} to graph-dependent samples via the following generalization bound.
\begin{theorem}
\label{thm:Xu_Mannor_noniid}
Let $(\Omega, \cF, P)$ be a probability space, and let $\cA$ be a $(K, \varepsilon)$-uniformly robust learning algorithm on $\cZ$. Then for any $\delta > 0$, with probability at least $1 - \delta$, and for all samples $\cS$ with dependency graph $G[\cS]$, we have that
\begin{align*}
&\left| \ell_{\mathrm{exp}}(\cA_{\cS}) - \ell_{\mathrm{emp}}(\cA_{\cS}) \right| 
\leq \\
&\varepsilon+ M \sqrt{\frac{\chi(G[\cS])\left( 2(K+1)\log 2 + 2\log\left(\frac{1}{\delta}\right)\right)}{|\cS|}},
\end{align*}
where $\ell_{\mathrm{exp}}(\cA_{\cS})$ and $\ell_{\mathrm{emp}}(\cA_{\cS})$ denote the expected and empirical losses, respectively, and $M \in \Nb$ is an upper bound on the loss function.
\end{theorem}

\textbf{Transductive robustness generalization analysis}
\label{sec:transductive-robustness}
In this section, we derive generalization bounds in the transductive setting using robustness. Unlike the inductive case, where robustness alone suffices, we require that the learning algorithm be stable to small changes in the training set. Informally, a transductive algorithm is $\beta$-stable if swapping one training and one test input changes the model’s predictions by at most $\beta$; see \cref{def:stability} in \cref{app_sec:transductive_dependency}; see~\cref{def:stability} in \cref{app_sec:transductive_dependency} for the formal definition.


To derive generalization bounds for robust and stable transductive learning algorithms, we replace Hoeffding’s inequality with Azuma’s inequality~\citep{Azuma1967}, see \cref{thm:azuma_inequality}) in \cref{app:martingales}, extending the Bretagnolle--Huber--Carol inequality to the following result.

\begin{lemma}
\label{lemma:Bretagnolle_Huber_permutation_extension}
Let $(\Omega, \cF, P)$ be a probability space, and let $n, n' \in \Nb$ with $n < n'$. Let $Z = \{z_1, z_2, \ldots, z_{n'}\}$ be a finite set, and let $\{C_j\}_{j \in [K]}$ be a partition of $Z$. Suppose $\pi \sim \text{Unif}(S_{n'})$, and define
\begin{equation*}
X_j = \sum_{i=1}^{n} \textbf{1}_{\{z_{\pi(i)} \in C_j\}}.
\end{equation*}
Then, for any $S \subset [K]$ and any $t > 0$, we have that
\begin{equation*}
P\left( \left| \sum_{j \in S} X_j - \Eb_{\pi}\left[ \sum_{j \in S} X_j \right] \right| \geq t \right) \leq 2 \exp\left( \frac{-t^2}{2 |S|^2 \cdot n} \right).
\end{equation*}
\end{lemma}

Using the lemma above, we derive, to the best of our knowledge, the first robustness-based generalization bound for the transductive learning setting.

\begin{theorem}
\label{thm:Xu_Mannor_transductive}
Let $(\Omega, \cF, P)$ be a probability space, $\ell$ be a loss function bounded by $M$ satisfying the conditions of \cref{lemma:stability_lipschitz_bound}. If $\cA$ is a transductive learning algorithm on $\cZ = \{z_i\}_{i=1}^{m+u}$ with hypothesis class $\cH$ that is $(K,\varepsilon)$-uniformly-robust and satisfies uniform transductive stability $\beta>0$. Then, for every $\delta > 0$, with probability at least $1 - \delta$, the following inequality holds,
\begin{align*}
\left|R_m - R_u \right| &\leq 
 2\varepsilon + 
 \left( \frac{1}{\sqrt{m}}+\frac{1}{\sqrt{u}} \right) \cdot M \cdot K \cdot \\
 & \cdot \sqrt{2(K+1)\log2 + 2\log\left(\frac{1}{\delta}\right)} + L_{\ell}\beta,
\end{align*}
\end{theorem}

\section{Graph learning and robustness generalization}
\label{sec:robustness_under_dependency_for_graphs}

Building on the previous section, this section presents generalization bounds for graph representation learning tasks, with a focus on node and link prediction in the inductive and transductive frameworks. Our analysis combines~\cref{thm:Xu_Mannor_noniid} and \cref{thm:Xu_Mannor_transductive} with \cref{thm:lipschitzimpliesrobustness} and the pseudo-metrics defined in \cref{sec:unrolling_distances}.


\textbf{Inductive node and link classification.}
We begin with the inductive learning approach for node and link prediction tasks.  We derive the main results and state the key assumption on the dependency structure for the binary classification node-level setting. The results can be extended straightforwardly to link prediction and multiclass classification. Recall that in the inductive setting, we assume that the dataset consists of different graphs with node features, where only a subset of the node labels is observed. The goal is to predict the labels of nodes in new, unseen graphs sampled from the same distribution. Following the notation from the previous section, in the binary classification case, the input space is $\cX \coloneqq \cG_{d}^{\Rb} \otimes V$ (node prediction) or  $\cX \coloneqq \cG_{d}^{\Rb} \otimes E$ (link prediction), and the label space is  $\cY \coloneqq \{0,1\}$. Our objective is to apply \cref{thm:Xu_Mannor_noniid} to robust learning algorithms by verifying the robustness condition via \cref{thm:lipschitzimpliesrobustness}, which relates robustness parameters to the covering number of a pseudo-metric space satisfying a Lipschitz condition.

In \cref{sec:unrolling_distances} and \cref{prop:Lipschitz_generalized_MPNNs}, we introduced a suitable pseudo-metric on $\cX$ and proved its Lipschitz property. To extend this to $\cX \times \cY$, we define a pseudo-metric $d_{\infty}$ using the discrete metric $\delta$ on $\{0,1\}$, i.e.,
\begin{equation*}
d_{\infty}\!\big((x,y),(x',y')\big) \;\coloneq\;
\max \left\{\, \mathrm{UD}_{\cT,\cV}^{(L)}(x,x'), \; \delta(y,y') \,\right\},
\end{equation*}
where
\begin{equation*}
\delta(y_1, y_2) \coloneq
\begin{cases}
  0, & \text{if } y_1 = y_2, \\[4pt]
  \infty, & \text{otherwise}.
\end{cases}
\end{equation*}

\textbf{Assumptions on the data dependencies}
To apply \cref{thm:Xu_Mannor_noniid}, we need to make assumptions about the dependency structure of the random variables in the training set. In our setting, each node--label pair corresponds to a random variable. It is therefore reasonable to assume that nodes belonging to different graphs are independent since they are generated from separate structures and do not share edges or features. In contrast, nodes within the same graph may exhibit dependencies due to the graph's relational structure.  Formally, we make the following assumption.

\begin{itemize}
\item[(A)] If $(G_1, u_1, y_1)$ and $(G_2, u_2, y_2)$ are two samples with $G_1 \neq G_2$, then they are independent; otherwise, they may be dependent.
\end{itemize}
We now present the main result for binary node classification. The same result extends to the multi-class case and to link prediction.

\begin{theorem}[Binary classification generalization]
\label{thm:binaryclassificationdatadepend}
Let $\cA$ be a learning algorithm on $\cZ = \cG_d^{\Rb} \otimes V \times \{0,1\}$.
Consider the hypothesis class $\cH$ consisting of $(\cT,\cV,\Psi)$-MPNNs with $L$-layers, where $\Psi$ is sub-sum.
For a sample $\cS$ (under the assumption (A)), and a loss function $\ell \colon \cH \times \cZ \to \Rb^{+}$ that is bounded by some $M \in \Rb$, and $\ell(h, \cdot)$ being $L_{\ell}$-Lipschitz concerning $d_{\infty}$, we have,
\begin{align*}
&\left| \ell_{\text{exp}}(\cA_{\cS}) - \ell_{\text{emp}}(\cA_{\cS}) \right| \leq 2C\varepsilon + \\
&+M \sqrt{\frac{ 2D_{\cS} \left( (2K_{\varepsilon}+1)\log 2 + \log\left(\frac{1}{\delta}\right)\right)}{N}}, \quad{\text{for all } \varepsilon >0,}
\end{align*}
where $K_{\varepsilon} = \cN\left( \cG_{d}^{\Rb}\otimes V,\mathrm{UD}_{\cT,\cV}^{(L)}, \varepsilon \right)$, $C$ is a constant depending on $L,L_{\ell}$, on the Lipschitz constants $L_{\varphi_t}$ of the MPNN layers, and $D_{\cS}$ is the maximum number of samples in $\cS$ that have been derived from the same graph, i.e.,  
\begin{equation*}
D_{\cS} \coloneq \max_{G \in \cG_d^{\Rb}} \left| \{ (G, v, y) \in \cS \mid v \in V(G), y \in \{0,1\} \} \right|. 
\end{equation*}
\end{theorem}

This theorem shows how the dependencies in the training data influence generalization. In particular, $D_{\cS}$ quantifies how many samples are drawn from the same underlying graph. If the training set consists of nodes drawn from many different graphs, then $D_{\cS}$ is small, which reduces the effect of dependencies and leads to tighter generalization bounds. Intuitively, in this case, the MPNN observes a more representative sample of the graph distribution, thereby generalizing better.

Note that the previous bound can become arbitrarily large when unbounded node degrees are present. However, unlike in graph classification, it is independent of the graph size. By restricting to graphs with maximum degree $q \in \Nb$, we can bound the covering number; see~\cref{cor:boundeddegreegeneralization} in \cref{app_sec:robustness_under_dependency_for_graphs}.

In \cref{fig:overview} we illustrate how graph structure influences the generalization bound. In particular, when the graph is simpler with respect to the chosen pseudo-metric, meaning it admits a small covering number at a relatively small radius, the resulting bounds become tighter.

\textbf{Transductive node and link classification}
We now turn to the transductive setting and state our main result for the binary classification case, noting that the result can be extended straightforwardly to multiclass classification problems. Recall that in the transductive setting we are given a single graph with partially observed node labels, and the goal is to predict the missing labels on this graph. We present the following generalization result for node (and link) binary classification problem under the transductive approach, utilizing \cref{thm:Xu_Mannor_transductive}.

\begin{theorem}[Binary classification generalization]
\label{thm:binaryclassificationtransductive}
Let $\cA$ be a transductive learning algorithm satisfying $\beta$-uniform stability on $\cZ = \{(x_i,y_i)\}_{i=1}^{m+u}$ where $(x_i,y_i) \in \cX \times\{0,1\}$, and $\cX = \cG_{d}^{\Rb} \otimes V$, or $\cX = \cG_{d}^{\Rb} \otimes E$. 
Consider the hypothesis class $\cH$ consisting of $(\cT,\cV,\Psi)$-MPNNs with $L$-layers, where $\Psi$ is sub-sum. For a loss function $\ell \colon \cH \times \cZ \to \Rb^{+}$ that is bounded by some $M \in \Rb$, and $\ell(h, \cdot)$ being $L_{\ell}$-Lipschitz with respect to $d_{\infty}$, we have, for $\varepsilon >0$,
\begin{align*}
&\left| \ell_{\text{exp}}(\cA_{\cS}) - \ell_{\text{emp}}(\cA_{\cS}) \right| \leq 2C\varepsilon + M \cdot K_{\varepsilon} \cdot \left(\frac{1}{\sqrt{m}}+\frac{1}{\sqrt{u}}\right) \\
&\cdot\sqrt{2(2K_{\varepsilon}+1)\log 2 + 2\log\left(\frac{1}{\delta}\right)} + L_{\ell}\beta, 
\end{align*}
where $K_{\varepsilon} = \cN\left( \cX, \mathrm{UD}^{(L)}_{\cT,\cV}, \varepsilon \right)$, $C$ is a constant depending on $L,L_{\ell}$, and on Lipschitz constants $L_{\varphi_t}$ of the message passing layers.
\end{theorem}

\begin{figure}
    \centering
    \scalebox{0.8}{\input{fig_overview.tex}}
    \caption{Illustrating how the graph structure, equipped with a proper pseudo-metric, induces an alignment between the node distances and their embeddings in the Euclidean space via Lipschitz continuity. The constant $c$ denotes the Lipschitz constant. Each color represents the corresponding unrolling tree, as indicated below the graph.}
    \vspace{-15pt}
    \label{fig:overview}
\end{figure}

\section{Limitations and looking ahead}
\label{app_sec:limitations}
While our framework offers a principled approach to studying MPNN generalization, it has several limitations. Computing the proposed bounds---especially for link prediction---can be costly due to unrolling-based pseudo-metrics. The Lipschitz constant---central to the bounds---may be large for certain architectures, yielding loose estimates. Our analysis also requires pooling functions with the sub-sum property---excluding common choices like mean pooling and attention. Moreover, although the bounds cover both inductive and transductive settings, evaluating them in practice demands estimating all parameters in \cref{thm:binaryclassificationtransductive,thm:binaryclassificationdatadepend}. Finally, as in generalization theory more broadly, the bounds may be vacuous in some regimes, since they govern the whole hypothesis class rather than trained models, which typically achieve much smaller losses. Still, such uniform bounds are useful for characterizing graph spaces where MPNNs can generalize.

\emph{Looking ahead,} we plan to extend the framework to broader architectural choices---including alternative pooling and normalization---making it applicable to classes like graph transformers~\citep{Mue+2023}. We also aim to connect our covering-number framework with SGD learning dynamics.

\begin{table*}[t]
\centering
\caption{Generalization results for different sampling strategies related to \textbf{Q2} for inductive node classification on the PATTERN dataset. Strategy \emph{random} refers to training nodes sampled from a few graphs (resulting in higher sample dependency), while strategy \emph{uniform} uses nodes sampled uniformly across many distinct graphs. Strategy \emph{mixed} uniformly samples from a predefined number of graphs. $D_{\mathcal{S}}$ denotes the maximum sampled nodes of a single graph in the training set. Further, $n_{\text{train}}$ and $n_{\text{test}}$ denote the number of train and test nodes.}
\resizebox{.65\linewidth}{!}{
\begin{tabular}{lcccccc}
\toprule
\bf Method & \bf Mixed-4k-r & \bf Mixed-4k-u & \bf Mixed-8k-r & \bf Mixed-8k-u & \bf Random & \bf Uniform \\
\midrule
$n_{\text{train}}$ & 120{\,}000 & 120{\,}000 & 120{\,}000 & 120{\,}000 & 120{\,}000 & 120{\,}000 \\
$n_{\text{test}}$  & 116{\,}232 & 116{\,}232 & 116{\,}232 & 116{\,}232 & 116{\,}232 & 116{\,}232 \\
$D_{\mathcal{S}}$  & 183 & 30 & 170 & 15 & 186 & 9 \\
\midrule
Training loss & 0.440 \tiny$\pm$0.010 & 1.5806 \tiny$\pm$0.003 & 0.430 \tiny$\pm$0.011 & 1.579 \tiny$\pm$0.010 & 0.2485 \tiny$\pm$0.0073 & 1.5739 \tiny$\pm$0.0123 \\
Test loss     & 1.763 \tiny$\pm$0.002 & 1.516 \tiny$\pm$0.030 & 1.749 \tiny$\pm$0.013 & 1.515 \tiny$\pm$0.034 & 1.8007 \tiny$\pm$0.0426 & 1.5114 \tiny$\pm$0.0152 \\
Gen. gap      & 1.323 & 0.0646 & 1.319 & 0.064 & 1.552 & 0.0625 \\
\bottomrule
\end{tabular}
}
\label{table:Q2}
\vspace{-3mm}
\end{table*}
\begin{table*}[t]
\centering
\caption{Evaluation of node prediction bounds with $p=0.001$, $cc=1$ for LSP-ER($n,p,cc$) graphs related to \textbf{Q3}, where $n$ denotes the number of nodes, $p$ the probability of an edge between nodes, and $cc$ the number of random links between disconnected components. Further experiments are evaluated in \Cref{app_sec:exp}.}
\resizebox{.60\linewidth}{!}{
\begin{tabular}{lcccc}
\toprule
\bf Dataset & \bf (100, 0.001, 1) & \bf (200, 0.001, 1) & \bf (500, 0.001, 1) & \bf (1000, 0.001, 1) \\
\midrule
Calculated bound   & 6.58 & 6.07 & 5.67 & 2.26 \\
Generalization gap & 0.58 & 0.57 & 0.46 & 0.44 \\
\bottomrule
\end{tabular}
}
\label{table:LSP0001}
\vspace{-3mm}
\end{table*}

\section{Experimental study}
\label{sec:experiments}

In the following, we investigate to what extent our theoretical results translate into practice. Specifically, we answer the following questions.

	\textbf{Q1} To what extent are MPNN outputs correlated with unrolling distances, and does this support the Lipschitz property empirically? Can we reliably estimate the corresponding Lipschitz constants?\\
	\textbf{Q2} Is assumption (A), used in \cref{thm:binaryclassificationdatadepend}, reasonable in the context of node-level prediction tasks? Specifically, does training on nodes from more distinct graphs lead to better generalization performance than training on nodes primarily drawn from a single graph? \\
	\textbf{Q3} Is the behavior of our theoretical bounds consistent with the actual generalization gap, defined as the difference between training and test error across different datasets? 

\textbf{Results and discussion}
We use a simplified version of GIN \citep{Xu+How+Powerful+Are+GNNs+GIN2019} aligned with \cref{def:sum_mpnnsgraphs} for our experiments on node prediction tasks, and a simplified version of SEAL \citep{Zha+2018} for link prediction tasks. To address \textbf{Q3}, we use synthetic datasets based on Erd\H{o}s--R\'enyi generated graphs (LSP-ER) so that we can control and efficiently compute the covering number of the graph. 
See \cref{app_sec:exp} for details on our synthetic datasets, neural architectures, experimental protocol, and model configurations. In the following, we address questions \textbf{Q1} to \textbf{Q3}.

\textbf{Q1} In \cref{fig:corrMPNNs_node_L2,fig:corrMPNNs_link_L2,fig:corrMPNNs_node_L3,fig:corrMPNNs_link_L3} in the appendix, we observe a correlation between the Euclidean norm of the difference in MPNN outputs and the corresponding unrolling distances. While some models may not exhibit a linear relationship, this does not contradict the Lipschitz property. Importantly, the Lipschitz constant (from \cref{prop:Lipschitz_generalized_MPNNs}) can be upper-bounded by the slope of the line passing through $(0, 0)$ that lies above all observed data points.

\textbf{Q2} To validate results and assumptions in the inductive setting, where training involves nodes sampled from multiple graphs, and to test the alignment between empirical performance and the bounds given in \cref{thm:binaryclassificationdatadepend}, we use different sampling processes, namely \emph{Uniform}, \emph{Random}, and \emph{Mixed-(u,r)}; see \cref{app_sec:exp} for a detailed description. Each sampling process contains the same number of training nodes (i.e., nodes with known labels), and the exact same test set, but differs in the number of distinct graphs used for sampling the training nodes. In \Cref{table:Q2}, we observe that MPNNs tend to generalize significantly better when the training set contains nodes sampled from many distinct graphs (and hence smaller value of $D_{\cS}$ in \cref{thm:binaryclassificationdatadepend}), compared to when nodes are drawn from only a few graphs (and hence larger $D_{\cS}$) verifying our dependence assumptions in the inductive setting.

\textbf{Q3} According to Theorem~10, a simpler graph (characterized by a smaller covering number $K_{\varepsilon}$) should result in a tighter generalization bound. This is exactly what we observe in \Cref{table:LSP0001}, \Cref{table:LSP0002} and \Cref{table:LSP00005}. More specifically, we increase the number of nodes $n$ while keeping the edge probability fixed. This makes the graph structurally simpler. As predicted, both the theoretical bound and the empirical generalization gap decrease.
 Similar experiments can be found in \Cref{table:increasingedge}, where we increase the edge probability while keeping the number of nodes fixed. This leads to a more complex graph, resulting in both a higher theoretical bound and a larger generalization gap.

\section{Conclusion}
In this work, we investigated how graph structure influences the generalization performance of MPNNs in non-i.i.d.\@ settings, considering both inductive and transductive learning frameworks. We introduced a unified theoretical framework encompassing most state-of-the-art MPNN-based architectures for node and link prediction tasks. Our empirical study supports the theoretical findings and validates the underlying assumptions. \emph{Overall, our theoretical framework constitutes an essential initial step in unraveling how graph structure influences the generalization abilities of MPNNs in node- and link-level prediction tasks. Beyond this, it provides a flexible framework for analyzing generalization in data-dependent settings, guided by the structural properties of the input space through the use of appropriately defined pseudo-metrics.}

\bibliography{bibliography}

\newpage

\appendix
\thispagestyle{empty}
\onecolumn

\section{Related works}
\label{app_sec:rel_work}

In the following, we discuss relevant related work.

\textbf{Graph neural networks} GNNs are a family of neural architectures designed to compute vectorial representations for the nodes of a given graph while encoding structural information about both the graph and its nodes. Recently,  MPNNs~\citep{Gil+2017,Sca+2009} have emerged as the most prominent architecture in graph machine learning. Notable instances of this architecture include, e.g.,~\citet{Duv+2015,Ham+2017,Kip+2017}, and~\citet{Vel+2018}, which can be subsumed under the message-passing framework introduced in~\citet{Gil+2017}. In parallel, approaches based on spectral information were introduced in, e.g.,~\citet{Bru+2014,Defferrard2016,Gam+2019,Kip+2017,Lev+2019}, and~\citet{Mon+2017}---all of which descend from early work in~\citet{bas+1997,Gol+1996,Kir+1995,Mer+2005,mic+2005,mic+2009,Sca+2009}, and~\citet{Spe+1997}. MPNNs have been applied to graph-, node-, and link-level prediction settings~\citep{Cha+2020,Vas+2024}.

\textbf{Node and link prediction using MPNNs} Since MPNNs compute a vectorial representation for each node, utilizing them for node-level prediction is straightforward. However, using them for link prediction is less straightforward. Hence, an extensive set of papers proposing MPNN architectures for link prediction exists~\citep{Ye+2022}. One of the earliest approaches, e.g., \citep{Kip+2016,Sch+2018,Vas+2020}, utilized MPNNs to compute a vectorial representation for each node, which is subsequently used to predict the existence of a link between two nodes. However, such two-stage approaches are not expressive enough for link prediction~\citep{Bal+2020,Zha+2021}. \citet{Zha+2018} proposed an architecture that, for two given nodes $v$ and $w$, computes the union of the subgraphs induced by all nodes within a pre-specified shortest-path distance of either $v$ or $w$, labels nodes with their distances according to $v$ and $w$, respectively, and uses an MPNN on top of this subgraph to predict the existence of a link between $v$ and $w$. \citet{Ter+2020} proposed a related method, taking the intersection instead of the union. \citet{Zha+2021}, introduced the \new{labeling trick} for MPNNs, which, assuming a directed graph, essentially labels the \say{source} and \say{target} nodes with unique labels, and potentially labels the other nodes with specific labels as well, and runs an MPNN on top of this specifically labeled graph to predict the existence of a link between the two nodes. They demonstrated that their labeling trick framework encompasses other architectures, such as those presented by \citet{li2020distance} and \citet{You+2020}. A more refined version of this idea, also leading to improved empirical results, was introduced in \citet{Zhu+2021}, introducing \emph{NBFNet}. See also~\citet{Kon+2022,Cha+2023,Wan+2024,Yun+2021} for more efficient variants. Subsequently, the works of~\citet{Zha+2021,Zhu+2021} were studied theoretically in~\citet{huang2023theory}, showing that their expressive power is upper bounded by a local variant of the $2$-dimensional Weisfeiler--Leman algorithm~\citep{Morris2020b,Bar+2022} and introduced an MPNN with the same expressive power as the former. 

An extensive set of works proposes MPNN architectures for link prediction in knowledge graphs, including embedding and path-based approaches predating MPNN approaches; see~\citet{Ali+2022,Ye+2022} for surveys. 

\textbf{Generalization analysis of MPNNs for node-level prediction} In a first attempt to understand the generalization abilities of MPNNs,~\citet{scarselli2018vapnik} leveraged classical techniques from learning theory~\citep{Kar+1997} to show that MPNNs' \new{VC dimension}~\citep{Vap+95} for node-level prediction tasks with piece-wise polynomial activation functions on a \emph{fixed} graph, under various assumptions, is in $\cO(P^2n\log n)$, where $P$ is the number of parameters and $n$ is the order of the input graph; see also~\citet{Ham+2001}. We note here that~\citet{scarselli2018vapnik} analyzed a different type of MPNN not aligned with modern MPNN architectures~\citep{Gil+2017}; see also~\citet{dinverno2024vc}. Moreover, the work does not account for non-i.i.d.\@ samples, and due to their reliance on VC dimension theory, they are bound to the binary-classification task and the use of the impractical $0$-$1$ loss function. \citet{Ver+2019} derived generalization bounds for node classification tasks, assuming that nodes are sampled in an i.i.d.\@ fashion from the given graph. They utilized the algorithmic stability~\citep{stabilityoriginal} to derive a generalization error bound for a single-layer MPNN layer, demonstrating that the \new{algorithmic stability} property strongly depends on the largest absolute eigenvalue of the graph convolution filter. 

In the transductive setting, building on the \emph{transductive Rademacher average framework} of~\citet{Yan+2009},~\citet{Ess+2021} derived generalization bounds that depend on the maximum norm (maximum absolute row sum) of the graph operator, e.g., the adjacency matrix or the graph's Laplacian; see also~\cite{Tan+2023} for refined results. \citet{Bar+2021a} studied the classification of a mixture of Gaussians, where the data corresponds to the node features of a stochastic block model, deriving conditions under which the mixture model is linearly separable using the GCN layer~\citep{Kip+2017}.

\textbf{Generalization analysis of MPNNs for link-level prediction} There is little work analyzing link prediction architectures' generalization abilities. \citet{Wha+2024} analyzed a large set of transductive link prediction architectures for knowledge graphs using a PAC-Bayesian analysis~\citep{DBLP:conf/colt/McAllester99,DBLP:conf/colt/McAllester03,DBLP:conf/nips/LangfordS02}. However,~\citet{Wha+2024} are restricted to the impractical margin loss, only consider less expressive MPNN-based link prediction architectures, e.g.,~\citet{Sch+2018,Vas+2020}, not considering modern link prediction architectures, and do not account for the influence of graph structure. 

\textbf{Generalization analysis of MPNNs for graph-level prediction} \citet{Gar+2020} showed that the empirical Rademacher complexity (see, e.g.,~\citet{Moh+2018}) of a specific, simple MPNN architecture, using sum aggregation and specific margin loss, is bounded in the maximum degree, the number of layers, Lipschitz constants of activation functions, and parameter matrices' norms. We note here that their analysis assumes weight sharing across layers. Recently,~\cite{Kar+2024} lifted this approach to $E(n)$-equivariant MPNNs~\citep{Sat+2021}. \citet{Lia+2021} refined the results of \citet{Gar+2020} via a PAC-Bayesian approach, further refined in~\citet{Ju+2023}. \citet{Mas+2022,Mas+2024} assumed that data is generated by random graph models, leading to MPNNs' generalization analysis depending on the (average) number of nodes of the graphs. In addition,~\citet{Lev+2023} and~\citet{Rac+2024} defined metrics on attributed graphs, resulting in a generalization bound for MPNNs depending on the covering number of these metrics.  Recently,~\cite{Mor+2023} made progress connecting MPNNs' expressive power and generalization ability via the Weisfeiler--Leman hierarchy. They studied the influence of graph structure and the parameters' encoding lengths on MPNNs' generalization by tightly connecting \new{$1$-dimensional Weisfeiler--Leman algorithm} (\wlone) expressivity and MPNNs' VC dimension. They derived that MPNNs' VC dimension depends tightly on the number of equivalence classes computed by the \wlone{} over a given set of graphs. Moreover, they showed that MPNNs' VC dimension depends logarithmically on the number of colors computed by the \wlone{} and polynomially on the number of parameters. Since relying on the \wlone{}, their analysis implicitly assumes a discrete pseudo-metric space, where two graphs are either equal or far apart. One VC lower bound reported in \citet{Mor+2023} was tightened in \citet{Daniels+2024} to MPNNs restricted to using a single layer and a width of one. In addition,~\citet{Pel+2024} extended the analysis of~\citet{Mor+2023} to node-individualized MPNNs and devised a Rademacher-complexity-based approach using a covering number argument~\cite{Bar+2017}. \citet{Fra+2024} studied the VC dimension of MPNNs, assuming linearly separable data, and demonstrated a tight relationship to the data's margin, which partially explains why more expressive architectures lead to better generalization. \cite{Li+2024} build on the margin-based generalization framework proposed by \cite{Chuang+2021}, which is based on $k$-Variance and the Wasserstein distance. They provide a method to analyze how expressiveness affects the inter- and intra-class concentration of graph embeddings. \citet{Kri+2018} leveraged results from graph property testing~\citep{Gol2010} to study the sample complexity of learning to distinguish various graph properties, e.g., planarity or triangle freeness, using graph kernels~\citep{Borg+2020,Kri+2019}. Most recently, building on the robustness framework of~\citet{Xu+2012},~\citet{Vas+2024} derived the generalization abilities of MPNNs for graph-level predictions by studying different pseudo-metrics that capture MPNNs' computation, thereby improving upon the results in~\citet{Mor+2023}. Finally,~\cite{Yeh+2021} showed negative results for MPNNs' generalization ability to larger graphs.

See~\citet{Vas+2024b} for a survey on generalization analyses of MPNNs and related architectures.

\section{Extended background}
\label{app_sec:background}

\subsection{Detailed notation}
The following provides a detailed summary of our notation.

\textbf{Basic notations} Let $\Nb \coloneq \{ 1, 2, \ldots \}$ and $\Nb_0 \coloneq \Nb \cup \{ 0 \}$. The set $\Rb^+$ denotes the set of non-negative real numbers. For $n \in \Nb$, let $[n] \coloneq \{ 1, \dotsc, n \} \subset \Nb$. We use $\oms \dotsc \cms$ to denote multisets, i.e., the generalization of sets allowing for multiple, finitely many instances for each of its elements. For an arbitrary set $X$, we denote by $2^X$ the set consisting of all possible subsets of $X$. For two non-empty sets $X$ and $Y$, let $Y^X$ denote the set of functions from $X$ to $Y$. Given a set $X$ and a subset $A \subset X$, we define the indicator function $1_A \colon X \to \{0,1\}$ such that $1_A(x) = 1$ if $x \in A$, and $1_A(x) = 0$ otherwise. Let $\vec{M}$ be an $n \times m$ matrix, $n>0$ and $m>0$, over $\Rb$, then $\vec{M}_{i,\cdot}$,  $\vec{M}_{\cdot,j}$, $i \in [n]$, $j \in [m]$, are the $i$th row and $j$th column, respectively, of the matrix $\vec{M}$. Let $\vec{N}$ be an $n \times n$ matrix, $n>0$, then the \new{trace} $\tr(\vec{N}) \coloneq \sum_{i \in [n]} N_{ii}$.  In what follows, $\vec{0}$ denotes an all-zero vector with an appropriate number of components.

\textbf{Norms} Given a vector space $V$, a \new{norm} is a function $\| \cdot \| \colon V \to \Rb^+$ which satisfies the following properties. For all vectors $\vec{u}, \vec{v} \in V$ and scalar $s\in\Rb$, we have (i)~\new{non-negativity,} $\| \vec{v} \| \geq 0$ with $\| \vec{v} \| = 0$ if, and only, if $\vec{v} = \vec{0}$; (ii)~\new{scalar multiplication,} $\| s\vec{v} \| = |s|\, \| \vec{v} \|$; and the (iii)~\new{triangle inequality} holds, $\| \vec{u}+ \vec{v} \| \leq \| \vec{u} \| + \| \vec{v} \|$. When $V$ is some real vector space, say $\Rb^{1 \times d}$, for $d > 0$, here, and in the remainder of the paper,  $\|\cdot\|_1$ and $\|\cdot\|_2$  refer to the \new{$1$-norm} $\|\vec{x}\|_1 \coloneq |x_1|+\cdots+|x_d|$ and \new{$2$-norm} $\|\vec{x}\|_2 \coloneq\sqrt{x_1^2+\cdots+x_d^2}$ , respectively, for $\vec{x}\in\Rb^{1 \times d}$. When considering the vector space $\Rb^{n \times n}$ of square $n \times n$ matrices, a \new{matrix norm} $\|\cdot\| $ is a norm as described above, with the additional property that
$\|\vec{M}\vec{N}\| \leq \|\vec{M}\| \|\vec{N}\|$
for all matrices $ \vec{M} $ and $ \vec{N} $ in $ \Rb^{n \times n}$. Finally, for two vectors $\vec{u} \in \Rb^{d_1}$ and $\vec{v} \in \Rb^{d_2}$, we denote by $\vec{u} \| \vec{v} \in \Rb^{d_1 + d_2}$ the concatenation of the two vectors.

\textbf{Graphs} An \new{(undirected) graph} $G$ is a pair $(V(G),E(G))$ with \emph{finite} sets of \new{nodes} $V(G)$ and \new{edges} $E(G) \subseteq \{ \{u,v\} \subseteq V(G) \mid u \neq v \}$. The \new{order} of a graph $G$ is its number $|V(G)|$ of nodes. We denote the set of all $n$-order (undirected) graphs by $\cG_n$. In a \new{directed graph}, we define $E(G) \subseteq V(G)^2$, where each edge $(u,v)$ has a direction from $u$ to $v$. The \new{chromatic number} $\chi(G)$ of a graph $G$ is the minimum number of colors required to color the nodes of $G$ such that no two adjacent nodes share the same color. Given a directed graph $G$ and nodes $u,v \in V(G)$, we say that $v$ is a \new{child} of $u$ if $(u,v) \in E(G)$. If a node has no children, we refer to this node as a \textit{leaf}. Given a (directed) graph $G$ and nodes $u,v \in V(G)$, we call a path from $u$ to $v$ a tuple $p=(u_1,\ldots,u_{k+1}) \in V(G)^{k}$ such that $(u_i,u_{i+1})\in E(G)$, for all $i \in [k]$, $u_i \neq u_{j}$, for $i\neq j$, $u_1=u$ and $u_{k+1}=v$. We refer to $k$ as the length of the path, and we write $\text{length}(p)=k$. We denote by $\mathcal{P}_G(u,v)$ the set of all possible paths from $u$ to $v$ on a graph $G$, and by $\mathcal{P}^{(k)}(u,v)$ the set of all possible paths from $u$ to $v$ with length $k$ on a graph $G$. Analogously, for undirected graphs by replacing $(u_{i-1},u_i)$ with $\{u_{i-1},u_i\}$. A graph $G$ is called \new{connected} if, for any $u,v \in V(G)$, a path exists from $u$ to $v$. We say that a graph $G$ is \new{disconnected} if it is not connected. For an $n$-order graph $G \in \cG_n$, assuming $V(G) = [n]$, we denote its \new{adjacency matrix} by $\vec{A}(G) \in \{ 0,1 \}^{n \times n}$, where $\vec{A}(G)_{vw} = 1$ if, and only, if $\{v,w\} \in E(G)$. For a graph $G$, the $k$-\new{neighborhood} of a node $v \in V(G)$ denoted by $N^{(k)}_G(v)$ contains all nodes in a path of length at most $k$ from v. When $k=1$, we refer to elements of $N^{(1)}_G(v)$ as neighbors of $v$ and we omit $k$ in the notation. The \new{degree} of a node $v$ is $|N_G(v)|$. When referring to a directed graph, we use the notation $N^{(k)}_{G,\text{out}}(v)$. For $S \subseteq V(G)$, the graph $G[S] \coloneq (S,E_S)$ is the \new{subgraph induced by $S$}, where $E_S \coloneq \{ (u,v) \in E(G) \mid u,v \in S \}$. A \new{(node-)featured graph} is a pair $(G,a_G)$ with a graph $G = (V(G),E(G))$ and a function $a_G \colon V(G) \to \Sigma$, where $\Sigma$ is an arbitrary set. Similarly, we define \new{(edge-)featured graphs}.  For a node $v \in V(G)$, $a_G(v)$ denotes its \new{feature}. We denote the class of all (undirected) graphs with $d$-dimensional, real-valued node features by $\cG_{d}^{\Rb}$. A \new{knowledge graph} is a directed edge-featured graph $G = (V(G),E(G),R)$, where $E(G) \subseteq V \times V \times R$ for a finite set $R$ of relation types. Each edge $(v,x,r) \in E(G)$ is labeled with a relation type $r \in R$. For a node $v \in V$, we define its \new{relation-aware neighborhood} as $N_r(v) \coloneq \{ (x, r) \in V \times R \mid (v, x, r) \in E(G) \}$, where each neighbor $x$ is connected to $v$ by an edge labeled with relation $r$. A \new{query vector} $\vec{w}_q \in \Rb^d$ encodes a query-specific signal that influences how information is passed along edges during message propagation. For each edge $(v,x,r)$, the embedding $\vec{w}_q(v,x,r)$ helps the model focus on paths and structures that are relevant to the query $q$.

\textbf{Trees} A graph $G$ is a \new{tree} if it is connected, but for any $e \in E(G)$ the graph $G\setminus \{e\}$ with $V(G\setminus \{e\})=V(G)$ and $E(G\setminus \{e\})=E(G)\setminus \{e\}$ is disconnected. A tree or a disjoint collection of trees is known as a \new{forest}. A \new{rooted tree} $(T,r)$ or $T_r$ is a tree where a specific node $r \in V(T)$ is marked as the \new{root}. For a rooted (undirected) tree $T_r$, we can define an implicit direction on all edges as pointing away from the root; thus, when we refer to the \new{children} of a node $u$ in a rooted tree, we implicitly consider this directed structure. To distinguish between a rooted tree $T_r$ and its directed counterpart, we use the notation $\overrightarrow{T_r}$ to denote directed rooted trees. For a rooted tree $(T,r)$, and $L\in\Nb$, we define the $L$ level of $(T,r)$ as the set of nodes $v \in V(T)$ satisfying the equation $\min\{\text{length}(p) \mid p \in \mathcal{P}_T(r,u)\}=L$.

\textbf{Graph isomorphisms}
Two graphs $G$ and $H$ are \new{isomorphic} if there exists a bijection $\varphi \colon V(G) \to V(H)$ that preserves adjacency, i.e., $(u,v) \in E(G)$ if, and only, if $(\varphi(u),\varphi(v)) \in E(H)$. In the case of node (or edge)-featured graphs, we additionally require that $a_G(v) = a_H(\varphi(v))$ for $v \in V(G)$ (or $a_G(e) = a_H(\varphi(e))$ for $e \in E(G)$) and for rooted trees, we further require that the root is mapped to the root. Moreover, we call the equivalence classes induced by $\simeq$ \emph{isomorphism types} and denote the isomorphism type of $G$ by $\tau(G)$. A \new{graph class} is a set of graphs closed under isomorphism.

\subsection{Metric spaces and continuity}
\label{app_sec:metric_spaces_continuity}
The following summarizes the foundational concepts of metric spaces and Lipschitz continuity used throughout this work.

\textbf{Metric spaces} In the remainder of the paper, \say{distances} between graphs play an essential role, which we make precise by defining a \new{pseudo-metric} (on the set of graphs). Let $\cX$ be a set equipped with a pseudo-metric $d \colon \cX\times \cX\to\Rb^+$, i.e., $d$ is a function satisfying $d(x,x)=0$ and $d(x,y)=d(y,x)$ for $x,y\in\cX$, and $d(x,y)\leq d(x,z)+d(z,y)$, for $x,y,z \in \cX$. The latter property is called the triangle inequality. The pair $(\cX,d)$ is called a \new{pseudo-metric space}. For $(\cX,d)$ to be a \new{metric space}, $d$ additionally needs to satisfy $d(x,y)=0\Rightarrow x=y$, for $x,y\in\cX$.\footnote{Observe that computing a metric on the set of graphs $\cG$ up to isomorphism is at least as hard as solving the graph isomorphism problem on $\cG$.}

\textbf{Covering numbers and partitions}
Let $(\cX,d)$ be a pseudo-metric space. Given an $\varepsilon>0$, an \new{$\varepsilon$-cover} of $\cX$ is a subset $C\subseteq \cX$ such that for all elements $x\in\cX$ there is an element $y\in C$ such that $d(x,y) \leq \varepsilon$. Given $\varepsilon > 0$ and a pseudo-metric $d$ on the set $\cX$, we define the \new{covering number} of $\cX$,
\begin{equation*}
	\cN(\cX,d,\varepsilon)\coloneqq\min\{m \mid \text{there is an $\varepsilon$-cover of $\cX$ of cardinality $m$} \},
\end{equation*}
i.e., the smallest number $m$ such that there exists a $\varepsilon$-cover of cardinality $m$ of the set $\cX$  with regard to the pseudo-metric $d$.

The covering number provides a direct way of constructing a partition of $\cX$. Let $K \coloneq \cN(\cX,d,\varepsilon)$ so that, by definition of the covering number, there is a subset $\{ x_1,\ldots,x_K \} \subset \cX$ representing an $\varepsilon$-cover of $\cX$.
We define a partition $\{ C_1, \ldots, C_K \}$ where
\begin{equation*}
	C_i \coloneq \{x\in\cX\mid d(x,x_i)=\min_{j\in[K]} d(x,x_j)\},
\end{equation*}
for $i \in [K]$. To break ties, we take the smallest $i$ in the above. Observe that $\cX = \bigcup_{i \in [K]} C_i$. We recall that the \new{diameter} of a set is the maximal distance between any two elements in the set. Implied by the definition of an $\varepsilon$-cover and the triangle inequality, each $C_i$
has a diameter of at most $2\varepsilon$.

\textbf{Continuity on metric spaces} Let $(\cX,d_\cX)$ and $(\cY,d_\cY)$ be two pseudo-metric spaces. A function $ f \colon \cX \to \cY $ is called \new{$c_f$-Lipschitz continuous} if, for $ x,x' \in \cX $,
\begin{equation*}
	d_{\cY} (f(x),f(x'))  \leq c_f \cdot d_{\cX}(x,x').
\end{equation*}

\section{Measure Theory and Random Variables}
\label{app:measurespacesandrandomvariables}

In this section, we provide the necessary background from measure theory. Specifically, we formally define measure spaces, probability spaces, random variables, distributions, and expectations.

\begin{definition}[$\sigma$-algebra]
Let $X$ be a set, and let $2^X = \{A \mid A \subset X\}$ denote its power set. A subset $\cF \subset 2^X$ is called a \emph{$\sigma$-algebra} on $X$ if it satisfies the following properties:
\begin{itemize}
    \item $X \in \cF$,
    \item If $A \in \cF$, then $X \setminus A \in \cF$ (closed under complementation),
    \item If $\{A_n\}_{n \in \Nb}$ is a sequence of sets in $\cF$, then $\bigcup_{n \in \Nb} A_n \in \cF$ (closed under countable unions).
\end{itemize}
If $X$ is a set and $\cF$ is a $\sigma$-algebra on $X$, the pair $(X,\cF)$ is called a \emph{measurable space}.
\end{definition}

Note that for any set $X$, the power set $2^X$ and the set $\{\emptyset, X\}$ are $\sigma$-algebras on $X$. Furthermore, given a set $X$ and a collection $\cF \subset 2^X$, we define $\sigma(\cF)$ as the smallest $\sigma$-algebra on $X$ containing $\cF$. This is well defined, as $2^X$ is always a $\sigma$-algebra. Verifying that a countable union of $\sigma$-algebras is again a $\sigma$-algebra is also straightforward.

\begin{definition}[Measure space]
Given a measurable space $(X, \cF)$, a function $\mu \colon \cF \to \Rb^{+}$ is called a \emph{measure} on $(X, \cF)$ if:
\begin{itemize}
    \item $\mu(\emptyset) = 0$,
    \item If $\{A_n\}_{n \in \Nb}$ is a collection of pairwise disjoint sets in $\cF$, then 
    \begin{equation*}
        \mu\left(\bigcup_{n \in \Nb} A_n \right) = \sum_{n \in \Nb} \mu(A_n).
    \end{equation*}
\end{itemize}
The triplet $(X, \cF, \mu)$ is called a \emph{measure space}.
\end{definition}

If, in addition, $\mu(X) = 1$, then $(X, \cF, \mu)$ is called a \emph{probability space}, and $\mu$ is referred to as a \emph{probability measure}. Below, we formally define random variables and their distributions.

\begin{definition}[Random variable and distribution]
Let $(\Omega, \cF, P)$ be a probability space and $(E, \mathcal{E})$ a measurable space. A function $X \colon \Omega \to E$ is called a \emph{random variable} if it is $\mathcal{E}$-measurable, i.e., 
\begin{equation*}
    X^{-1}(A) \in \cF, \quad \forall A \in \mathcal{E}.
\end{equation*}
Furthermore, the function $P^{X} \colon \mathcal{E} \to [0,1]$ defined by
\begin{equation*}
    P^{X}(A) = P(\{\omega \mid X(\omega) \in A\})
\end{equation*}
is a probability measure on $(E, \mathcal{E})$, and it is called the \emph{distribution} induced by $X$. 
\end{definition}

If two random variables $X$ and $Y$ satisfy $P^{X} \equiv P^{Y}$, then we say that $X$ and $Y$ are \emph{identically distributed}. Additionally, when $E = \Rb^d$ for some $d \in \Nb$, the $\sigma$-algebra $\mathcal{E}$ is implicitly assumed to be the Borel $\sigma$-algebra, denoted by $\mathcal{B}(\Rb^d)$. This is the smallest $\sigma$-algebra containing all open subsets of $\Rb^d$. In any other case where the $\sigma$-algebra is omitted, we implicitly refer to the power set of the space.

Following the notation in the definition of a random variable, given a random variable $X$, we denote by $\sigma(X)$ the smallest $\sigma$-algebra on $E$ such that $X \colon (\Omega, \cF) \to (E, \sigma(X))$ is measurable. Below, we introduce the concept of dependence between random variables. Note that all integrals below refer to the Lebesgue integral.

\begin{definition}[Expectation and conditional expectation]
Let $(\Omega, \cF, P)$ be a probability space, and let $X \colon \Omega \to \Rb$ be a random variable. Given $A \in \cF$ with $P(A) > 0$, we define:
\begin{itemize}
    \item The expectation of $X$:
    \begin{equation*}
        \Eb(X) = \int_{\Omega} X(\omega) P(d\omega).
    \end{equation*}
    If $\Eb(|X|) < \infty$, we say that $X$ is \emph{integrable}.
    
    \item The conditional expectation of $X$ given event $A$:
    \begin{equation*}
        \Eb(X | A) = \frac{1}{P(A)} \int_A X(\omega) P(d\omega).
    \end{equation*}
\end{itemize}
\end{definition}

In the above setting, if $\cF' \subset \cF$ is another $\sigma$-algebra on $X$ and $X$ is integrable, it can be shown that there exists a unique random variable $Y$ satisfying (i) integrability, (ii) $\cF'$-measurability, and (iii) for all $A \in \cF'$, with $P(A)>0$, we have $\Eb(X | A) \coloneq \Eb(Y | A)$.
We denote this random variable as $\Eb[X | \cF']$, which is called the \emph{expectation of $X$ conditioned on $\cF'$}. Consequently, given two random variables $X$ and $Y$, we define the conditional expectation of $X$ given $Y$ as $\Eb(X | Y) = \Eb(X | \sigma(Y))$. 

Another important result concerning the expectation of random variables is the so-called change of measure formula. That is if $X \colon \Omega \to E$ is a random variable and $h \colon E \to \Rb$ is a measurable function, then  
\begin{equation*}  
\int_{\Omega} h(X(\omega)) P(d\omega) = \int_{E} h(x) P^{X}(dx).  
\end{equation*}  
This formula expresses the expectation of $h(X)$ as an integral concerning the pushforward measure $P^X$, rather than the original probability measure $P$. 
Lastly, we state the \emph{tower property}, useful in martingale theory.

\begin{lemma}[Tower property]
Let $(\Omega, \cF, P)$ be a probability space, and let $X \colon \Omega \to \Rb$ be an integrable random variable. If $\cF_1, \cF_2 \subset \cF$ are two $\sigma$-algebras such that $\cF_1 \subset \cF_2$, then
\begin{equation*}
    \Eb \big(\Eb(X | \cF_2) \mid \cF_1 \big) = \Eb(X \mid \cF_1).
\end{equation*}
\end{lemma}

As a corollary, for any two random variables $X$ and $Y$, we have the well-known formula:
\begin{equation*}
    \Eb(\Eb(X | Y)) = \Eb(X).
\end{equation*}

Note that, using the definition of expectation and the above properties, we can derive many useful results in probability theory through the following simple observation: $\Eb(\textbf{1}_{\{X \in A\}})=P(A)$. 

\begin{definition}[Independence]
Let $(\Omega, \cF, P)$ be a probability space, and let $(E_i, \mathcal{E}_i)$ for $i=1,2$ be measurable spaces. Suppose $X_i \colon \Omega \to E_i$ are random variables for $i=1,2$. We say that $X_1$ and $X_2$ are \emph{independent} (denoted as $X_1 \perp X_2$) if
\begin{equation*}
    P(X_1 \in A, X_2 \in B) = P^{X_1}(A) P^{X_2}(B), \quad \forall A \in \mathcal{E}_1, B \in \mathcal{E}_2.
\end{equation*}
Two sets of random variables $S_1$ and $S_2$ are independent if every pair $(X_1, X_2)$ with $X_1 \in S_1$ and $X_2 \in S_2$ is independent.
\end{definition}

\section{Martingales}
\label{app:martingales}

This section introduces martingale theory, which is essential when dealing with non-independent data. We derive a useful concentration inequality, known as the Azuma inequality, which is analogous to Hoeffding's inequality. Recall that Hoeffding's inequality provides a probabilistic bound on the deviation of a sum of independent random variables from its expectation. Similarly, Azuma's inequality establishes a bound for this deviation in cases where the random variables are not independent but exhibit a controlled dependence.

Before formally defining a martingale, we introduce the notion of a filtration on a measurable space. Let $(\Omega, \cF, P)$ be a probability space, and let $\{ \cF_n\}_{n \in \Nb}$ be an increasing sequence of $\sigma$-algebras (i.e., $\cF_n \subset \cF_{n+1}$ for all $n \in \Nb$) such that $\cF_n \subset \cF$. Then, $\{\cF_n \mid n \in \Nb\}$ is called a filtration on $(\Omega, \cF)$.

\begin{definition}
Let $(\Omega, \cF, P)$ be a probability space equipped with a filtration $\{\cF_n\}_{n\in \Nb}$. A sequence of random variables $\{X_n\}_{n \in \Nb} \colon \Omega \to \Rb$ adapted to $\{\cF_n\}_{n \in \Nb}$, meaning that $X_n$ is $\cF_n$-measurable for all $n\in \Nb$, is called a martingale if:
\begin{itemize}
    \item[(i)] $X_n$ is integrable for all $n\in \Nb$.
    \item[(ii)] $\Eb(X_{n+1}|\cF_n) = X_n$ for all $n \in \Nb$.
\end{itemize}
\end{definition}

Below, we define the Doob martingale associated with a given random variable. The Doob martingale (or Lévy martingale) is a sequence of random variables approximating the given random variable while satisfying the martingale property concerning a given filtration.

\begin{definition}[Doob martingale]
Let $(\Omega, \cF, P)$ be a probability space, $\{\cF_n\}_{n \in \Nb}$ be a filtration, and $X \colon \Omega \to \Rb$ be an integrable random variable. We define the sequence $\{W_n\}_{n \in \Nb}$ inductively as follows:
\begin{equation*}
     W_{0} = \Eb(X), \quad \text{and} \quad W_n = \Eb(X \mid \cF_{n}), \text{ for }  n\geq 1.
\end{equation*}
Using the tower property of conditional expectation, it is easy to verify that $\{W_n\}_{n \in \Nb}$ is a martingale concerning the filtration $\{\cF_n\}_{n \in \Nb}$. The sequence $\{W_n\}_{n \in \Nb}$ is called the Doob martingale of $X$ with respect to the filtration $\{\cF_n\}_{n \in \Nb}$.
\end{definition}

The usefulness of the above definition becomes evident when considering the case where $X$ is a function of $n$ random variables. That is, let $Y_1, \ldots, Y_n \colon \Omega \to E$ and $f \colon E^n \to \Rb$. Define $X \coloneq f(Y_1, \ldots, Y_n)$, and let $\{W_n\}$ be the Doob martingale associated with the filtration $\{\cF_n\}_{n \in \Nb}$, where $\cF_k = \cup_{i=1}^{k}\sigma(Y_i)$ for $k \in [n]$, and $\cF_k = \cF_n$ for  $k>n$.

Then, we have that $W_n = f(Y_1,\ldots,Y_n)$ and $W_0 = \Eb(f(Y_1,\ldots,Y_n))$. Therefore, bounding the difference $W_n - W_0$ leads to a concentration inequality. The following result, known as Azuma's inequality \citep{Azuma1967}, describes the exact conditions under which this difference can be controlled.

\begin{theorem}[Azuma's Inequality]
\label{thm:azuma_inequality}
Let $(\Omega, \cF, P)$ be a probability space, and let $\{W_n\}_{n \in \Nb}$ be a martingale adapted to a filtration $\{\cF_n\}_{n \in \Nb}$. Suppose there exist constants $c_k \in \Rb$ for $k \in \Nb$ such that, with probability 1,
\begin{equation*}
\left|W_k - W_{k-1}\right| \leq c_k.
\end{equation*}
Then, for all $N \in \Nb$ and every $t > 0$, we have
\begin{equation*}
P \left( \left|W_N - W_0\right| \geq t \right) \leq 2\exp\left( \frac{-t^2}{2\sum_{k=1}^{N}c_k^2} \right).
\end{equation*}
\end{theorem}

\section{Link prediction MPNNs}
\label{app_sec:MPNN_link_pred}
Here, we provide an overview of state-of-the-art MPNN-based link prediction architectures that serve as the basis for our analysis. While the architectures are presented in the general setting of knowledge graphs, for simplicity, our analysis focuses on homogeneous graphs (i.e., $|R| = 1$), considering only initial node features and ignoring edge features. Throughout this work, we treat all these architectures within the framework of generalized MPNNs, as defined in \cref{sec:generalized_mpnns_and_pms}. For completeness, we present their original formulations below.

\textbf{SEAL and subgraph Information}
\citet{Zha+2018} and \citet{Ter+2020} propose to use subgraph information around the target nodes to compute more expressive architectures. Importantly, these methods allow for inductive sampling from training graphs, as they only require enclosed subgraphs. Formally, let $(G,a_G)$ be an undirected attributed graph, let $u,v \in V(G)$ be two target nodes, and let $k > 0$. The SEAL architecture~\citet{Zha+2018} first extracts the subgraph $S$ induced by all nodes with a shortest-path distance of most $k$ from one of the target nodes. We label the nodes in $S$ using a \new{double-radius node labeling} $d \colon V(S) \to \Rb^2$, where each node $w \in V(S)$ is labeled with its respective distance to the two target nodes $u$ and $v$. For nodes with $d(w,x) = \infty$ or $d(w,y) = \infty$, the corresponding node label is set to $0$. We then run an MPNN on top of the resulting node-labeled graph to compute a vectorial representation for the two target nodes. The idea was later generalized to the \new{labeling trick}~\citep{Zha+2021}, allowing for arbitrary node and edge labeling and more expressive architectures.

\textbf{Conditional message-passing neural networks}
To study the expressive power of link prediction architectures, \citet{huang2023theory} proposed conditional message-passing neural networks (C-MPNNs), which cover architectures such as NBFNet~\citep{Zhu+2021}, NeuralLP~\citep{yang2017differentiable}, and DRUM~\citep{sadeghian2019drum}. Given a (knowledge) graph $G$, C-MPNNs compute pairwise node representations in the graph $G$ for a fixed query vector $q$ and source node $u \in V(G)$. The computation of all pairs $(u,v)$ is conditioned on the source node $u$ to provide a node representation of $v$ conditioned by $u$. Hence, a conditional vectorial node representation $\hb^{(t+1)}_{v \mid u,q} \in \Rb^d$ is computed for $L$ layers by using the conditional message-passing architecture,
\begin{equation}
\label{eq:CMPNN}
\hb^{(t+1)}_{v \mid u,q} \coloneq \UPD \Big(\hb_{v \mid u,q}^{f(t)}, \AGG \big(\oms \mathsf{MSG}_r(\hb_{x \mid u,q}^{(t)}, \vec{w}_q) \mid x \in N_r(v), r \in R  \cms \Big),
\end{equation}
where $\hb_{v \mid u,q}^0 = \mathsf{INIT}(u,v,q)$ denotes an initialization function satisfying  node distinguishability; i.e., $\mathsf{INIT}(u,v) \neq \mathsf{INIT}(u,u)$ for $u \neq v$. Moreover, the function $f \colon \Nb \to \Nb$ is usually set to the identity function. By setting $f(t) = 0$ in \cref{eq:CMPNN}, one obtains the Neural Bellman--Ford Networks (NBFNets), as defined by \citet{Zhu+2021}, as a special case.

\textbf{Neural common neighbors}
Since SEAL, C-MPNN, and other subgraph-based architectures for link prediction propose labeling a graph and then applying an MPNN to evaluate node representations, which leads to poor scalability, \citet{wang2024neural} proposed reversing this process. For this, they devise the \new{neural common neighbors framework} (NCN), offering an approach of using an MPNN architecture on the original graph and enhancing the subsequent pooling process using structural information.  For two nodes $i,j$, the node representation is computed using a two-step process. In the first step, an $L$-layer MPNN computes vectorial representations $\hb_i$ and $\hb_j$ for the nodes $i$ and $j$. In contrast to SEAL and other subgraph-based architectures, the MPNN computation is now concatenated with the node representations from the common $k$-hop neighborhoods of the two nodes to compute the target tuple representation $\vec{z}_{ij}$. Furthermore, the node representations are aggregated using the Hadamard product $\odot$,
\begin{equation*}
\vec{z}_{ij} \coloneqq \vec{h}_i \odot \vec{h}_j \Bigr\| \sum_{u \in N_G^k(i) \cap N_G^k(j)} \hb_u.
\end{equation*}
This tuple representation is calculated for each proposed link and evaluated using a feed-forward neural network, resulting in representations of each target link tuple $(i,j)$.

\textbf{Feed-forward neural networks}
\label{fnn} An $L$-layer \new{feed-forward neural network} (FNN), for $L \in \Nb$,
is a parametric function $\FNN^\tup{L}_{\mathbold{\theta}} \colon \Rb^{1\times d} \to \Rb$, $d>0$, where $\mathbold{\theta} \coloneqq (\vec{W}^{(1)}, \ldots, \vec{W}^{(d)}) \subseteq \Theta$ and $\vec{W}^{(i)} \in \Rb^{d \times d}$, for $i \in [L-1]$, and $\vec{W}^{(L)} \in \Rb^{d \times 1}$, where
\begin{equation*}
	\vec{x} \mapsto \sigma \Bigl( \cdots  \sigma \mleft(  \sigma \mleft(\vec{x}\vec{W}^{(1)} \mright) \vec{W}^{(2)} \mright) \cdots \vec{W}^{(L)} \Bigr) \in \Rb,
\end{equation*}
for $\vec{x} \in \Rb^{1\times d}$. Here, the function $\sigma \colon \Rb \to \Rb$ is an \new{activation function}, applied component-wisely, e.g., a \emph{rectified linear unit} (ReLU), where $\sigma(x) \coloneqq \max(0,x)$. For an FNN where we do not need to specify the number of layers, we write $\FNN_{\mathbold{\theta}}$.

\subsection{Link prediction as generalized MPNNs}
\label{app_sec:link_mpnns_as_generalizedMPNNs}
Below, we describe how the link prediction MPNN architectures discussed above, i.e., SEAL, C-MPNN, and NCN, can be viewed as special cases of the generalized MPNN framework introduced in \cref{sec:generalized_mpnns_and_pms}.

\textbf{SEAL} Verifying that SEAL constitutes a special case of generalized MPNNs as defined above is straightforward. To establish this, we must define the transformation function $\cT$, the selection function $\cV$, and the pooling function $\Psi$ appropriately. For simplicity, we assume a directed graph $G$ without initial vertex features. For each pair $(u,v) \in V(G)^2$, with $u \neq v$, we define the node set $V_{uv} \coloneq N_G(u) \cup N_G(v) \cup \{u,v\}$. We then compute the induced subgraph $\bar{G}_{uv} = G[V_{uv}]$. Vertex features are added to $\bar{G}_{uv}$ using the double-radius node labeling described in~\citet{Zha+2018}, yielding the graph $G_{uv}$. We define the transformation and selection functions as follows,
\begin{align*}
&\cT(G) =\dot{\cup}_{(u,v) \in V(G)^2} G_{uv}, \text{ for } u\neq v, \quad \text{and}, \\
&\cV(G, (u,v)) = V(G_{uv}),
\end{align*}
where $\dot{\cup}_{(u,v) \in V(G)^2} G_{uv}$ denotes the disjoint union of graphs. Finally, we use sum-pooling as our pooling function $\Psi$.

\textbf{C-MPNNs} We show that C-MPNNs constitute a special case of generalized MPNNs. For simplicity, we again assume a graph $G$ without vertex features and set $f(t) = 1$ in \cref{eq:CMPNN}. Additionally, we assume that the graph contains only one type of edge, i.e., $|R| = 1$. We construct a new graph $G'$ with vertex set $V(G') = V(G)^2$, and define initial vertex features using the $\text{INIT}$ function such that $\text{INIT}(u,u) \neq \text{INIT}(u,v)$ for all $u \neq v$. The following condition defines the edge relation in $G'$, 
\begin{equation*}
\{(u,v), (x,y)\} \in E(G') \iff u = x \text{ and } y \in N_G(v).
\end{equation*}
The transformation and selection functions are then defined as:
\begin{align*}
&\cT(G) = G', \quad \text{and} \\
&\cV(G, (u,v)) = \{ (u,w) \mid w \in N_G(v) \}.
\end{align*}
Again, we use sum-pooling as our pooling function $\Psi$.  According to the definition of C-MPNNs, we can slightly modify the transformation and selection functions to capture the NBFNets. 

\textbf{Neural common neighbors} For the neural common neighbors architecture, we have $\cG' \otimes \mathfrak{R} = \{(G,E(G)) \mid G \in \cG_{d}^{\Rb}, \{u,v\} \in E(G)\}$ and we set,
\begin{align*}
& \mathcal{T}(G) \coloneqq G, \quad \text{and}\\
& \mathcal{V}(G,\{u,v\}) \coloneqq \{u,v\} \cup \left( N_G^k(u) \cap N_G^k(v) \right).
\end{align*}
The pooling function is then given by 
\begin{equation}
\label{eq:CNN_pooling}
\Psi(F(G, \{u,v\})) \coloneqq \Bigl(\hb_{\mathcal{T}(G)}^{(L)}(u) \odot  \hb_{\mathcal{T}(G)}^{(L)}(v) , \sum_{x \in N_G^k(u) \cap N_G^k(v)} \hb_{\mathcal{T}(G)}^{(L)}(x)\Bigr).
\end{equation}
From this, we arrive at the MPNN model defined by~\citet[Equation~(9)]{wang2024neural}.

\subsection{The sub-sum pooling property}
\label{app_sec:sub_sum_pooling}
Here we define a useful property of the pooling function $\Psi$, which generalizes sum-pooling while preserving key properties relevant to our later analysis. 

\begin{definition}[sub-sum pooling]
\label{def:sub-sum_pooling}
Let $(\cT,\cV,\Psi)$-MPNN$(L)$ be a generalized MPNN as previously described. We say that $\Psi$ has the \emph{sub-sum property} with parameter $C$ if there exists a constant $C > 0$ such that for all $(G_1,S_1), (G_2,S_2) \in \cG_{d}^{\Rb} \otimes \mathfrak{R}$, with $|\cV(G_1,S_1)| > |\cV(G_2,S_2)|$, and for all surjective mappings $\sigma \colon \cV(G_1,S_1) \to \cV(G_2,S_2) \cup \{*\}$ satisfying
\begin{equation}
\label{eq:extended_bijection_property}
    u_1 \neq u_2, \quad \sigma(u_1) = \sigma(u_2) \implies \sigma(u_1) = \sigma(u_2) = *,
\end{equation}
we have
\begin{equation}
\label{eq:sub-sum_property}
    \| \Psi(F(G_1,S_1)) - \Psi(F(G_2,S_2)) \|_2 \leq C \cdot \sum_{x \in \cV(G_1,S_1)} \| \hb_{\cT(G_1)}^{(L)}(x)  - \hb_{\cT(G_2)}^{(L)}(\sigma(x)) \|_2,
\end{equation}
where $\hb_{\cT(G_2)}^{(L)}(*) \coloneq 0_{\Rb^d}$. In the case where $|\cV(G_1,S_1)| = |\cV(G_2,S_2)|$, we replace $\sigma$ with a bijection $\cV(G_1,S_1) \to \cV(G_2,S_2)$.

Additionally, for any two sets $A,B$ with $|A| > |B|$, we refer to a function $A \to B \cup \{*\}$ satisfying \cref{eq:extended_bijection_property} as an \emph{extended bijection} between $A$ and $B$.
\end{definition}

The sub-sum property plays a central role in our theoretical analysis. The above architectures employ pooling functions such as sum aggregation, Hadamard product, and concatenation. We now show that each of these pooling operators satisfies the sub-sum property.

\begin{proposition}
\label{prop:sub_sum_propety}
Let $\Psi_1$, $\Psi_2$, and $\Psi_3$ be pooling functions defined as follows,
\begin{align*}
&\Psi_1(F(G, \mathcal{V}(G))) \coloneq \sum_{x \in \mathcal{V}(G)} \hb_{G}^{(L)}(x), \\
&\Psi_2(F(G, \{x,y\})) \coloneq \hb_{G}^{(L)}(x) \odot \hb_{G}^{(L)}(y), \\
&\Psi_3(F(G, \{x,y\})) \coloneq (\hb_{G}^{(L)}(x), \hb_{G}^{(L)}(y)).
\end{align*}
Then $\Psi_1$, $\Psi_2$, and $\Psi_3$ satisfy the sub-sum property for all $(G_1, S_1), (G_2, S_2) \in \cG_{d}^{\Rb} \otimes \mathfrak{R}$, for a suitable choice of $\mathfrak{R}$.
\end{proposition}

\begin{proof}
We separate the proof into two steps: First, we show that sum aggregation and the Hadamard product have the sub-sum property. Then, we conclude the same result for the concatenation operator. 

Let $\Psi_1$ be the sum aggregation pooling. Further, we consider $\sigma$ an extended bijection as shown in Definition \ref{def:sub-sum_pooling}. Then the sum aggregation can be applied to Equation \ref{eq:sub-sum_property}:
\begin{align*}
\| \Psi_1(F(G_1,S_1)) - \Psi_1(F(G_2,S_2)) \|_2 = \| \sum_{x\in \mathcal{V}(G_1,S_1)} \hb_{G_1}^{(L)}(x) - \sum_{x'\in \mathcal{V}(G_2,S_2)} \hb_{G_2}^{(L)}(x') \|_2 
\end{align*}

Since we know $\sigma$ to be an extended bijection or a bijection depending on $\mathcal{V}(G_1,S_1), \mathcal{V}(G_2,S_2)$ it follows: 
\begin{align*}
\| \sum_{x\in \mathcal{V}(G_1,S_1)} \hb_{G_1}^{(L)}(x) - \sum_{x'\in \mathcal{V}(G_2,S_2)} \hb_{G_2}^{(L)}(x') \|_2  = \| \sum_{x\in \mathcal{V}(G_1,S_1)} \hb_{G_1}^{(L)}(x) -  \hb_{G_2}^{(L)}(\sigma(x)) \|_2 \\
\leq \sum_{x\in \mathcal{V}(G_1,S_1)} \|  \hb_{G_1}^{(L)}(x) -  \hb_{G_2}^{(L)}(\sigma(x)) \|_2
\end{align*}

The last inequality results from applying the triangle inequality to the sum of the two terms. From this, the sub-sum property directly follows with $C=1$. 

For the Hadamard product $\odot$ we consider additional notation first. Given an encoding vector $\hb_{G_1}^{(L)}(x) \in \mathbb{R}^n$ we denote the $i$-th element of such vector with $x_i$. We further note the following inequality: 
\begin{align*}
\lvert x_i y_i - x_i'y_i'\rvert = \lvert (x_i - x_i') y_i + x_i' (y_i - y_i') \rvert \leq \lvert x_i - x_i'\rvert \lvert y_i \rvert +  \lvert y_i - y_i'\rvert \lvert x_i' \rvert
\end{align*}
Using this inequality, we can derive an expression for $\lvert x_i y_i - x_i'y_i'\rvert^2$:
\begin{align*}
\lvert x_i y_i - x_i'y_i'\rvert^2 \leq \lvert x_i - x_i'\rvert^2 \lvert y_i \rvert^2 +  \lvert y_i - y_i'\rvert^2 \lvert x_i' \rvert^2 + 2 \lvert x_i \rvert \lvert y_i \rvert \lvert x_i - x_i' \rvert \lvert y_i - y_i' \rvert   \\
\leq 2 \lvert y_i \rvert^2 \lvert x_i - x_i'\rvert^2 + 2 \lvert x_i'\rvert^2 \lvert y_i - y_i'\rvert^2  
\end{align*}
We now consider the Hadarmard product for $\mathcal{V}(G_1,S_1) = \{x,y\}$, $\mathcal{V}(G_2,S_2) = \{x',y'\}$ and $\sigma$ as an extended bijection mapping $x$ to $x'$ and $y$ to $y'$. Further $\mathcal{T}$ denotes a transformation function:
\begin{align*}
\| \Psi_2(F(G_1,S_1)) - \Psi_2(F(G_2,S_2)) \|_2 = \| 
\hb_{\mathcal{T}(G_1)}^{(L)}(x) \odot \hb_{\mathcal{T}(G_1)}^{(L)}(y) -  \hb_{\mathcal{T}(G_2)}^{(L)}(x') \odot \hb_{\mathcal{T}(G_2)}^{(L)}(y') \|_2 
\end{align*}
\begin{align*}
\| 
\hb_{\mathcal{T}(G_1)}^{(L)}(x) \odot \hb_{\mathcal{T}(G_1)}^{(L)}(y) -  \hb_{\mathcal{T}(G_2)}^{(L)}(x') \odot \hb_{\mathcal{T}(G_2)}^{(L)}(y') \|_2^2 = \sum_{i=1}^{n}(x_iy_i)^2 - \sum_{i=1}^{n}(x'_iy'_i)^2 = \sum_{i=1}^{n}\lvert x_i y_i - x_i'y_i'\rvert^2
\end{align*}
With our previously obtained result for $\lvert x_i y_i - x_i'y_i'\rvert^2$ we derive an upper bound for the Hadarmard product difference, assuming $\|\hb_{G}^{(L)}(x) \|_{\infty} \leq b$ for some $b > 0$:
\begin{align*}
\|\hb_{\mathcal{T}(G_1)}^{(L)}(x) \odot \hb_{\mathcal{T}(G_1)}^{(L)}(y) -  \hb_{\mathcal{T}(G_2)}^{(L)}(x') \odot \hb_{\mathcal{T}(G_2)}^{(L)}(y') \|_2^2 \leq \sum_{i=1}^{n} 2 \lvert y_i \rvert^2 \lvert x_i - x_i'\rvert^2 + 2 \lvert x_i'\rvert^2 \lvert y_i - y_i'\rvert^2 \\
=2 \sum_{i=1}^{n}  \lvert y_i \rvert^2 \lvert x_i - x_i'\rvert^2 + \lvert x_i'\rvert^2 \lvert y_i - y_i'\rvert^2 \leq 2b (\|\hb_{\mathcal{T}(G_1)}^{(L)}(x) - \hb_{\mathcal{T}(G_2)}^{(L)}(x')\|_2^2 + \|\hb_{\mathcal{T}(G_1)}^{(L)}(y) - \hb_{\mathcal{T}(G_1)}^{(L)}(y')\|_2^2)
\end{align*}
This implies the following sub-sum property for the Hadarmard product with $C= \sqrt{2b}$:
\begin{align*}
\| 
\hb_{\mathcal{T}(G_1)}^{(L)}(x) \odot \hb_{\mathcal{T}(G_1)}^{(L)}(y) -  \hb_{\mathcal{T}(G_2)}^{(L)}(x') \odot \hb_{\mathcal{T}(G_2)}^{(L)}(y') \|_2 \\ \leq \sqrt{2b}(\|\hb_{\mathcal{T}(G_1)}^{(L)}(x) - \hb_{\mathcal{T}(G_2)}^{(L)}(x')\|_2 + \|\hb_{\mathcal{T}(G_1)}^{(L)}(y) - \hb_{\mathcal{T}(G_1)}^{(L)}(y')\|_2)
\end{align*}
In the last step, we consider the concatenation operator $\Psi_3$. Since the concatenation of two vectors $x,y \in \mathbb{R}^n$ is given by $(x_1,\dots,x_n,y_1,\dots,y_n)  = (x,y) \in \mathbb{R}^{2n}$ we can separate the Euclidean norm:
\begin{align*}
\| (x,y) - (x',y')\|_2 = \|(x-x') + (y-y') \|_2 \leq \| x-x' \|_2 + \|y-y' \|_2
\end{align*}
From this, the sub-sum property directly follows. 
\end{proof}

Since the Hadamard product and the concatenation operator have the sub-sum property, it is easy to verify that the property also holds for the pooling function used in neural common neighbors \cref{eq:CNN_pooling}.

\section{The \texorpdfstring{$1$}{1}-dimensional Weisfeiler--Leman algorithm} 
\label{app_sec:1WL}

The \new{$1$-dimensional Weisfeiler--Leman algorithm} (\wlone) or \new{color refinement} is a well-studied heuristic for the graph isomorphism problem, originally proposed by~\citet{Wei+1968}.\footnote{Strictly speaking, the \wlone{} and color refinement are two different algorithms. That is, the \wlone{} considers neighbors and non-neighbors to update the coloring, resulting in a slightly higher expressive power when distinguishing nodes in a given graph; see~\cite {Gro+2021} for details. Following customs in the machine learning literature, we consider both algorithms equivalent.} Intuitively, the algorithm determines if two graphs are non-isomorphic by iteratively coloring or labeling nodes. Given an initial coloring or labeling of the nodes of both graphs, e.g., their degree or application-specific information, in each iteration, two nodes with the same label get different labels if the number of identically labeled neighbors is unequal. These labels induce a node partition, and the algorithm terminates when, after some iteration, the algorithm does not refine the current partition, i.e., when a \new{stable coloring} or \new{stable partition} is obtained. Then, if the number of nodes annotated with a specific label differs between the two graphs, we can conclude that the two graphs are not isomorphic. It is easy to see that the algorithm cannot distinguish all non-isomorphic graphs~\citep{Cai+1992}. However, it is a powerful heuristic that can successfully decide isomorphism for a broad class of graphs~\citep{Arv+2015,Bab+1979}.

Formally, let $(G,a_G)$ be a node-featured graph. In each iteration, $t > 0$, the \wlone{} computes a \new{node coloring} $C^1_t \colon V(G) \to \Nb$, depending on the coloring of the neighbors. That is, in iteration $t>0$, we set
\begin{equation*}
	C^{1}_t(v) \coloneq \REL\Bigl(\!\bigl(C^{1}_{t-1}(v),\oms C^{1}_{t-1}(u) \mid u \in N(v)  \cms \bigr)\! \Bigr),
\end{equation*}
for node $v \in V(G)$, where $\REL$ injectively maps the above pair to a unique natural number, which has not been used in previous iterations. In iteration $0$, the coloring $C^1_{0}\coloneqq a_G$ is used.\footnote{Here, we implicitly assume an injective function from $\Sigma$ to $\Nb$.} To test whether two graphs $G$ and $H$ are non-isomorphic, we run the above algorithm in ``parallel'' on both graphs. If the two graphs have a different number of nodes colored $c \in \Nb$ at some iteration, the \wlone{} \new{distinguishes} the graphs as non-isomorphic. Moreover, if the number of colors between two iterations, $t$ and $(t+1)$, does not change, i.e., the cardinalities of the images of $C^1_{t}$ and $C^1_{i+t}$ are equal, or, equivalently,
\begin{equation*}
	C^{1}_{t}(v) = C^{1}_{t}(w) \iff C^{1}_{t+1}(v) = C^{1}_{t+1}(w),
\end{equation*}
for all nodes $v,w \in V(G\,\dot\cup H)$, then the algorithm terminates. For such $t$, we define the \new{stable coloring} $C^1_{\infty}(v) = C^1_t(v)$, for $v \in V(G\,\dot\cup H)$. The stable coloring is reached after at most $\max \{ |V(G)|,|V(H)| \}$ iterations~\citep{Gro2017}.

\textbf{Unrollings}
Following~\citet{Morris2020b}, given a (node)-featured graph $(G,a_G)$, we define the \new{unrolling tree} of depth $L \in \Nb_0$ for a node $u \in V(G)$, denoted as $\UNR{G,u,L}$, inductively as follows.
\begin{enumerate}
	\item For $L=0$, we consider the trivial tree as an isolated node with feature $a_G(u)$.
	\item For $L>0$, we consider the root node with label $a_G(u)$ and, for $v \in N(u)$, we attach the subtree $\UNR{G,v,L-1}$ under the root.
\end{enumerate}
The above unrolling tree construction characterizes the $\wlone$ algorithm through the following lemma.
\begin{lemma}[Folklore, see, e.g.,~\citet{Mor+2020}]
	\label{wlone:unrollingschar}
	The following are equivalent for $L \in \Nb_0$, given a featured graph $(G,a_G)$ and nodes $u,v \in V(G)$.
	\begin{enumerate}
		\item The nodes $u$ and $v$ have the same color after $L$ iterations of the \wlone.
		\item The unrolling trees $\UNR{G,u,L}$ and $\UNR{G,v,L}$ are isomorphic. 
	\end{enumerate}
\end{lemma}

Note that for an edge-featured graph, we can analogously define a $\wlone$-variant and an unrolling tree (with edge features), satisfying a similar version of \cref{wlone:unrollingschar}.

\section{Learning frameworks}
\label{app_sec:learning_formalism}

This appendix fully formalizes the inductive and transductive learning frameworks introduced in \cref{sec:learningongraphs}. We specify the underlying probability spaces, data generation models, symmetry assumptions, and loss functions that underpin our generalization analysis. The distinction between these frameworks lies in the way they leverage available information. Inductive learning involves training models to infer general patterns from observed data, enabling predictions on unseen instances. This ability to generalize beyond the training set is at the core of most supervised learning methods. In contrast, transductive learning focuses on deriving predictions directly for specific, unlabeled data points, as seen in semi-supervised and few-shot learning, without constructing an explicit general model. In the subsequent sections, we formalize the statistical learning framework for any machine learning task, define the generalization error in both settings, and describe how node and link prediction tasks can be addressed using either an inductive or transductive approach.

\textbf{Inductive statistical learning} Let $(\Omega, \cF,P)$ be a probability space $\cX$ be an arbitrary input space and $\cY$ the label space. For example, in \emph{binary classification}, the label set is typically represented as $\{0,1\}$. Let $\cZ$ denote the product space $\cX \times \cY$. In supervised learning, we have access to a finite set of \emph{training points},
\begin{equation*}
    \cS \coloneq \{(x, y) \mid (x, y) \in \cX \times \cY\},
\end{equation*}
where each $(x_i, y_i)_{i=1}^{N}$ is sampled according to the same distribution $\mu$ on $\cZ$, and $|\cS|=N \in \Nb$. In the special case where each $(x_i, y_i) \in \cS$ is independently drawn, we say that we have an \emph{i.i.d.\@ sampling process}.

We consider a family of mappings from $\cX$ to $\cY$, denoted by $\cH$, often referred to as the \emph{hypothesis class}, and its elements as \new{concepts}. To formally measure the quality of a concept, we define a \new{loss function} $\ell \colon \cH \times \cX \times \cY \to \Rb^{+}$ bounded by some $M \in \Nb$. The goal then is to choose a function $h \in \cH$ that minimizes the expectation $\Eb_{\mu}[\ell(h, X, Y)]$.

A \emph{learning algorithm} $\cA$ is a procedure that takes the training data $\cS$ as input and outputs a concept $h \in \cH$; we usually denote $h_{\cS}\coloneq \cA_{\cS}$. Hence, the chosen concept $h \in \cH$ also depends on the training data $\cS$. For a given sample $\cS$ of cardinality $N$, we write
\begin{equation*}
    \ell_{\text{exp}}(\cA_{\cS}) \coloneqq \Eb_{\mu}\left(\ell(\cA_{\cS}, X, Y)\right),
\end{equation*}
to denote the \emph{expected risk}. However, since the underlying distribution $\mu$ is unknown, we usually resort to minimizing the \emph{empirical risk} for a given sample $\cS$, i.e.,
\begin{equation*}
    \ell_{\text{emp}}(\cA_{\cS}) \coloneqq  \frac{1}{N}\sum_{(x,y) \in \cS} \ell(\cA_{\cS},x,y).
\end{equation*}
Given the above notation, the \new{generalization error} of a hypothesis $h \in \cH$ is defined as the absolute difference of the above two quantities, i.e., 
\begin{equation*}
|\ell_{\mathrm{emp}}(h) - \ell_{\mathrm{exp}}(h)|.
\end{equation*}

\textbf{Transductive statistical learning} Again, let $(\Omega, \cF, P)$ be a probability space. As in the inductive setting, we are given pairs $(x,y)$ of labeled data where $x\in \cX$ and $y\in \cY$, for some arbitrary input space $\cX$ and label space $\cY$, and we similarly define a hypothesis class $\cH$ and a loss function $\ell$. However, in the transductive setting, we do not assume an unknown underlying distribution for generating the data, as we know the inputs for which we want to predict labels from the beginning, allowing us to utilize them throughout the learning process. 

Thus, we consider the space $\cZ \coloneqq \{(x_i,y_i)\}_{i=1}^{m+u}$, where $n=m+u$, consisting of pairs $(x,y) \in \cX \times \cY$. We can access $m$ labeled pairs and the remaining $u$ unlabeled inputs. The objective is to choose a function $h \in \cH$ that minimizes the average loss on the $u$ unlabeled samples.

Since we have not assumed an underlying distribution, we must introduce randomness into our sampling process in some way. Following \citet[Setting 1]{DBLP:books/sp/Vapnik06}, we model the sampling process by assuming that the $m$ labeled training examples are sampled \emph{uniformly without replacement} from the whole space $\cZ$.

We may assume that we first choose a permutation $\pi \colon [n] \to [n]$ uniformly at random from the set of all possible $n!$ permutations. We then order the space $\cZ$ according to this permutation, obtaining
\begin{equation*}
    \cZ_{\pi} = (x_{\pi(i)}, y_{\pi(i)})_{i=1}^{m+u}.
\end{equation*}
We select the first $m$ elements from this ordered sequence as our labeled training set and the remaining $u$ elements as the unlabeled test set. Thus, our training set is defined as
\begin{equation*}
    \cS_{\pi} = \{(x_{\pi(i)}, y_{\pi(i)})\}_{i=1}^{m} \cup \{x_{\pi(i)}\}_{i=m+1}^{m+u}.
\end{equation*}
The only source of randomness in this setting comes from the random variable $\pi \colon \Omega \to S_n$, where $S_n = \{ \rho \colon [n] \to [n] \mid \rho \text{ is a bijection} \}$ denotes the symmetric group, with
\begin{equation*}
    P( \pi \in A ) = \frac{|A|}{n!}, \quad \text{for } A \subset S_n.
\end{equation*}
We write $\pi \sim \text{Unif}(S_n)$. A \emph{transductive learning algorithm} $\cA$ is then defined as a function $S_n \to \cH$, and the generalization gap is given by the difference 
\begin{equation*}
|R_m-R_u|,
\end{equation*}
where,
\begin{equation*}
    R_m \coloneq \frac{1}{m} \sum_{i=1}^{m} \ell(\cA_{\pi}, x_{\pi(i)}, y_{\pi(i)}) \quad \text{and} \quad R_u \coloneq \frac{1}{u} \sum_{i=m+1}^{m+u} \ell(\cA_{\pi}, x_{\pi(i)}, y_{\pi(i)}),
\end{equation*}
where $\cA_{\pi}$ depends on $\pi$ through $\cS_{\pi}$.

\textbf{Remarks}
\begin{enumerate}
    \item This setting differs significantly from the i.i.d.\@ case, as the samples are not independent. Specifically, since we sample without replacement, each training sample depends on the others. Consequently, concentration inequalities for independent random variables or graph-dependent data, such as \cref{thm:Janson's_concentration}, are not directly applicable.
    
    \item In this setting, the underlying distribution and dependencies between samples are explicitly known. These dependencies should be exploited to derive generalization bounds.
    
    \item Considering our goal, it is reasonable to focus on learning algorithms that exhibit symmetry concerning the arguments $\{(x_{\pi(i)}, y_{\pi(i)})\}_{i=1}^{m}$ for any permutation $\pi \in S_n$. This symmetry ensures that the expected values satisfy the following conditions:
        \begin{align*}
        & \Eb_{\pi}[\ell(\cA_{\pi}, x_{\pi(i)}, y_{\pi(i)})] = \Eb_{\pi}[\ell(\cA_{\pi}, x_{\pi(1)}, y_{\pi(1)})], \text{ for } i\in [m] \\
        & \Eb_{\pi}[\ell(\cA_{\pi}, x_{\pi(m+j)}, y_{\pi(m+j)})] = \Eb_{\pi}[\ell(\cA_{\pi}, x_{\pi(m+1)}, y_{\pi(m+1)})], \text{ for } j\in [u].
    \end{align*}
    Implying,
    \begin{align*}
        &\Eb_{\pi}[R_m] = \Eb_{\pi}[\ell(\cA_{\pi}, x_{\pi(1)}, y_{\pi(1)})], \\
        &\Eb_{\pi}[R_u] = \Eb_{\pi}[\ell(\cA_{\pi}, x_{\pi(m+1)}, y_{\pi(m+1)})].
    \end{align*}
\end{enumerate}

\textbf{Learning on Graphs}
We specify below how the statistical learning approaches discussed earlier (inductive and transductive) apply to graph-learning tasks. More precisely, we focus on node and link prediction, the primary tasks analyzed in this work. However, the setting can be naturally extended to any graph-representation learning task.

In the inductive setting for node prediction, the input space is given by $\cX \coloneqq \cG' \otimes V$, where $\cG' \subset \cG_{d}^{\Rb}$, and the label space is $\cY = \{0,1\}$ for classification or $\Rb$ for regression. Thus, the dataset consists of triplets $(G, u, y)$, where $G \in \cG'$, $u \in V(G)$, and $y \in \cY$, sampled according to a distribution $\mu$, as described earlier. Possible dependencies between samples will be specified based on our assumptions, as detailed in \cref{sec:robustness_under_dependency}. A natural assumption regarding the dependency structure in the dataset is that two sampled triplets $(G_1, u_1, y_1)$ and $(G_2, u_2, y_2)$ are possibly dependent when $G_1 = G_2$. This follows from the assumption that the underlying distribution first generates the entire graph with all node labels, and then only a subset of these labels is revealed to the learner. Similarly, the input space is given by $\cX \coloneqq \cG' \otimes E$ for link prediction.

In the transductive setting for node prediction, we assume that we are given a fixed graph $G \in \cG_{d}^{\Rb}$ with node set $V(G) = \{u_1, \dots, u_n\}$, where $n = m + u$, and that the labels for $m$ nodes in $V(G)$ are provided. Thus, we define the dataset as $\cZ \coloneqq \{(G, u_i, y_i)\}_{i=1}^{m+u}$. When we say the graph is given, we implicitly assume that the initial node features are also provided. Similarly, in the transductive framework for link prediction, we define $\cZ = \{(G, e_i, y_i)\}_{i=1}^{m+u}$, where each $(G, e_i)$ belongs to $G \otimes E \coloneq \{(G, e) \mid e = \{u,v\} \subset V(G) \}$

\textbf{Cross-entropy loss function}
Finally, we note that the quantities $\ell_{\mathrm{exp}}, \ell_{\mathrm{emp}}, R_u$, and $R_m$ can be similarly defined when the hypothesis class $\cH$ consists of functions mapping from $\cX$ to $\Rb$ (rather than $\{0,1\}$), as long as the loss function $\ell \colon \cH \times \cZ \to \Rb^{+}$ is properly defined. In this work, all $L$-layer generalized MPNNs are assumed to map inputs to vectors $\hb \in \Rb^{d_L}$, which are then processed by a trainable feedforward neural network (FNN) producing an output $x \in \Rb$. We use the binary cross-entropy loss with a sigmoid activation to compute the loss between $x$ and the true label $y \in \{0,1\}$:
\begin{equation*}
\mathrm{BCE}(x, y) \coloneq y \log(\sigma(x)) + (1 - y) \log(1 - \sigma(x)),
\end{equation*}
where $\sigma(x)$ denotes the sigmoid function. Importantly, note that when using the above setting for the loss function $\ell \colon \cH \times \cZ \to \Rb^{+}$, it is easy to verify that $\ell(h, \cdot)$, satisfies the Lipschitz property as required by \cref{thm:binaryclassificationdatadepend}, and \cref{thm:binaryclassificationtransductive}.

\section{The unrolling distances}
\label{app_sec:unrolling_distances_extension}

Here, we formally define the unrolling distances introduced in \cref{sec:unrolling_distances} and prove their main properties.

Recalling the definition of unrolling trees from \cref{app_sec:1WL}, given $(G,S) \in \cG' \otimes \mathfrak{R}$, a transformation function $\cT \colon \cG' \to \cG_{d}^{\Rb}$, selection function $\cV \colon (G,S) \mapsto A \subset V(\cT(G))$, and a number of layers $L \in \Nb$, we define the following multi-set of unrolling trees,
\begin{equation*}
	\cF^{(L)}(G,S) = \oms \UNR{\cT(G),u,L} \mid u \in \cV(G,S) \cms.
\end{equation*}
Our metric is intended to measure the alignment between the multisets of unrolling trees $\cF^{(L)}(G_1,S_1)$ and $\cF^{(L)}(G_2,S_2)$. We first define a padding process for handling multisets and unrolling of different sizes.

Given a rooted tree $(T,r)$ with node features in $\Rb^d\setminus \{\vec{0}_{\Rb^d}\}$, and parameters $L, q \in \Nb$ such that  
\begin{equation*}  
L \geq \max\{ \text{length}(p) \mid v \in V(T), p \in \mathcal{P}(r,v) \}, \quad q \geq \max\{ |N_{T,\text{out}}(v)| \mid v \in V(T) \},  
\end{equation*}  
where $L$ is at least the depth of the tree, and $q$ is at least the maximum out-degree of the corresponding directed tree, noting that any rooted tree has an implicit direction pointing away from the root, we define the $(q,L)$-padded rooted tree $(T',r)$ according to the following algorithm.

\begin{enumerate}
    \item Initialize $l=0$ and set $T_0 = T$.
    \item For each node $u$ at level $l$ of $T_0$, add children with node feature $\vec{0}_{\Rb^d}$ until $u$ has exactly $q$ children. The resulting tree is denoted as $T_1$.
    \item If $l = L-1$, the algorithm terminates, and we set $T' = T_1$. Otherwise, repeat step 2 with $l = l+1$ and set $T_0 = T_1$.
\end{enumerate}
We denote the resulting rooted tree $T'$ as $\mathsf{padd}_{q,L}(T)$, where the root remains unchanged.

Now, given $(G_1,S_1),(G_2,S_2) \in (\cG_{d}^{\Rb},\mathfrak{R})$, we apply the $\mathsf{padd}_{q,L}$ process to each unrolling tree in both $\cF^{(L)}(G_1,S_1)$ and $\cF^{(L)}(G_2,S_2)$, with $q$ set to the maximum node degree, considering the trees as directed trees, and $L$ set to the maximum depth among all trees in $\cF^{(L)}(G_1,S_1)$ and $\cF^{(L)}(G_2,S_2)$. Ignoring node features, the resulting trees are all isomorphic, i.e., all resulting trees are $L$-depth-rooted trees where each non-leaf node has exactly $q$ children. Finally, suppose the resulting multisets have different sizes. In that case, we augment the smaller one with additional trees of the same structure and node features $\vec{0}_{\Rb^{d}}$ until both multisets have the same cardinality. We denote the resulting multisets as $(M_1,M_2)$ and define the mapping,

\begin{equation*}
\rho \colon (\cF^{(L)}(G_1,S_1), \cF^{(L)}(G_2,S_2)) \mapsto (M_1, M_2).
\end{equation*}
See \cref{fig:padding_process} for an illustration of the padding process through the function $\rho$.

Finally, let $G_A, G_B$ be two graphs such that an edge-preserving bijection exists between their nodes. We define  
\begin{equation*}  
S(G_A, G_B) \coloneq \{ \varphi \colon V(G_A) \to V(G_B) \mid \varphi \text{ is an edge-preserving bijection} \}.
\end{equation*}

We now recall the definition of the unrolling distance function on
\begin{equation*}
\cG_{d}^{\Rb} \otimes \mathfrak{C} \coloneq \{ (G,S) \mid G \in \cG_{d}^{\Rb}, S \subseteq V(G) \}.
\end{equation*}

\begin{definition}[Restated, \cref{def:gen_unrolling_distances}]
Given a transformation function $\cT \colon \cG' \to \cG_{d}^{\Rb}$, and a selection function $\cV \colon (G,S) \mapsto A \subset V(\cT(G))$, for $(G,S) \in \cG' \otimes \mathfrak{R}$, we define the $(\cT,\cV)$-\new{Unrolling Distance} ($\mathrm{UD}_{\cT,\cV}^{(L)}$) of depth $L$ on $\cG_{d}^{\Rb} \otimes \mathfrak{C}$ as follows. Let $(G_1,S_1), (G_2,S_2) \in \cG' \otimes \mathfrak{C}$ and $\rho ( (\cF^{(L)}(G_1,S_1), \cF^{(L)}(G_2,S_2)) )=(M_1,M_2)$. Then,
\begin{equation*}
\mathrm{UD}^{(L)}_{\cT,\cV}( (G_1,S_1), (G_2,S_2) ) = \min_{ \varphi \in S(F_1,F_2)}  \sum_{x \in V(F_1)} \norm{ a_{F_1}(x) - a_{F_2}(\varphi(x))}_2,
\end{equation*}
where $F_i$ is the forest consisting of the disjoint union of all trees in $M_i$, for $i=1,2$, and we recall that
\begin{align*}  
S(G, H) \coloneq \{ \varphi \colon V(G) \to V(H) \mid \varphi \text{ is an edge-preserving bijection} \}.
\end{align*}
\end{definition}

Let $T_1,T_2$ be two rooted trees such that both have depth $L$ and all non-leaf nodes in $T_1$ and $T_2$ have exactly $q$ children. Let $\overrightarrow{T_1}$ and $\overrightarrow{T_2}$ be the directed versions of these trees. Comparing the degrees of the nodes (for $T_1,T_2$), it is easy to verify that $S(T_1, T_2) = S(\overrightarrow{T_1}, \overrightarrow{T_2})$, implying that in the above definition, it does not matter whether we consider the padded unrollings as directed or undirected trees.

We now prove that the Unrolling Distance is a well-defined pseudo-metric on $\cG_{d}^{\Rb} \otimes \mathfrak{C}$.

\begin{lemma}
\label{lemma:unrol_dist_pseudo-metricproperty} 
For every $L \in \Nb$, the Unrolling Distance $\mathrm{UD}_{\cT,\cV}^{(L)}$ in \cref{eq:unrolling_distance} is a well-defined pseudo-metric on 
$\cG_{d}^{\Rb} \otimes \mathfrak{C}.$
\end{lemma} 

\begin{proof}
The proof follows the reasoning of the proof in \citet[Lemma 1]{bento2019family}.
Let $\cG'$ denote the set of all graphs in $\cG_{d}^{\Rb}$ that do not contain nodes with zero feature vectors. That is,  
$\cG' = \cG_{d}^{\Rb} \cap \{G \in \cG_{d}^{\Rb} \mid a_G(u) \neq \vec{0}, \forall u \in V(G) \}$. We prove the pseudo-metric property on $\cG' \otimes \mathfrak{C}$ for simplicity of notation, but it can easily be extended to $\cG_{d}^{\Rb} \otimes \mathfrak{C}$.

Let $(G_i,S_i)\in\cG' \otimes \mathfrak{R}$. We denote by $F_i$ the forest consisting of all padded unrolling trees of depth $L$ of graph $G_i$ as defined in the definition of the Unrolling Distance and $a_{F_i}(u)$, being the node feature of node $u$ in $V(F_i)$. Note that we can assume that $F_1, F_2, F_3$ have the same structure, since if that is not the case, we can pad them by adding nodes (or even trees) with $\vec{0}_{\Rb^d}$ features. We further assume for simplicity of notation that $V(F_i)=[n]$, for $i \in [3]$. 

Now, we define the distance matrices $\vec{L}^{(\kappa, \lambda)}$, as $\vec{L}^{(\kappa, \lambda)}_{i,j} = \|a_{F_{\kappa}}(i) - a_{F_{\lambda}}(j)\|_2$, for $\kappa,\lambda,i,j \in [3]$. Let also $\cP_{i,j}$, being the set of all permutation $n\times n$ matrices $\vec{P}$ satisfying $\vec{A}(F_i)\vec{P} = \vec{P} \vec{A}(F_j)$, for $i,j \in [3]$.
Hence, the unrolling distance can be rewritten as,
\begin{equation}
\label{eq:unrolling_dystance_permutation_formula}
\mathrm{UD}_{\cT,\cV}^{(L)}((G_1,S_2), (G_2,S_2)) = \min_{\vec{P} \in \cP_{1,2}} \left( \tr(\vec{P}^{\mathrm{T}} \vec{L}^{(1,2)} ) \right).
\end{equation}
It is easy to observe from \cref{eq:unrolling_dystance_permutation_formula} that the identity and symmetry properties are trivially satisfied. For the triangle inequality, let $\vec{P}^{(i,j)} \in \cP_{i,j}$ be the permutation matrix that minimizes $\mathrm{UD}_{\cT,\cV}^{(L)}((G_i,S_i), (G_j,S_j))$. Note that $\vec{P}^{(1,2)}\vec{P}^{(2,3)} \in \cP_{1,3}$, that is because permutation matrices are closed under products and 
\begin{equation*}
\vec{A}(F_1)\vec{P}^{(1,2)}\vec{P}^{(2,3)} = \vec{P}^{(1,2)}\vec{A}(F_2) \vec{P}^{(2,3)} = \vec{P}^{(1,2)}\vec{P}^{(2,3)} \vec{A}(F_3).
\end{equation*}
Therefore, 
\begin{equation*}
\mathrm{UD}_{\cT,\cV}^{(L)}((G_1,S_1), (G_3,S_3)) \leq \tr((\vec{P}^{(1,2)}\vec{P}^{(2,3)})^{T}\vec{L}^{(1,3)} ).
\end{equation*}
It is, therefore, sufficient to show that,
\begin{equation*}
\tr\left((\vec{P}^{(1,2)}\vec{P}^{(2,3)})^{\mathrm{T}}\vec{L}^{(1,3)} \right) \leq  
\tr\left(\left(\vec{P}^{(1,2)}\right)^T\vec{L}^{(1,2)}\right) + \tr\left(\left(\vec{P}^{(2,3)}\right)^T\vec{L}^{(2,3)}\right).
\end{equation*}
To verify this, we use the following property of the trace operator. For $\vec{A},\vec{B},\vec{C} \in \Rb^{n\times n}$, we have that,
\begin{equation*}
\tr((\vec{A}\vec{B})^{T}\vec{C}) = \sum_{i,j,k=1}^{n}A_{i,k}B_{k,j}C_{i,j}.
\end{equation*}
Applying this we have that:
\begin{align*}
\tr( ( \vec{P}^{(1,2)}\vec{P}^{(2,3)})^T \vec{L}^{(1,3)}  ) &= \sum_{i,j,k=1}^{n} P^{(1,2)}_{i,k}P^{(2,3)}_{k,j} \|a_{F_{1}}(i) - a_{F_{3}}(j)\|_2 \\
&\leq \sum_{i,j,k=1}^{n} P^{(1,2)}_{i,k}P^{(2,3)}_{k,j} \left( \|a_{F_{1}}(i) - a_{F_{2}}(k)\|_2 + \|a_{F_{2}}(k) - a_{F_{3}}(j)\|_2 \right) \\
& = \sum_{i,k=1}^{n} P^{(1,2)}_{i,k} \|a_{F_{1}}(i) - a_{F_{2}}(k)\|_2 \underbrace{\sum_{j=1}^{n} P^{(2,3)}_{k,j}}_{=1} \\
&+ \sum_{k,j=1}^{n} P^{(2,3)}_{k,j} \|a_{F_{2}}(k) - a_{F_{3}}(j)\|_2 \underbrace{\sum_{i=1}^{n} P^{(1,2)}_{i,k}}_{=1} \\
& = \tr\left(\left(\vec{P}^{(1,2)}\right)^T\vec{L}^{(1,2)}\right) + \tr\left(\left(\vec{P}^{(2,3)}\right)^T\vec{L}^{(2,3)}\right)
\end{align*}
\end{proof}

Finally, note that the above definition can be generalized by introducing different weights on each level of the unrolling trees. Given a forest of rooted trees $F$, define the level function $l_F \colon V(F) \to \Nb$, where $l_F(u)$ is the level of the tree node $u$ belongs to. Given a weight function $\omega \colon \Nb \to \Rb^{+}$, we define the following \new{Weighted Unrolling Distance},
\begin{equation}
\label{def:weighted_node-forest-distance}
\mathrm{UD}_{\cT,\cV, \omega}^{(L)}((G_1,S_2), (G_2,S_2)) = \min_{ \varphi \in S(F_1,F_2)}  \sum_{x \in V(F_1)} \omega(l_{F_1}(x)) \norm{ a_{F_1}(x) - a_{F_2}(\varphi(x))}_2.
\end{equation}
To verify that \cref{def:weighted_node-forest-distance} is a well-defined pseudo-metric, we would have to modify the definition of unrolling trees by weighting the node features based on their level on the tree and then follow the proof of \cref{lemma:unrol_dist_pseudo-metricproperty}. Regarding the numerical computation of the above distance, we follow the optimal transport approach described in~\citet{chuang2022tree}, which formulates the problem as an instance of optimal transport and solves it using dynamic programming. Alternatively, faster approximation methods, such as those proposed by~\citet{Altschuler2017}, can be employed.

The following result shows that the sub-sum property, see~\cref{def:sub-sum_pooling}, provides a sufficient condition for a generalized $(\cT,\cV,\Psi)$-MPNN$(L)$ to satisfy the Lipschitz property with respect to the $\mathrm{UD}^{(L)}_{\cT,\cV,\omega}$ pseudo-metric, for a specific weighting function $\omega$.

\begin{proposition}
\label{prop:Lipschitz_generalized_MPNNs}
Let $\cG_{d}^{\Rb} \otimes \mathfrak{R}$ be a representation task, and let $(\cT, \cV, \Psi)$ be a generalized MPNN$(L)$ with an MPNN architecture given in \cref{sec:MPNNs}, where $\Psi$ satisfies the sub-sum property with parameter $\widetilde{C}$. Then, for 
\begin{equation*}
\omega \colon l \mapsto \binom{L}{l}, \quad l \in \{0,\ldots,L\},
\end{equation*}
where $\binom{L}{l}$ denotes the binomial coefficient. The generalized MPNN is Lipschitz concerning the weighted unrolling distance $\mathrm{UD}_{\cT,\cV, \omega}^{(L)}$. That is, there exists a constant $C > 0$ such that for all $(G_1,S_1), (G_2,S_2) \in \cG_{d}^{\Rb} \otimes \mathfrak{R}$, we have
\begin{equation*}
    \|h_{\cT,\cV,\Psi}(G_1,S_1) - h_{\cT,\cV,\Psi}(G_2,S_2) \|_2 \leq C \cdot \mathrm{UD}_{\cT,\cV,\omega}^{(L)}((G_1,S_1),(G_2,S_2)).
\end{equation*}
Moreover, the constant $C$ depends on $L$, $\widetilde{C}$, the Lipschitz constants $L_{\varphi_t}$ and the bound $B$ on the $2$-norm of weighted matrices described in \cref{sec:MPNNs}.
\end{proposition}

\begin{proof}
Let $(G_1,S_1), (G_2,S_2) \in \cG_{d}^{\Rb}$ (without $\vec{0}$ node features), and let $F_1, F_2$ be their corresponding padded forests as described in \cref{def:gen_unrolling_distances}. First, note that each edge-preserving bijection $\varphi$ between $F_1$ and $F_2$ (i.e., $\varphi \in S(F_1,F_2)$) induces a sequence of bijections $\widebar{\sigma}_{l} \colon V_1^{(l)} \to V_2^{(l)}$, where $V_1^{(l)}$ and $V_2^{(l)}$ denote the sets of nodes in $F_1$ and $F_2$ at depth $l$, respectively, for $l \in \{0,\ldots, L\}$. 

Based on $\widebar{\sigma}_{l}$, we can define a sequence of extended bijections $\sigma_{l}$ between $V_1^{(l)} \cap \{ u \in V_1^{(l)} \mid a_{F_1}(u) \neq \vec{0} \}$ and $V_2^{(l)} \cap \{ u \in V_2^{(l)} \mid a_{F_2}(u) \neq \vec{0} \}$ as follows:
\begin{equation*}
    \sigma_{l}(u) = \begin{cases} 
    \widebar{\sigma}_{l}(u), & \text{if } a_{F_1}(u) \neq \vec{0}, \\
    *, & \text{otherwise}. 
    \end{cases}
\end{equation*}
A similar definition applies when $|V_1^{(l)}| < |V_2^{(l)}|$. Therefore, each $\varphi \in S(F_1,F_2)$ uniquely induces a sequence of extended bijections. It is straightforward to verify that this sequence is unique and that this process can be reversed, meaning that each sequence of extended bijections induces a unique $\varphi \in S(F_1,F_2)$. For simplicity, we assume $B=1$ (the proof follows similarly for general $B$). Finally, by the construction of the unrolling trees, each node in $V_i^{(l)} \cap \{ u \in V_1^{(l)} \mid a_{F_1}(u) \neq \vec{0} \}$ can be mapped to the corresponding node $u \in V(\cT(G_i))$, for $i=1,2$, $l=0,\ldots,L$. Therefore, with a slight abuse of notation, $\sigma_{l}$ also defines an extended bijection between nodes in $\cT(G_1)$ and $\cT(G_2)$ that can be found in the $l$-th level of some tree.

Let $\varphi \in S(F_1,F_2)$, and let $\sigma_0,\ldots,\sigma_L$ be the corresponding induced extended bijections. We define $\delta(L,l)$ as follows \footnote{In this notation, without loss of generality and for simplicity, we assume that 
$|N_{F_1}(x) | \geq | N_{F_2}(y) |$. If this assumption does not hold, the notation can be trivially adjusted to reflect the reverse case.}:
\begin{align*}
    & \delta(L,0) = \sum_{x_{0} \in \cV(G_1,S_1)} \| h_{\cT(G_1)}^{(L)}(x_{0}) - h_{\cT(G_2)}^{(L)}(\sigma_0(x_{0})) \|_2, \\
    & \delta(L,l) = \sum_{x_{0} \in \cV(G_1,S_1)} \sum_{x_{1} \sim x_{0}} \dots \sum_{x_l \sim x_{l-1}} \| h_{\cT(G_1)}^{(L)}(x_l) - h_{\cT(G_2)}^{(L)}(\sigma_l(x_l)) \|_2, \quad \text{for } l \in [L],
\end{align*}
where $h_{\cT(G_i)}^{(l)}(*)=\vec{0}$, for all $i=1,2$, $l \in \{0,\ldots,L\}$ and $x_l \sim x_{l-1} \Leftrightarrow x_l \in N_{\cT(G_1)}(x_{l-1})$, for all $l \in \{0,\ldots ,L\}$.

It is easy to observe that,
\begin{align*}
    \delta(L,l) &\leq L_{\varphi_{L}} \sum_{x_{0} \in \cV(G_1,S_1)} \sum_{x_{1} \sim x_{0}} \dots \sum_{x_l \sim x_{l-1}} \| \vec{W}^{(1)}_L \| \| h_{\cT(G_1)}^{(L-1)}(x_l) - h_{\cT(G_2)}^{(L-1)}(\sigma_l(x_l)) \|_2 \\
    & + L_{\varphi_{L}} \sum_{x_{0} \in \cV(G_1,S_1)} \sum_{x_{1} \sim x_{0}} \dots \sum_{x_{l+1} \sim x_l} \| \vec{W}^{(2)}_L \| \| h_{\cT(G_1)}^{(L-1)}(x_l) - h_{\cT(G_2)}^{(L-1)}(\sigma_{l+1}(x_{l+1})) \|_2 \\
    &\leq L_{\varphi_{L}} \delta(L-1,l)  +  L_{\varphi_{L}} \delta(L-1,l+1) 
\end{align*}
Thus, recursively applying this bound, we get:
\begin{align*}
    \| \Psi (F (G_1,S_1)) - \Psi(F(G_2,S_2)) \|_2 &\leq \widetilde{C}  \cdot \sum_{x \in \cV(G_1,S_1)} \| h_{\cT(G_1)}^{(L)}(x)  - h_{\cT(G_2)}^{(L)}(\sigma(x)) \|_2 \\
    &= \widetilde{C} \delta(L,0) \\
    &\leq \widetilde{C} L_{\varphi_{L}} \left( \delta(L-1,0) + \delta(L-1, 1) \right) \\
    &\leq \dots \\
    &\leq \widetilde{C} \left( \prod_{t=1}^{L} L_{\varphi_{t}} \right) \sum_{l=0}^{L-1} \binom{L}{l} \delta(0,l) \\
    &= \widetilde{C} \left( \prod_{t=1}^{L} L_{\varphi_{t}} \right)  \cdot \mathrm{UD}_{\cT,\cV,\omega}^{(L)}((G_1,S_1),(G_2,S_2)).
\end{align*}
\end{proof}

\section{Missing proofs from Section~\ref{sec:robustness_under_dependency}}
\label{app_sec:robustness_under_dependency}

This appendix contains the proofs of the main results of this work: \cref{thm:Xu_Mannor_noniid} and \cref{thm:Xu_Mannor_transductive}. We organize the presentation of the proofs as follows. For \cref{thm:Xu_Mannor_noniid}, we first establish a concentration inequality for dependent data, similar in spirit to the Bretagnolle--Huber--Carol inequality, stated as \cref{lemma:Bretagnolle_Huber_extension}, which is then used to prove the theorem. Similarly, for \cref{thm:Xu_Mannor_transductive}, we prove a concentration inequality adapted to the transductive setting (\cref{lemma:Bretagnolle_Huber_permutation_extension}), which forms the basis of the proof of the main result.

\subsection{Dependency graphs and concentration inequalities}
\label{app_sec:dependency_graphs}
We begin with the formal definition of dependency graphs, which model dependencies between random variables.

\begin{definition}[Dependency graphs]
Let $(\Omega, \cF, P)$ be a probability space and $(E, \mathcal{E})$ a measurable space. Given $n \in \Nb$ and random variables $X_i \colon \Omega \to E$, for $i \in [n]$, let $G$ be an undirected graph with $V(G) = [n]$. We say that $G$ is a \emph{dependency graph} of $\vec{X} = (X_1, \ldots, X_n)$, or that $\vec{X}$ is $G$-dependent, if for all disjoint subsets $I, J \subset [n]$, whenever $I$ and $J$ are not adjacent in $G$ (i.e., for all $u \in I, v \in J$ such that $\{u,v\} \notin E(G)$), the collections $\{X_i\}_{i \in I}$ and $\{X_j\}_{j \in J}$ are independent.
\end{definition}

This structure enables us to extend concentration inequalities designed initially for independent variables to settings with controlled dependencies by penalizing the complexity of the corresponding dependency graph, as demonstrated by the following theorem.

\begin{theorem}[{\citet[Theorem 2.1]{DBLP:journals/rsa/Janson04}}]
\label{thm:Janson's_concentration}
Let $(\Omega, \cF, P)$ be a probability space, and let $X_i \colon (\Omega, \cF) \to \Rb$, for $i \in [n]$, be random variables with dependency graph $G$. Suppose that each $X_i$ takes values in a real interval of length at most $c_i \geq 0$ for all $i \in [n]$. Then, for every $t > 0$,

\begin{equation*}
P \left( \sum_{i=1}^{n}X_i - \Eb\left[ \sum_{i=1}^{n}X_i \right] \geq t \right) 
\leq \exp\left( \frac{-2t^2}{\chi(G) \|\vec{c}\|_2^2} \right),
\end{equation*}

where $\vec{c} = (c_1, \ldots, c_n)$, and $\chi(G)$ denotes the chromatic number of $G$. More specifically, the chromatic number $\chi(G)$ in the bound can be replaced by the fractional chromatic number of $G$, which provides a tighter bound, as it lower bounds $\chi(G)$; see \citep{DBLP:journals/rsa/Janson04}.
\end{theorem}

We now use \cref{thm:Janson's_concentration} to prove a variant of the Bretagnolle--Huber--Carol inequality for non-independent samples through the following lemma.

\begin{lemma}[Restated, \cref{lemma:Bretagnolle_Huber_extension}]
Let $(\Omega, \cF, P)$ be a probability space, let $n \in \Nb$, and let $G \in \cG_n$. Suppose that $X_i \colon \Omega \to A \subset \Rb$ are $G$-dependent random variables with identical distribution $\mu$, for all $i \in [n]$. Furthermore, assume that there exists $K \in \Nb$ such that the set $A$ can be partitioned into $K$ disjoint subsets, denoted by $\{C_j\}_{j=1}^{K}$. Define the random variables 
\begin{equation*}
Z_j = \sum_{i=1}^{n} \mathbf{1}_{\{X_i \in C_j\}}, \quad \text{for } j \in [K],
\end{equation*}
where $\mathbf{1}$ denotes the indicator function. Then, for all $t > 0$, the following inequality holds,

\begin{equation*}
P \left( \sum_{j=1}^{K} \left| Z_j - n \mu(C_j) \right| \geq 2t \right) 
\leq 2^{K+1} \exp\left( \frac{-2t^2}{\chi(G)n} \right),
\end{equation*}
where, $\chi(G)$ denotes the chromatic number of the dependency graph $G$.
\end{lemma}

\begin{proof}
We first observe that
\begin{align*}
\sum_{j=1}^{K} |Z_j - n \mu(C_j)| &= \sum_{j \colon Z_j \geq n \mu(C_j)}(Z_j - n \mu(C_j)) + \sum_{j \colon Z_j < n \mu(C_j)}(n \mu(C_j) - Z_j) \\
&= \underbrace{\max_{S \subset [K]} \sum_{j \in S}(Z_j - n \mu(C_j))}_{A} + \underbrace{\max_{S \subset [K]} \sum_{j \in S}(n \mu(C_j) - Z_j)}_{B}  \\
&\leq 2\max\{A,B\}.
\end{align*}
Therefore,
\begin{align*}
P \left( \sum_{j=1}^{K} \left| Z_j - n \mu(C_j) \right| \geq 2t \right)  & \leq P \left( \max\{A,B\} \geq t \right) \\
&= P \left( (A \geq t)\cup(B \geq t) \right) \\
&\leq P \left( A \geq t \right) + P \left( B \geq t \right).
\end{align*}
Now, note that
\begin{align*}
P(A \geq t) &= P \left(\bigcup_{S \subset [K]} \left\{ \sum_{j \in S} (Z_j - n \mu(C_j)) \geq t \right\} \right) \\
& \leq \sum_{S \subset [K]} P\left( \sum_{j \in S}(Z_j - n \mu(C_j))\geq t \right) \\
& = \sum_{S \subset [K]} P\left( \sum_{j \in S}Z_j -  n\sum_{j \in S}\mu(C_j) \geq t \right).
\end{align*}
Now, observe that $\sum_{j \in S} Z_j = \sum_{i=1}^{n} Y_{i}^{(S)}$, where $Y_{i}^{(S)} = \sum_{j \in S} \mathbf{1}_{X_i \in C_j}$, and $Y_{i}^{(S)} \in \{0,1\}$ since $\{C_j\}_{j=1}^{K}$ are pairwise disjoint as a partition. Note also that $\Eb[Y_{i}^{(S)}] = \sum_{j \in S} \mu(C_j)$. Finally, it is easy to see that since $X_i$ are $G$-dependent, the random variables $Y_{i}^{(S)}$ are also $G$-dependent. By \cref{thm:Janson's_concentration}, we obtain
\begin{align*}
P\left( A \geq t \right) & \leq \sum_{S\subset [K]} P\left( \sum_{i=1}^{n} Y_{i}^{(S)} - \Eb\left[\sum_{i=1}^{n} Y_i^{(S)}\right] \geq t \right) \\
& \leq \sum_{S \subset [K]} \exp\left( \frac{-2t^2}{\chi(G)n} \right) \\
&= 2^K \exp\left( \frac{-2t^2}{\chi(G)n} \right).
\end{align*}
Similarly, noting that $\sum_{j \in S} -Z_j = \sum_{i=1}^{n} \widetilde{Y}_{i}^{(S)}$, where $\widetilde{Y}_{i}^{(S)} = -\sum_{j \in S} \mathbf{1}_{X_i \in C_j}$, $\widetilde{Y}_{i}^{(S)} \in \{0,-1\}$, and $\Eb[\widetilde{Y}_{i}^{(S)}] = - \sum_{j \in S} \mu(C_j)$, we obtain:
\begin{align*}
P\left( B \geq t \right) & \leq \sum_{S\subset [K]} P\left( \sum_{i=1}^{n} \widetilde{Y}_{i}^{(S)} - \Eb\left[\sum_{i=1}^{n} \widetilde{Y}_i^{(S)}\right] \geq t \right) \\
& \leq \sum_{S \subset [K]} \exp\left( \frac{-2t^2}{\chi(G)n} \right) \\
&= 2^K \exp\left( \frac{-2t^2}{\chi(G)n} \right).
\end{align*}
Hence, 
\begin{equation*}
P \left( \sum_{j=1}^{K} \left| Z_j - n \mu(C_j) \right| \geq 2t \right) 
\leq 2^{K+1} \exp\left( \frac{-2t^2}{\chi(G)n} \right).
\end{equation*}
\end{proof}

We are now ready to prove the first robustness generalization bound for the inductive setting.

\begin{theorem}[Restated, \cref{thm:Xu_Mannor_noniid}]
Let $(\Omega, \cF, P)$ be a probability space, and let $\cA$ be a $(K,\varepsilon)$-uniformly robust learning algorithm on $\cZ$. Then, for every $\delta > 0$, with probability at least $1 - \delta$ (under $P$), and for all samples $\cS$ with dependency graph $G[\cS]$, the following inequality holds:

\begin{equation*}
\left| \ell_{\mathrm{exp}}(\cA_{\cS}) - \ell_{\mathrm{emp}}(\cA_{\cS}) \right| 
\leq \varepsilon + M \sqrt{\frac{\chi(G[\cS])\left( 2(K+1)\log 2 + 2\log\left(\frac{1}{\delta}\right)\right)}{|\cS|}},
\end{equation*}
where $\ell_{\mathrm{exp}}(h_{\cS})$ denotes the expected loss, $\ell_{\mathrm{emp}}(h_{\cS})$ denotes the empirical loss, $\chi(G[\cS])$ is the chromatic number of the dependency graph induced 
\end{theorem}

\begin{proof}
Let $\{C_j\}_{j=1}^{K}$ be the partition of $\cZ$ by the robustness property, and $N_j = \{ s\in \cS \mid s \in C_j \}$, for $j \in [K]$, then, 
\begin{align*}
& |\ell_{\mathrm{exp}}(\cA_{\cS}) - \ell_{\mathrm{emp}}(\cA_{\cS})|                                                                                                                                                                                                                                                                             \\
& = \mleft| \sum_{j=1}^K \Eb(\ell(\cA_{\cS}, \vec{z}) \mid \vec{z} \in C_j) \mu(C_j) - \frac{1}{|\cS|} \sum_{s \in \cS} \ell(\cA_{\cS}, s) \mright|                                                                                                                                                                                                            \\
& \leq \mleft| \sum_{j=1}^K \Eb(\ell(\cA_{\cS}, \vec{z}) \mid \vec{z} \in C_j) \frac{|N_j|}{|\cS|} - \frac{1}{|\cS|} \sum_{s \in \cS} \ell(\cA_{\cS}, s) \mright|                                                                                                                                                                                              \\
& + \mleft| \sum_{j=1}^K \Eb(\ell(\cA_{\cS}, \vec{z}) \mid  \vec{z} \in C_j) \mu(C_j) -  \sum_{j =1}^{K} \Eb \mleft( \ell(\cA_{\cS},  \vec{z}) \mid  \vec{z}\in C_j \mright)\frac{|N_j|}{|\cS|} \mright|                                                                                                                                                                         \\
& \leq \mleft| \frac{1}{|\cS|} \sum_{j=1}^{K}\sum_{s \in N_j} \max_{ \vec{z}_2 \in C_j} \mleft| \ell \mleft(\cA_{\cS},s\mright) - \ell \mleft(\cA_{\cS}, \vec{z}_2\mright) \mright| \mright| + \mleft| \max_{ \vec{z} \in \cZ} \mleft| \ell\mleft( \cA_{\cS},  \vec{z} \mright) \mright| \sum_{j=1}^{K} \mleft| \frac{|N_j|}{|\cS|}-\mu\mleft( C_j \mright) \mright| \mright| \\
& \leq \varepsilon + M\sum_{j=1}^{K} \mleft| \frac{|N_j|}{|\cS|}-\mu\mleft( C_j \mright) \mright| \\
& \leq \varepsilon + M\sqrt{\frac{\chi(G[\cS])\left( 2(K+1)\log 2 + 2\log\left(\frac{1}{\delta}\right)\right)}{|\cS|}},
\end{align*}
where in the last inequality we applied \cref{lemma:Bretagnolle_Huber_extension}.
\end{proof}

\subsection{Transductive and concentration inequalities}
\label{app_sec:transductive_dependency}

We now continue with the transductive setting. First, we must formally define the transductive stability property described in \cref{sec:transductive-robustness}. 

Before defining the notion of stability, we introduce the following notation. Given a permutation $\pi \in S_n$ and indices $i, j \in [n]$ with $i \neq j$, we denote by $\pi^{i,j}$ the permutation $\widebar{\pi} \in S_n$ obtained by swapping the elements at positions $i$ and $j$ in $\pi$. More precisely, $\widebar{\pi}(k) = \pi(k)$, for $k \neq i, j$, while $\widebar{\pi}(i) = \pi(j)$ and $\widebar{\pi}(j) = \pi(i)$. We now formally define the notion of stability following an adaptation of \citet{stabilityoriginal}.

\begin{definition}[Uniform transductive stability]
\label{def:stability}
A transductive learning algorithm $\cA$ has uniform transductive stability $\beta>0$ if, for all $\pi \in S_n$, $i \in [m]$, and $j \in [n] \setminus [m]$, the following holds,
\begin{equation*}
\max_{1\leq k \leq n} \left| \cA_{\pi}(x_k) - \cA_{\pi^{i,j}}(x_k) \right| \leq \beta.
\end{equation*}
\end{definition}

The following lemma shows that if a transductive learning algorithm satisfies uniform stability and the loss function is Lipschitz continuous, then the absolute difference between the expected empirical risk $R_m$ and the expected transductive risk $R_u$ is bounded by a constant that depends on the stability parameter and the Lipschitz constant.

\begin{lemma}
\label{lemma:stability_lipschitz_bound}
Let $\cA$ be a transductive learning algorithm satisfying uniform transductive stability $\beta$, for $\beta>0$, and $\cY=\{0,1\}$ or $\Rb$. Then, for a loss function $\ell \colon \cH \times V(G) \times \cY \to \Rb^{+}$ that is bounded by $M>0$ and satisfies the following Lipschitz property. For all $\pi \in S_n, x,x'\in V(G), y\in \cY$,
\begin{equation*}
\left| \ell(\cA_{\pi},x,y) - \ell(\cA_{\pi},x',y) \right| \leq L_{\ell} \left| \cA_{\pi}(x) - \cA_{\pi}(x') \right|,
\end{equation*}
we have that,
\begin{equation*}
| \Eb_{\pi}[R_m] - \Eb_{\pi}[R_u] | \leq L_{\ell}\cdot \beta.
\end{equation*}
\end{lemma}

\begin{proof}
By \citet[Lemma 7]{DBLP:phd/il/Pechyony08}, we have
\begin{equation*}
\left| \Eb_{\pi}[R_m] - \Eb_{\pi}[R_u] \right| = \left| \Eb_{\pi,\ i \sim I_{1}^{m},\ j \sim I_{m+1}^{u}} \left[ \ell\left( \cA_{\pi^{i,j}}, x_i, y_i \right) - \ell\left( \cA_{\pi}, x_i, y_i \right) \right] \right|,
\end{equation*}
where $i \sim I_{1}^{m}$ denotes that $i$ is sampled uniformly from $\{1,\ldots,m\}$, and $j \sim I_{m+1}^{u}$ denotes that $j$ is sampled uniformly from $\{m+1,\ldots,u\}$. 

Applying the Lipschitz continuity of the loss function $\ell$ (second inequality), and the uniform stability (third inequality), we obtain:
\begin{align*}
&\left| \Eb_{\pi,\ i \sim I_{1}^{m},\ j \sim I_{m+1}^{u}} \left[ \ell\left( \cA_{\pi^{i,j}}, x_i, y_i \right) - \ell\left( \cA_{\pi}, x_i, y_i \right) \right] \right| \\
&\leq \Eb_{\pi,\ i \sim I_{1}^{m},\ j \sim I_{m+1}^{u}} \left[ \left| \ell\left( \cA_{\pi^{i,j}}, x_i, y_i \right) - \ell\left( \cA_{\pi}, x_i, y_i \right) \right| \right] \\
&\leq L_{\ell} \Eb_{\pi,\ i \sim I_{1}^{m},\ j \sim I_{m+1}^{u}} \left[ \left| \cA_{\pi^{i,j}}(x_i) - \cA_{\pi}(x_i) \right| \right] \\
&\leq L_{\ell} \Eb_{\pi,\ i \sim I_{1}^{m},\ j \sim I_{m+1}^{u}} \left[ \beta \right] \\
&= L_{\ell} \beta.
\end{align*}
\end{proof}

Below, we prove the concentration inequality for the transductive setting that is required to establish our main theorem.

\begin{lemma}[Restated, \cref{lemma:Bretagnolle_Huber_permutation_extension}]
Let $(\Omega, \cF, P)$ be a probability space, and $n,n' \in \Nb$ with $n<n'$. If $Z=\{z_1,z_2,\ldots, z_{n'} \}$ be an arbitrary finite set, $\{C_j\}_{j \in [K]}$ being a partition of $Z$, $\pi \sim \text{Unif}(S_{n'})$, and
\begin{equation*}
X_j = \sum_{i=1}^{n} \textbf{1}_{\{z_{\pi(i)} \in C_j\}},
\end{equation*}
then the following inequality holds. For all $S \subset [K]$, and for all $t > 0$,
\begin{equation*}
P\left( \left| \sum_{j \in S} X_j - \Eb_{\pi}\left(\sum_{j \in S} X_j\right) \right| \geq t \right) \leq 2\text{exp}\left( \frac{-t^2}{2|S|^2 \cdot n} \right).
\end{equation*}
\end{lemma}

\begin{proof}
For $S \subset [K]$, let $f_{S} \colon S_n \to \{0,1 \}$ with 
\begin{equation*}
f_S(\pi) = \sum_{j \in S} \sum_{i=1}^{n} \textbf{1}_{\{z_{\pi(i)} \in C_j \}}.
\end{equation*}
We define the filtration $\cF_{i} = \cup_{j=0}^{i} \sigma(\pi(j))$, for $i \in [n]$ where, by convention, $\sigma(\pi(0)) \coloneq \{ \Omega, \emptyset \}$ is the trivial $\sigma$-algebra. 

Next, we define $(W_i)_{i=0}^{n}$ as the Doob martingale of the random variable $f(\pi)$ with respect to the filtration $\{\cF_i\}_{i=1}^{n}$. That is, 
\begin{equation*}
W_0 = \Eb_{\pi}(f_S(\pi)), \quad W_i = \Eb(f_S(\pi) | \cF_i), \quad \text{for } i\in[n].
\end{equation*}
We aim to apply Azuma's inequality (\cref{thm:azuma_inequality}). To do so, we need to bound the differences $|W_{k}-W_{k-1}|$, for $k \in [n]$. We have,
\begin{align*}
|W_k - W_{k-1} | &= \left| \Eb(f(\pi) \mid \cF_k) - \Eb(f(\pi) \mid \cF_{k-1}) \right| \\
&= \left| \Eb \left( \sum_{j \in S} \sum_{i=1}^{n} \textbf{1}_{\{z_{\pi(i)} \in C_j \}} \Big| \cF_k \right) - \Eb \left( \sum_{j \in S} \sum_{i=1}^{n} \textbf{1}_{\{z_{\pi(i)} \in C_j \}} \Big| \cF_{k-1} \right) \right| \\
&= \left| \sum_{i=k}^{n} \sum_{j \in S} \left( \Eb(\textbf{1}_{\{z_{\pi(i)} \in C_j \}} \mid \cF_k) - \Eb(\textbf{1}_{\{z_{\pi(i)} \in C_j \}} \mid \cF_{k-1}) \right) \right| \\
&\leq \sum_{j \in S} \sum_{i=k}^{n} \underbrace{\left| \Eb(\textbf{1}_{\{z_{\pi(i)} \in C_j \}} \mid \cF_k) - \Eb(\textbf{1}_{\{z_{\pi(i)} \in C_j \}} | \cF_{k-1}) \right|}_{A_{i,j,k}},
\end{align*} 
where we have used the tower property, i.e., $\Eb(X \mid \cF) = \Eb(X)$ if $X$ is $\cF$-measurable.

We now show that $A_{i,j,k} \leq \frac{1}{n-k}$, for all $k \leq n$, $i \in [n] \setminus [k]$, and $j \in S$. Let 
\begin{equation*}
m_{j}^{(k)} = \left| \{ i \in [k] \mid z_{\pi(i)} \in C_j \} \right|.
\end{equation*}
Since $\pi$ is a uniformly random permutation, we have:
\begin{equation*}
\Eb_{\pi}(\textbf{1}_{\{z_{\pi(i)} \in C_j\}} | \cF_{k-1} ) = \frac{|C_j| - m_j^{(k-1)}}{n - (k-1)},
\end{equation*}
and
\begin{equation*}
\Eb_{\pi}(\textbf{1}_{\{z_{\pi(i)} \in C_j\}} | \cF_{k} ) = \frac{|C_j| - m_j^{(k)}}{n-k} = \frac{|C_j| - m_j^{(k-1)} - \textbf{1}_{\{z_{\pi(k)} \in C_j\}}}{n-k}.
\end{equation*}
If $\textbf{1}_{\{z_{\pi(k)} \in C_j\}} = 1$, we set $a = |C_j| + m_{j}^{(k-1)} - 1$ and $b = n-k$, obtaining
\begin{align*}
A_{i,j,k} &= \left| \frac{a}{b} - \frac{a+1}{b+1} \right| = \frac{1}{b+1} \left| \frac{b-a}{b} \right| \leq \frac{1}{b}.
\end{align*}
Similarly, if $\textbf{1}_{\{z_{\pi(k)} \in C_j\}} = 0$, setting $a = |C_j| + m_{j}^{(k-1)}$ and $b = n-k$ gives
\begin{align*}
A_{i,j,k} &= \left| \frac{a}{b} - \frac{a}{b+1} \right|= \frac{1}{b+1} \left| \frac{b-a}{b} \right| \leq \frac{1}{(b+1)}\leq \frac{1}{b}.
\end{align*}
Therefore, we conclude that $A_{i,j,k} \leq \frac{1}{n-k}$, which implies 
\begin{equation*}
|W_k - W_{k-1} | \leq |S|.
\end{equation*}

Applying Azuma's inequality, we obtain
\begin{equation*}
P \left( \left|  f_S(\pi) - \Eb_{\pi} \left( f_S(\pi) \right) \right| \geq t \right) = P\left( \left| \sum_{j \in S} X_j - \Eb_{\pi}\left(\sum_{j \in S} X_j\right) \right| \geq t \right) \leq 2\exp\left( \frac{-t^2}{2n|S|^2} \right).
\end{equation*}
\end{proof}

Finally, we restate and prove the main robustness generalization theorem for the transductive setting.

\begin{theorem}[Restated, \cref{thm:Xu_Mannor_transductive}]
Let $(\Omega, \cF, P)$ be a probability space, $\ell$ be a loss function bounded by $M$ satisfying the conditions of \cref{lemma:stability_lipschitz_bound}. If $\cA$ is a transductive learning algorithm on $\cZ = \{z_i\}_{i=1}^{m+u}$ with hypothesis class $\cH$ that is $(K,\varepsilon)$-uniformly-robust and satisfies uniform transductive stability $\beta>0$. Then, for every $\delta > 0$, with probability at least $1 - \delta$, the following inequality holds,
\begin{align*}
&\left|\frac{1}{m} \sum_{i=1}^{m} \ell(\cA_{\pi},x_{\pi(i)}, y_{\pi(i)}) - \frac{1}{u} \sum_{i=m+1}^{m+u} \ell(\cA_{\pi},x_{\pi(i)}, y_{\pi(i)})\right| \leq \\
& 2\varepsilon + \left( \frac{1}{\sqrt{m}}+\frac{1}{\sqrt{u}} \right) \cdot M \cdot K \cdot \sqrt{2(K+1)\log2 + 2\log\left(\frac{1}{\delta}\right)} + L_{\ell}\beta,
\end{align*}
where $M$ is an upper bound for the loss function $\ell$.
\end{theorem}

\begin{proof}
We begin by considering the absolute difference:
\begin{align*}
& |R_m - R_u | \\
& \leq |R_m - \Eb_{\pi}[R_m]| + |R_u -\Eb_{\pi}[R_u]| + |\Eb_{\pi}[R_m]-\Eb_{\pi}[R_u]|  \\
& =   |R_m - \Eb_{\pi}[\ell(\cA_{\pi}, x_{\pi(1)}, y_{\pi(1)})]| + |R_u -\Eb_{\pi}[\ell(\cA_{\pi}, x_{\pi(m+1)}, y_{\pi(m+1)})]| + |\Eb_{\pi}[R_m]-\Eb_{\pi}[R_u]| \\
& \leq  |R_m - \Eb_{\pi}[\ell(\cA_{\pi}, x_{\pi(1)}, y_{\pi(1)})]| + |R_u -\Eb_{\pi}[\ell(\cA_{\pi}, x_{\pi(m+1)}, y_{\pi(m+1)})]|  + L_{\ell}\beta.
\end{align*}    

The first inequality follows from the triangle inequality. The equality follows from the symmetry of the learning algorithm (Remark 3). The last inequality is obtained by applying \cref{lemma:stability_lipschitz_bound}. 

Next, we proceed similarly to the proof of \cref{thm:Xu_Mannor_noniid}, but we employ \cref{lemma:Bretagnolle_Huber_permutation_extension} instead of \cref{lemma:Bretagnolle_Huber_extension} to bound the term $|R_m - \Eb_{\pi}[\ell(\cA_{\pi}, x_{\pi(1)}, y_{\pi(1)})]|$. Using the same reasoning, we can bound $|R_u -\Eb_{\pi}[\ell(\cA_{\pi}, x_{\pi(m+1)}, y_{\pi(m+1)})]|$. 

Defining $N_{j}^{m} = \{i \in [m] \mid z_{\pi(i)} \in C_j \}$, we obtain,
\begin{align*}
& |R_m -\Eb_{\pi}[\ell(\cA_{\pi}, x_{\pi(1)}, y_{\pi(1)})]|  = \left| \sum_{j=1}^K \Eb_{\pi}(\ell(\cA_{\pi}, z_{\pi(1)}) \mid z_{\pi(1)} \in C_j) P(z_{\pi(1)} \in C_j ) - \frac{1}{m} \sum_{i=1}^{m} \ell(\cA_{\pi}, z_{\pi(i)}) \right| \\
& \leq \left| \sum_{j=1}^K \Eb_{\pi}(\ell(\cA_{\pi}, z_{\pi(1)}) \mid z_{\pi(1)} \in C_j) \frac{|N_j^m|}{m} - \frac{1}{m} \sum_{i=1}^{m} \ell(\cA_{\pi}, z_{\pi(i)}) \right| \\
& \quad + \left| \sum_{j=1}^K \Eb_{\pi}(\ell(\cA_{\pi}, z_{\pi(1)}) \mid z_{\pi(1)} \in C_j) P(z_{\pi(1)} \in C_j ) -  \sum_{j =1}^{K} \Eb_{\pi}\left( \ell(\cA_{\pi}, z_{\pi(1)}) \mid z_{\pi(1)}\in C_j \right)\frac{|N_j^m|}{m} \right| \\
& \leq \left| \frac{1}{m} \sum_{j=1}^{K}\sum_{z \in N_j^m} \max_{z' \in C_j} \left| \ell \left(\cA_{\pi},z\right) - \ell \left(\cA_{\pi},z'\right) \right| \right|  + \left| \max_{z \in \cZ} \left| \ell\left( \cA_{\pi}, z \right) \right| \sum_{j=1}^{K} \left| \frac{|N_j^m|}{m}- P(z_{\pi(1)} \in C_j ) \right| \right| \\
& \leq \varepsilon + M\sum_{j=1}^{K} \left| \frac{|N_j^m|}{m}- \frac{|C_j|}{n} \right|.
\end{align*}
Finally, for $t\geq 0$, we apply \cref{lemma:Bretagnolle_Huber_permutation_extension} to bound the probability, 
\begin{equation*}
P\left(\sum_{j=1}^{K} \mleft| \frac{|N_j^m|}{m}- \frac{|C_j|}{n} \mright|\geq t\right).
\end{equation*}

Following the derivations from the proof of \cref{lemma:Bretagnolle_Huber_extension}, we obtain  
\begin{align*}
P\left(\sum_{j=1}^{K} \mleft| \frac{|N_j^m|}{m}- \frac{|C_j|}{n} \mright|\geq t\right) 
&\leq \sum_{S \subset [K]} P\left( \left| \sum_{i=1}^{m} Y_{i}^{(S)} - \Eb\left(\sum_{i=1}^{m} Y_i^{(S)}\right) \right| \geq \frac{mt}{2} \right),
\end{align*}
where  
\begin{equation*}
Y_{i}^{(S)} = \sum_{j \in S} \mathbf{1}_{\{z_{\pi(i)} \in C_j\}}.
\end{equation*}  
Applying \cref{lemma:Bretagnolle_Huber_permutation_extension}, we obtain  
\begin{align*}
P\left(\sum_{j=1}^{K} \mleft| \frac{|N_j^m|}{m}- \frac{|C_j|}{n} \mright|\geq t\right)  
\leq \sum_{S \subset [K]} 2\exp\left(\frac{-t^2 m}{2|S|^2} \right) \leq 2^{K+1} \exp\left(\frac{-t^2 m}{2K^2} \right).
\end{align*}  
Thus, for any $\delta>0$, with probability at least $1-\delta$, we have  
\begin{equation*}
\sum_{j=1}^{K} \mleft| \frac{|N_j^m|}{m}- \frac{|C_j|}{n} \mright| \leq \frac{1}{\sqrt{m}} \cdot K \cdot \sqrt{2(K+1)\log2 + 2\log\left(\frac{1}{\delta}\right)}.
\end{equation*}  
Following the same reasoning for  
\begin{equation*}
|R_u -\Eb_{\pi}(R_u)|,
\end{equation*}  
we derive,
\begin{align*}
\left|R_m - R_u\right| \leq  
2\varepsilon + \left( \frac{1}{\sqrt{m}}+\frac{1}{\sqrt{u}} \right) \cdot M \cdot K \cdot \sqrt{2(K+1)\log2 + 2\log\left(\frac{1}{\delta}\right)} + L_{\ell}\beta.
\end{align*}  
\end{proof}

\section{Missing proofs from Section~\ref{sec:robustness_under_dependency_for_graphs}}
\label{app_sec:robustness_under_dependency_for_graphs}

The main theorems in this section are direct consequences of results established earlier. Specifically, \cref{thm:binaryclassificationdatadepend} follows immediately by combining \cref{thm:Xu_Mannor_noniid} with \cref{thm:lipschitzimpliesrobustness}, and \cref{thm:binaryclassificationtransductive} by combining \cref{thm:Xu_Mannor_transductive} with \cref{thm:lipschitzimpliesrobustness}. 

Specifically, the Lipschitz constant $C$ appearing in \cref{thm:binaryclassificationdatadepend} and \cref{thm:binaryclassificationtransductive} is the same and is given by
\begin{equation*}
C = 2 \widetilde{L} L_{\ell} \left( \prod_{t=1}^{L} L_{\varphi_t} \right),
\end{equation*}
where $L_{\varphi_t}$ denotes the Lipschitz constant of the $t$-th message-passing layer as defined in \cref{def:sum_mpnnsgraphs}, $L_{\ell}$ the Lipschitz constant of the loss function, and $\widetilde{L}$ is the Lipschitz constant established in \cref{prop:Lipschitz_generalized_MPNNs}.

\textbf{On assumptions in node prediction bounds}
Previous works such as \citet{scarselli2018vapnik, Ver+2019, Gar+2020} study generalization in node prediction under an inductive setting, but rely on strong and limiting assumptions. Specifically, they assume that graphs decompose into independent substructures, similarly to the computational trees we defined in \cref{app_sec:1WL}, and that the data distribution is defined over the product space of these trees and their labels. A more restrictive assumption is that these tree-label pairs are sampled independently, even when the trees originate from the same graph, thus ignoring inherent dependencies within the graph structure.
In contrast, our framework treats each graph as a relational object, without reducing it to a collection of unrollings. While we relax the i.i.d.\@ assumption, we still capture dependencies within the training set through a mild and natural condition: independence is assumed only across nodes from different graphs. In contrast, dependencies are allowed among nodes within the same graph. This setting better reflects real-world scenarios and provides more realistic assumptions than prior work.

Finally, we state and prove the following corollary, showing that if we restrict our attention to graphs with maximum node degree $q$, for some $q \in \Nb$, we can bound the covering number from \cref{thm:binaryclassificationdatadepend}.
\begin{corollary}
\label{cor:boundeddegreegeneralization}
Let $\cG_{d,q}^{(-1,1)}$ be the space of featured graphs with node features in $(-1,1)^d$ and maximum node degree $q$, and consider the setting from \cref{thm:binaryclassificationdatadepend}.  
Then,  
\begin{equation*}
\left| \ell_{\text{exp}}(\cA) - \ell_{\text{emp}}(\cA_{\cS}) \right| \leq 2C\varepsilon + M \sqrt{\frac{ D_{\cS} \left( \left(  4 \log2 \right)\left(\frac{3}{\varepsilon}\right)^{d\cdot Q} +2\log 2 + 2\log\left(\frac{1}{\delta}\right)\right)}{N}}, \quad{\text{for all } \varepsilon >0.}
\end{equation*}
where $Q = \frac{q^{L+1}-1}{q-1}$, $C$, and $D_{\cS}$ as previously.
\end{corollary}

\begin{proof}
The proof is a straightforward application of the bound on the covering number of a $d$-dimensional unit Euclidean ball with radius $0 < \epsilon < 1$, which is at most $\left( \frac{3}{\epsilon} \right)^d$ (see \citet[Lemma 5.13]{Handel2014}).
\end{proof}

\section{Experimental analysis}
\label{app_sec:exp}
This appendix presents the experimental protocol underlying the results and insights discussed in \cref{sec:experiments}. The source code of all methods is available in the supplementary material.

\textbf{Sampling strategies for \textbf{Q2}}
In total we consider three different sampling strategies for inductive node classification tasks in \textbf{Q2}. Since we want to estimate the difference in generalization capabilities for sampling methods depending on the number of training graphs sampled from we provide methods using the same number of nodes with different distinct training graphs. Throughout the experiments, we fixed the random seed to ensure the same sampling process in each iteration. 

First we provide \emph{random} sampling, selecting a subsample of graphs by randomly ordering the graphs and sampling nodes graph-by-graph (except graphs containing test set nodes), exhausting each graph before moving to the next or stopping once the desired number of nodes is reached. This results in a small number of graphs, where training nodes are obtained exclusively, 
leading to a larger chromatic number $\chi(G[\cS])$ (as in \cref{thm:Xu_Mannor_noniid}) or, equivalently, a larger $D_{\cS}$ (following the notation in \cref{thm:binaryclassificationdatadepend}). 

In contrast \emph{uniform} sampling uses the same number of randomly ordered nodes from each graph in the training dataset. This ensures that at least one node is sampled from every training graph and during training each graph is seen at least once. With the selected number of nodes for both datasets we sample multiple nodes from each graph. 

As an addition to uniform and random sampling we provide the \emph{mixed} sampling process. This method uses either random or uniform sampling for a specified number of graphs, denoted by r and u in the name respectively. In addition, we conduct our experiments with 1000, 4000 and 8000 distinct graphs seen during training for each dataset. This allows us to set the number of graphs instead of just setting the number of nodes, resulting in a fine grained exploration of influences on generalization errors seen for both datasets

Throughout all sampling strategies, we first choose a subset of graphs from the dataset and uniformly sample $n_{test}$ nodes to
form the test set. We then fix the number of training nodes to $n_{train}$. 

    \begin{table}
    \resizebox{\columnwidth}{!}{
    \begin{tabular}{|c|c|c|c|c|c|c|c|c}
    PATTERN & Mixed-1k-r &  Mixed-1k-u & Mixed-4k-r & Mixed-4k-u & Mixed-8k-r &  Mixed-8k-u & random & uniform \\
    \hline
    $n_{train}$ & 120000 & 120000& 120000& 120000& 120000& 120000 & 120000 & 120000 \\
    $n_{test}$  & 116232 & 116232 & 116232 &116232 & 116232 & 116232 & 116232 & 116232\\
    $\mathcal{D}_S$& 186 & 120 & 183 & 30 & 170 & 15 & 186 & 9 \\
    Training loss & 0.445 \tiny$\pm$ 0.01 & 1.5885 \tiny$\pm$ 0.002 & 0.44 \tiny$\pm$ 0.01 & 1.5806 \tiny$\pm$ 0.003 & 0.43 \tiny$\pm$ 0.011 & 1.579 \tiny$\pm$ 0.01 & 0.2485 \tiny$\pm$ 0.0073 & 1.5739 \tiny$\pm$ 0.0123\\
    Test loss & 1.7685 \tiny$\pm$ 0.005 & 1.5203 \tiny$\pm$ 0.03 & 1.763 \tiny$\pm$ 0.002 & 1.516 \tiny$\pm$ 0.03 & 1.749 \tiny$\pm$ 0.013 & 1.515 \tiny$\pm$0.034 &  1.8007 \tiny$\pm$ 0.0426 & 1.5114 \tiny$\pm$ 0.0152 \\
    Gen. Gap & 1.3235 & 0.0682 & 1.323 & 0.0646 & 1.319 & 0.064 & 1.552 & 0.0625 \\ 

    \hline
    CLUSTER & Mixed-1k-r &  Mixed-1k-u & Mixed-4k-r & Mixed-4k-u & Mixed-8k-r &  Mixed-8k-u & random & uniform \\
    \hline
    $n_{train}$ & 120000 & 120000& 120000& 120000& 120000& 120000 & 120000& 120000\\
    $n_{test}$  & 116633  & 116633  & 116633  &116633  & 116633  & 116633 & 116633 & 116633 \\
     $\mathcal{D}_S$ & 177 & 120 & 176 & 30 & 177 & 15 & 190 & 10 \\
    Training loss & 0.1244 \tiny$\pm$ 0.003 & 0.4002 \tiny$\pm$ 0.01 & 0.1187 \tiny$\pm$ 0.003 & 0.3806 \tiny$\pm$ 0.0161 & 0.1157 \tiny$\pm$ 0.001 & 0.394 \tiny$\pm$ 0.018 & 0.1741 \tiny$\pm$0.0260 & 0.2578\tiny$\pm$0.1293 \\
    Test loss & 0.3920 \tiny$\pm$ 0.01& 0.3359 \tiny$\pm$ 0.029& 0.3869 \tiny$\pm$ 0.001 & 0.3396 \tiny$\pm$ 0.024 & 0.3882 \tiny$\pm$ 0.002 & 0.356 \tiny$\pm$ 0.007 & 0.7964 \tiny$\pm 0.3241$ & 0.3579 \tiny$\pm0.0043$  \\
    Gen. Gap & 0.2676 & 0.0643 & 0.2682 & 0.041 & 0.2725 & 0.038 &  0.6223 & 0.1001 \\ 
    \end{tabular}
    }
      \caption{Generalization results for different sampling strategies related to \textbf{Q2} for inductive node classification on the CLUSTER and PATTERN dataset. Strategy random refers to training nodes sampled from a few graphs (resulting in higher sample dependency), while strategy uniform uses nodes sampled uniformly across many distinct graphs. In addition we use the Strategy mixed with different numbers of distinct graphs and sampling processes. Both strategies use the same number of training and test nodes. $D_{\cS}$ denotes the maximum sampled nodes of a single graph in the training set. Further, $n_{\text{train}}$ and $n_{\text{test}}$ denote the number of train and test nodes.}
      \label{table:Q2_complete}
    \end{table}

\textbf{Datasets} To investigate \textbf{Q1} we focus on the transductive setting for node predictions, and we use the datasets Wisconsin, Cornell, Texas \citep{DPeiGeom-GCN2020} available as part of the WebKB dataset available at \url{https://github.com/bingzhewei/geom-gcn} and Cora dataset \citep{YANGRevisiting2016, Sen2008Collective} available under the CC-BY 4.0 license at \url{https://pytorch-geometric.readthedocs.io/en/latest/generated/torch_geometric.datasets.Planetoid.html}. To investigate \textbf{Q2} under the inductive setting, we consider Pattern, and Cluster datasets \citep{DwivediBenchmarking2023}. These datasets are available under the MIT license at \url{https://github.com/graphdeeplearning/benchmarking-gnns}. 

We compute the covering number for all transductive datasets. In the case of the inductive datasets, we omit the covering number computation due to its similarity to the computations provided by \citep{Vas+2024}. Common dataset statistics and properties are in \Cref{tab:datasetstats}. 

In addition we consider synthetic datasets generated using Erdos-Reyni graphs for \textbf{Q3}. 
 We generate the longest shortest path Erdős–Rényi graph LSP-ER(n,p,cc), used for binary node classification, as follows: We first create an Erdős–Rényi (ER) graph with  nodes $n$ and edge probability $p$. In a second step, we connect currently disconnected components in the graph. For this $cc$ denotes the number of randomly sampled links between disconnected components. For example,  would mean one link between disconnected components. This process is executed for pairs of graph components. Then, we identify the longest shortest path and assign label 1 to all nodes involved in this path. In the case of multiple longest shortest paths, we regenerate the graph to ensure an unambiguous result. All other nodes are labeled 0.
 
We chose this generation process to allow some control over how the covering number of the node space varies across datasets. Especially with few suitable small scale datasets available for transductive node-level tasks we aim to generate graphs to determine the behavior of our generalization bounds. Furthermore, with generated data we are able to select a suitable feature space and conduct ablation studies with regard to graph density. Specifically, smaller $p$ (fixed)  and larger $n$ produce sparser graphs in which nodes tend to have similar computation trees. This results in smaller distances between nodes and, consequently, tighter generalization bounds. A similar effect is observed when decreasing $p$ for fixed $n$. 
\begin{table}
\centering
\resizebox{0.7\columnwidth}{!}{
\begin{tabular}{ @{}lcccccc@{} } 
\hline
\bf Hidden dimension & \bf 16 & \bf 32 & \bf 64 & \bf 128 & \bf 256 & \bf 512\\
\hline
Calculated bound (n=500) & 3.575 & 4.602 & 2.71 & 4.53 & 4.04 & 2.578 \\
Generalization Gap (n=500)& 0.471 & 0.461 & 0.436 & 0.438 & 0.424 & 0.425 \\
\hline
Calculated bound (n=1000)& 5.27 & 9.03 & 4.59 & 4.67 & 5.92 & 4.94\\
Generalization Gap (n=1000)& 0.242 & 0.213 & 0.201 & 0.204 & 0.20 & 0.194\\
\hline
\end{tabular}
\label{table:LSPhiddendim}
}
\caption{Ablation study for the size of hidden layers and node prediction bounds on LSP-ER(n,p,cc) graphs with p=,cc=1}
\end{table}

\begin{table}
\centering
\resizebox{0.7\columnwidth}{!}{
\begin{tabular}{ @{}lcccc@{} } 
\hline
\bf Dataset &  \bf (100,0.002,1) & \bf (200,0.002,1) & \bf (500,0.002,1) & \bf (1000,0.002,1) \\
\hline
Calculated bound & 5.88 & 5.17 & 4.53 & 4.66 \\
Generalization Gap & 0.56 & 0.548 & 0.44 & 0.21 \\
\hline
\end{tabular}
}
\caption{Evaluation of node prediction bounds with p=0.002, cc=1 for LSP-ER(n,p,cc) graphs}
\label{table:LSP0002}
\end{table}

\begin{table}
\centering
\resizebox{0.85\columnwidth}{!}{
\begin{tabular}{ @{}lcccccc@{} } 
\hline
\bf Dataset & \bf (100,0.0005,1) & \bf (200,0.0005,1) & \bf (300,0.0005,1) & \bf (500,0.0005,1) & \bf (1000,0.0005,1) \\
\hline
Calculated bound & 6.59 & 5.65 & 4.87 & 4.83 & 2.65 \\
Generalization Gap & 0.581 & 0.537 & 0.527 & 0.416 & 0.416 \\
\hline
\end{tabular}
}
\caption{Evaluation of node prediction bounds with p=0.0005, cc=1 for LSP-ER(n,p,cc) graphs}
\label{table:LSP00005}
\end{table}

\begin{table}
\centering
\resizebox{0.7\columnwidth}{!}{
\begin{tabular}{ @{}lcccccc@{} } 
\hline
\bf Dataset & \bf (200,0.0005,1) &\bf (200,0.001,1) & \bf (500,0.0005,1) &\bf (500,0.001,1)\\
\hline
Calculated bound &  5.65& 6.07 & 4.83 & 5.66\\
Generalization Gap& 0.54& 0.574& 0.416 & 0.462\\
\hline
\end{tabular}
}
\caption{Node prediction bounds with increasing edge probability for LSP-ER(n,p,cc) graphs}
\label{table:increasingedge}
\end{table}

\begin{table}
\caption{Statistics for each dataset considered in \Cref{sec:experiments}\label{tab:datasetstats}}
	\centering
	\resizebox{.70\textwidth}{!}{ 	\renewcommand{\arraystretch}{1.05}
		\begin{tabular}{@{}lcccccc@{}}
			\toprule   & \multicolumn{4}{c}{\textbf{Dataset}}
			\\\cmidrule{2-7} 
            & \textsc{Texas} & \textsc{Wisconsin} & \textsc{Cornell} & \textsc{Cora} & \textsc{Pattern} & \textsc{Cluster} \\
			\midrule
            \small \# Graphs         & 1 & 1 & 1 & 1  & 14000 & 12000 \\
			\small \# Avg. nodes         & 183 & 251 & 183 & 2708  & 118.9 & 117.2 \\
			\small \# Avg. edges         & 325 & 515 & 298 & 10556 & 6098.9 & 4303.9 \\
            \small \# Classes       & 5   & 5   & 5   & 7 & 2 & 6 \\
            \bottomrule
		\end{tabular}}
\end{table}

\textbf{Neural architectures} To address \textbf{Q1}, we use randomly initialized GIN and SEAL architectures with three layers. We further incorporate node features into the GIN architecture as initial inputs to the model. Nonlinearity is introduced via the ReLU function as discussed in our theoretical examination. A complete list of hyperparameters used is provided in \Cref{tab:hyperparams}. Note that the slope of the line that upper bounds the observations in \cref{fig:corrMPNNs_link_L2} and \Cref{fig:corrMPNNs_link_L3} can be used as an upper bound on the Lipschitz constant $C$ in \cref{thm:binaryclassificationtransductive}. 

In the case of SEAL, we also use GIN as the underlying GNN architecture. We sample one negative link for each link in the dataset, providing an equal amount of positive and negative links. Across all datasets, we use the commonly used data splits, and for SEAL, an 80/10/10 train-valid-test split. Similar to the node prediction tasks, we employ ReLU nonlinearity and sum pooling to compute each link representation. The subgraph sampling is done as proposed in the original SEAL paper \citep{Zha+2018}, omitting the target link in its respective subgraph.

For \textbf{Q2} we use the same GIN architecture as in \textbf{Q1} but trained it for 200 epochs to get a suitable generalization error. Furthermore, we use the hyperparameters detailed in \Cref{tab:hyperparams_node_inductive}. 

In the case of \textbf{Q3} we use a randomly initialized GIN architecture with a hidden dimension of 16 or 32 and 3 layers. We incorporate node features into the GIN architecture as initial inputs to the model. Otherwise, we use the same GIN architecture as for \textbf{Q1}. 

\textbf{Hyperparameters}
For inductive node classification on Pattern and Cluster, we trained for 200 epochs using a modified version of GIN, aligning with \Cref{def:sum_mpnnsgraphs}. In addition, we tuned the learning rate using the set $\{0.01, \textbf{0.001}, 0.0001 \}$, while employing the Adam optimizer \citep{KingmaAdam2015}. Across all tasks and models, we used a batch size of 32 and set dropout to 0.1. Further, we do not use learning rate decay across all datasets. We report results on the inductive node classification tasks and the sampling strategies in \Cref{table:Q2}. A list of all hyperparameters used can be found in \Cref{tab:hyperparams_link}, \Cref{tab:hyperparams_node_inductive} and \Cref{tab:hyperparams} for each experiment. 

For the correlation experiments, we use a randomly initialized GIN model and a SEAL model using GIN layers. We use 3 layers each to align the computation with possible computations of the generalization bounds. We set the hidden dimension to 16 for Cornell, Texas, and Wisconsin, and 32 for Cora, respectively. 
Since we do not train the models, we omit further training-specific hyperparameters.

In case of \textbf{Q3} we use a randomly initialized GIN model with 3 layers to compute the generalization bounds. In accordance with \textbf{Q1} we set the hidden dimension of the model to 16 or 32, dropout to 0.1 and do not train the model. Therefore, we omit training-specific hyperparameters. In contrast to real-world datasets and trained models, we utilize all available nodes for evaluating the generalization bound.

Furthermore, we report the runtime and memory usage of our experiments in \Cref{tab:runtime} and \Cref{tab:runtime_lsp}. 
We provide a PyTorch Geometric implementation for each model. All our experiments were executed on a system with 12 CPU cores, an Nvidia L40 GPU, and 120GB of memory.

\textbf{Experimental protocol and model configuration} 
To evaluate \textbf{Q1}, we measure whether the perturbed inputs to the generalized distance lead to perturbations in the MPNN outputs. We use the same 80/10/10 train-valid-test split as for the other tasks. We then select nodes or links randomly from the transductive dataset for each dataset. We then compute the (unrolling-based) distances for two sampled nodes or links and compare them to the Euclidean distance of their respective GIN or SEAL outputs. Sample plots for selected nodes showing the correlation between the generalized distance and MPNN outputs are shown in \Cref{fig:corrMPNNs_link_L2}, \Cref{fig:corrMPNNs_node_L2}, \Cref{fig:corrMPNNs_link_L3}, and \Cref{fig:corrMPNNs_node_L3}.

Concerning \textbf{Q2}, we provide three scenarios for sampling nodes from the graphs in the dataset. First, we consider the case where the train dataset nodes are sampled from specific graphs. Secondly, we assume uniform node subsampling across all training graphs. Finally, we consider the case of both sampling methods with the number of graphs seen during training fixed. We fix the test set for these experiments to a randomly determined train-valid-test split. We report the obtained results in \Cref{table:Q2} and \Cref{table:Q2_complete}, which showcase the difference in generalization performance between each sampling method. Node samples denote the fraction of nodes used for the computation. 
\begin{table}
\caption{Hyperparameters used for correlation experiments in \Cref{fig:corrMPNNs_link_L2}, \Cref{fig:corrMPNNs_link_L3} with two and three GIN layers.  \label{tab:hyperparams_link}}
	\centering
	\resizebox{.60\columnwidth}{!}{ 	\renewcommand{\arraystretch}{1.05}
		\begin{tabular}{@{}lcccc@{}}
			\toprule   & \multicolumn{4}{c}{\textbf{Dataset}}
			\\\cmidrule{2-5} 
                 & \textsc{Texas}  & \textsc{Cornell} & \textsc{Wisconsin} & \textsc{Cora}  \\
            \midrule
            Embedding dim.& 16 & 16 & 16 & 32  \\
            Hidden dim.& 16 & 16 & 16  & 32  \\
            Dropout & 0.1& 0.1 & 0.1 & 0.1  \\
            Node samples & 1 & 1 & 1 & 0.25 \\
			\bottomrule
		\end{tabular}}
\end{table}

\begin{table}
\caption{Hyperparameters used for inductive node classification experiments in \textbf{Q2} and for generalization bound computation experiments on the LSP-ER(n,p,cc) datasets. For LSP-ER(500,p,cc) and LSP-ER(1000,p,cc) we used 32 as hidden dimension, otherwise 16.  \label{tab:hyperparams_node_inductive}}
	\centering 
	\resizebox{.55\columnwidth}{!}{ 	\renewcommand{\arraystretch}{1.05}
		\begin{tabular}{@{}lccc@{}}
			\toprule   & \multicolumn{3}{c}{\textbf{Dataset}}
			\\\cmidrule{2-4} 
            & \textsc{Pattern}    & \textsc{Cluster} & \textsc{LSP-ER}   \\  
            \midrule
            Learning Rate & 0.001&  0.001 & --\\
            Batch Size& 32& 32 & 32\\
            Embedding dim.& 64 & 64  & 64\\
            Hidden dim.& 64 & 64  & 16/32\\
            Epochs & 200 & 200 & -- \\
            LR decay & 0& 0 & -- \\
            Gradient norm & 1& 1 & --\\
            Dropout & 0.1 & 0.1 & 0.1 \\
			\bottomrule
		\end{tabular}}
\end{table}
To get generalization bounds for \textbf{Q3}, we evaluate the perturbed inputs to the generalized distance, which leads to perturbations in MPNN outputs. Similar to \textbf{Q1}, we select nodes randomly from the transductive dataset and compute the unrolling-based distances and compare them to the Euclidean distances of MPNN outputs. 
To compute the generalization bound, we estimate the Lipschitz constant by linearly bounding the computed MPNN outputs based on the node distance. In a second step, we then compute the generalization bound by determining the optimal covering numbers required for the bound estimate. With the optimal result obtained through searching over possible covering numbers, we further get the loss bound and actual generalization gap from the experimental data of our evaluation.  
We report results in \Cref{table:LSP0001}, \Cref{table:LSP0002}, \Cref{table:LSP00005} and \Cref{table:increasingedge}. Further, we showcase ablation results for different hidden sizes in \Cref{table:LSPhiddendim}. 

\begin{table}
\caption{Runtime and Memory Usage for each experiment in \Cref{sec:experiments}\label{tab:runtime}. The first value denotes the runtime in seconds of each experiment, and the second value the used VRAM in MB. We do not report VRAM used in the correlation tasks as we only do a single forward pass. All results were obtained on a single computing node with an Nvidia L40 GPU and 128GB of RAM.}
	\centering
	\resizebox{\textwidth}{!}{ 	\renewcommand{\arraystretch}{1.05}
		\begin{tabular}{@{}lcccccc@{}}
			\toprule   & \multicolumn{6}{c}{\textbf{Dataset}}
			\\\cmidrule{2-7} 
            & \textsc{Texas}    & \textsc{Wisconsin}      & \textsc{Cornell}  & \textsc{Cora} & \textsc{Pattern} & \textsc{Cluster}  \\
			\midrule
            Correlation (SEAL, L=2,3) & 13.32/-& 19.56/-& 13.76/-&- &- &- \\
			Correlation (Node, L=2,3)  & 4.32/- & 7.42/- & 4.45/- & 469.02/- & -& -\\
            Inductive Node Uniform  &- & - & - & - & 199.08/96.71 & 169.64 / 78.77\\
			Inductive Node Random  & - & - & - & - & 257.24/88.67 & 233.11 / 69.46 \\
            Inductive Node Mixed-1k-r &- & - & - & - & 193.42/96.39 & 184.94/75.54 \\
            Inductive Node Mixed-1k-u &- & - & - & - & 219.44/88.67 & 196.38/68.14 \\
            Inductive Node Mixed-4k-r &- & - & - & - & 207.73/96.40 & 177.82/75.53 \\
            Inductive Node Mixed-4k-u &- & - & - & - & 216.34/88.67 & 194.35/68.15 \\
            Inductive Node Mixed-8k-r &- & - & - & - & 179.41/96.40 & 158.61/75.10 \\
            Inductive Node Mixed-8k-u &- & - & - & - & 218.84/88.70 & 200.69/68.14 \\
			\bottomrule
		\end{tabular}
        }
\end{table}

\begin{table}
	\centering
	\resizebox{\textwidth}{!}{ 	\renewcommand{\arraystretch}{1.05}
		\begin{tabular}{@{}lcccccc@{}}
			\toprule   & \multicolumn{6}{c}{\textbf{Dataset}}
			\\\cmidrule{2-7} 
            & \textsc{(100,0.001,1)}    & \textsc{(100,0.0005,1)}      & \textsc{(100,0.002,1)}  & \textsc{(200,0.001,1)} & \textsc{(200,0.0005,1)} & \textsc{(200,0.002,1)}  \\
			\midrule
            Bound Computation (Q3)  & 18.10/- & 15.31/- & 14.78/- & 33.21/- & 36.71/- & 61.02/- \\ 
			\bottomrule
		\end{tabular}
        }

	\centering
	\resizebox{\textwidth}{!}{ 	\renewcommand{\arraystretch}{1.05}
		\begin{tabular}{@{}lcccccc@{}}
			\toprule   & \multicolumn{6}{c}{\textbf{Dataset}}
			\\\cmidrule{2-7} 
            & \textsc{(500,0.001,1)}    & \textsc{(500,0.0005,1)}      & \textsc{(500,0.002,1)}  & \textsc{(1000,0.001,1)} & \textsc{(1000,0.0005,1)} & \textsc{(1000,0.002,1)}  \\
			\midrule
            Bound Computation (Q3)  & 287.01/- & 330.64/- & 314.74/- & 2157.13/- & 3122.56/- & 1937.44/- \\ 
			\bottomrule
		\end{tabular}
        }
        \caption{Runtime and Memory Usage for each experiment in \Cref{sec:experiments}\label{tab:runtime_lsp} using the LSP-ER datasets. The first value denotes the runtime in seconds of each experiment, and the second value denotes the used VRAM in MB. We do not report VRAM used in the correlation tasks, as we only do a single forward pass. All results were obtained on a single computing node with an Nvidia L40 GPU and 128GB of RAM.}
\end{table}

\begin{table}
\caption{Hyperparameters used for correlation experiments in \Cref{fig:corrMPNNs_node_L2} and \Cref{fig:corrMPNNs_node_L3}. One-to-one sampling describes the process of sampling one negative link to each link obtained from the graph.\label{tab:hyperparams}}
	\centering
	\resizebox{.60\textwidth}{!}{ 	\renewcommand{\arraystretch}{1.05}
		\begin{tabular}{@{}lccc@{}}
			\toprule   & \multicolumn{3}{c}{\textbf{Dataset}}
			\\\cmidrule{2-4} 
                  & \textsc{Texas}  & \textsc{Cornell} & \textsc{Wisconsin}   \\
            \midrule
            Embedding dim.& 16 & 16  &  16 \\
            Hidden dim.& 16 &  16 & 16  \\
            Link sample size & all & all & all  \\
            Link sampling & one to one & one to one & one to one  \\
            Dropout & 0.1 & 0.1 & 0.1  \\
            Use node features & True & True & True \\
			\bottomrule
		\end{tabular}}
\end{table}

\begin{figure}
        \begin{center}
        \centering
\includegraphics[width=1\textwidth]{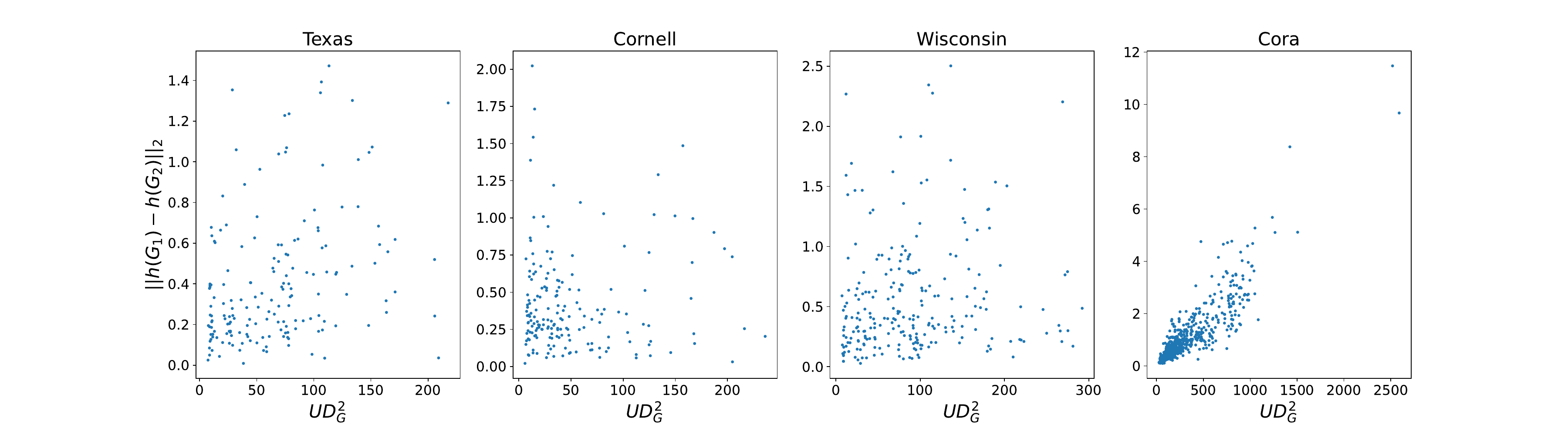}
        \end{center}
	\caption{Correlation between GIN-MPNN outputs and the corresponding unrolling distance across real-world datasets for two GIN layers.\label{fig:corrMPNNs_link_L2}}
\end{figure}

\begin{figure}
        \begin{center}
        \centering
\includegraphics[width=1.0\textwidth]{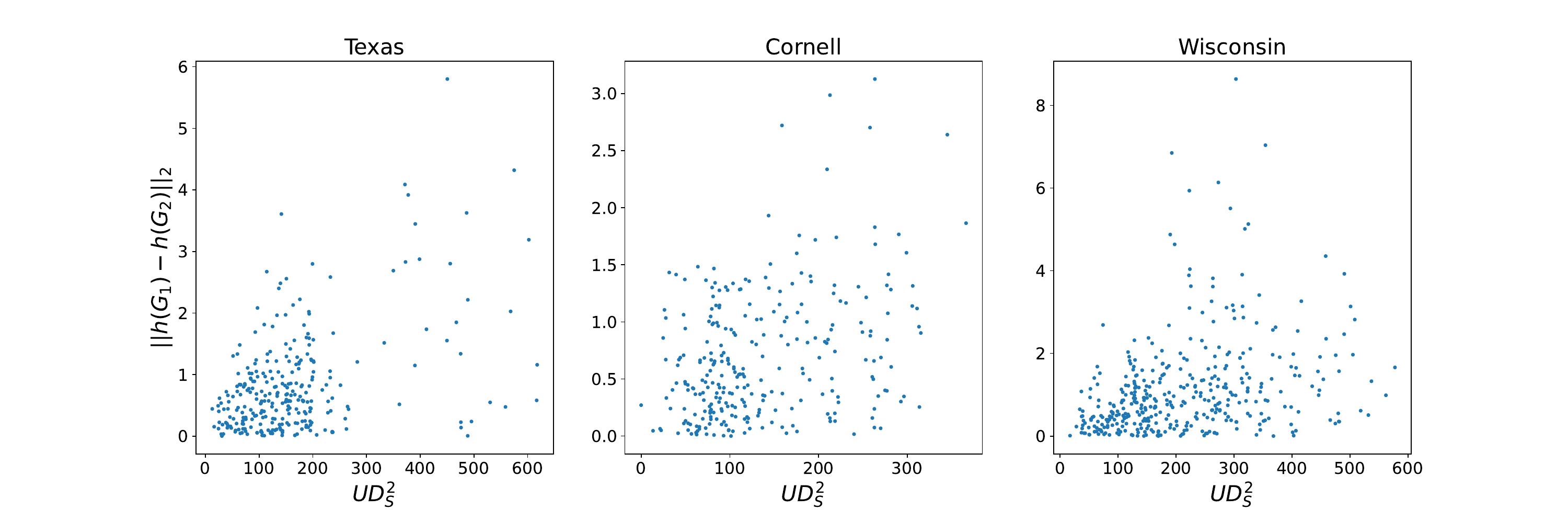}
    \end{center}
	\caption{Correlation between SEAL-MPNN outputs and the corresponding unrolling distance across real-world datasets for two GIN layers.\label{fig:corrMPNNs_node_L2}}
\end{figure}

\begin{figure}
        \begin{center}
        \centering
\includegraphics[width=1\textwidth]{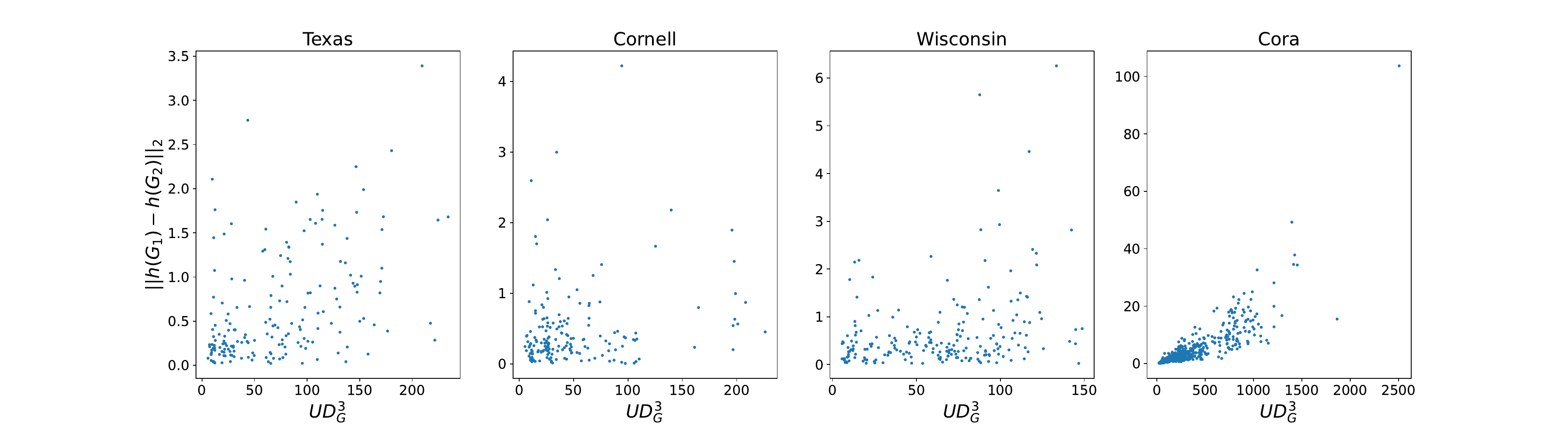}
        \end{center}
	\caption{Correlation between GIN-MPNN outputs and the corresponding unrolling distance across real-world datasets for three GIN layers.\label{fig:corrMPNNs_link_L3}}
\end{figure}

\begin{figure}
        \begin{center}
        \centering
\includegraphics[width=1.0\textwidth]{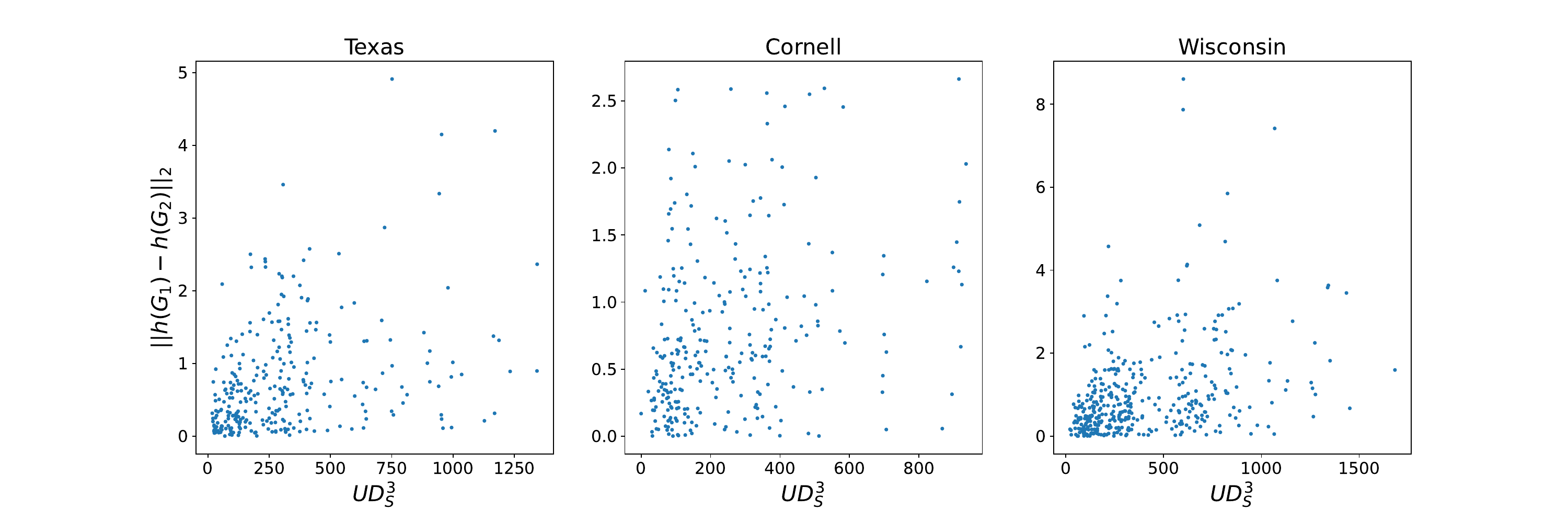}
    \end{center}
	\caption{Correlation between SEAL-MPNN outputs and the corresponding unrolling distance across real-world datasets for three GIN layers.\label{fig:corrMPNNs_node_L3}}
\end{figure}

\end{document}

%% file: macros.tex
\newcommand{\cA}{\mathcal{A}}

\newcommand{\cF}{\mathcal{F}}
\newcommand{\cG}{\mathcal{G}}
\newcommand{\cH}{\mathcal{H}}
\newcommand{\cN}{\mathcal{N}}
\newcommand{\cO}{\mathcal{O}}
\newcommand{\cP}{\mathcal{P}}

\newcommand{\cS}{\mathcal{S}}
\newcommand{\cT}{\mathcal{T}}

\newcommand{\cV}{\mathcal{V}}

\newcommand{\cX}{\mathcal{X}}
\newcommand{\cY}{\mathcal{Y}}
\newcommand{\cZ}{\mathcal{Z}}

\newcommand{\Eb}{\mathbb{E}}

\newcommand{\Nb}{\mathbb{N}}

\newcommand{\Rb}{\mathbb{R}}

\newcommand{\FNN}{\mathsf{FNN}}
\newcommand{\MPNN}{\textsf{MPNN}}
\newcommand{\UNR}[1]{\ensuremath{\mathsf{unr}(#1)}}

\newcommand{\wlone}{$1$\textrm{-}\textsf{WL}}

\newcommand{\hb}{\mathbold{h}}

\newcommand{\UPD}{\mathsf{UPD}}
\newcommand{\AGG}{\mathsf{AGG}}

\newcommand{\REL}{\mathsf{RELABEL}}

\newcommand{\new}[1]{\emph{#1}}

\renewcommand{\vec}[1]{\mathbold{#1}}
\newcommand{\oms}{\{\!\!\{}
\newcommand{\cms}{\}\!\!\}}
\newcommand{\tup}[1]{{(#1)}}
\DeclarePairedDelimiterX{\norm}[1]{\lVert}{\rVert}{#1}

\DeclareFontFamily{U}{mathx}{\hyphenchar\font45}
\DeclareFontShape{U}{mathx}{m}{n}{<-> mathx10}{}
\DeclareSymbolFont{mathx}{U}{mathx}{m}{n}
\DeclareMathAccent{\widebar}{0}{mathx}{"73}

\DeclareMathOperator{\tr}{\mathrm{Tr}}

%% file: rho.tex
\begin{tikzpicture}[transform shape, scale=1]

\definecolor{lgreen}{HTML}{4DA84D}
\definecolor{fontc}{HTML}{403E30}
\definecolor{llred}{HTML}{FF7A87}
\definecolor{llblue}{HTML}{7EAFCC}
\definecolor{lviolet}{HTML}{756BB1}
\definecolor{lorange}{HTML}{FF7F0E}

\newcommand{\gnode}[5]{%
    \draw[#3, fill=#3!20] (#1) circle (#2pt);
    \node[anchor=#4,fontc] at (#1) {\tiny #5};
}

\begin{scope}[shift={(0,0)}]
\begin{scope}[shift={(0,0)}]

\node[fontc] at (-0.75,-0.4) {\scalebox{3}{\textnormal{(}}};

\draw[decorate,line width=0.6pt,decoration={brace,amplitude=4pt},fontc] ((-0.49,-0.9) -- (-0.49,0.1) node[midway,xshift=9pt,llblue] {};
\draw[decorate,line width=0.6pt,decoration={brace,amplitude=4pt},fontc] ((0.49,0.1) -- (0.49,-0.9) node[midway,xshift=9pt,llblue] {};

\node[fontc] at (0.65,-0.9) {\scalebox{1.2}{\textnormal{,}}};

\draw[fontc] (0,0) -- (-0.25,-0.4);
\draw[fontc] (0,0) -- (0.25,-0.4);
\draw[fontc] (0.25,-0.4) -- (0.25,-0.8);

\gnode{0,0}{2}{llblue}{south}{\scalebox{0.5}{$(1,0)$}};
\gnode{-0.25,-0.4}{2}{llblue}{south}{};
\gnode{0.25,-0.4}{2}{llblue}{north}{};
\gnode{0.25,-0.8}{2}{llblue}{north}{\scalebox{0.5}{$(1,1)$}};

\draw[decorate,line width=0.6pt,decoration={brace,amplitude=2pt},fontc] ((0.4,-1.07) -- (-0.4,-1.07) node[midway,xshift=9pt,llblue] {};
\node[fontc] at (0,-1.3) {\scalebox{0.5}{$\cF^{(1)}(G_1,S_1)$}};

\node[fontc] at (-0.34,-0.25) {\scalebox{0.5}{\tiny $(0,1)$}};
\node[fontc] at (0.34,-0.25) {\scalebox{0.5}{\tiny $(1,0)$}};

\end{scope}

\begin{scope}[shift={(2,0)}]

\draw[decorate,line width=0.6pt,decoration={brace,amplitude=4pt},fontc] ((-0.7,-0.9) -- (-0.7,0.1) node[midway,xshift=9pt,llblue] {};
\draw[decorate,line width=0.6pt,decoration={brace,amplitude=4pt},fontc] ((1.44,0.1) -- (1.44,-0.9) node[midway,xshift=9pt,llblue] {};

\draw[fontc] (0,0) -- (-0.36,-0.4);
\draw[fontc] (0,0) -- (0.36,-0.4);
\draw[fontc] (-0.36,-0.4) -- (-0.54,-0.8);
\draw[fontc] (-0.36,-0.4) -- (-0.18,-0.8);
\draw[fontc] (0.36,-0.4) -- (0.18,-0.8);
\draw[fontc] (0.36,-0.4) -- (0.54,-0.8);

\gnode{0,0}{2}{llblue}{south}{\scalebox{0.5}{$(1,1)$}};
\gnode{-0.36,-0.4}{2}{llblue}{north}{};
\gnode{0.36,-0.4}{2}{llblue}{north}{};
\gnode{-0.54,-0.8}{2}{llblue}{north}{\scalebox{0.5}{$(1,1)$}};
\gnode{-0.18,-0.8}{2}{llblue}{north}{\scalebox{0.5}{$(1,1)$}};
\gnode{0.18,-0.8}{2}{llblue}{north}{\scalebox{0.5}{$(1,0)$}};
\gnode{0.54,-0.8}{2}{llblue}{north}{\scalebox{0.5}{$(0,1)$}};

\node[fontc] at (-0.43,-0.25) {\scalebox{0.5}{\tiny $(0,1)$}};
\node[fontc] at (0.43,-0.25) {\scalebox{0.5}{\tiny $(1,0)$}};

\node[fontc] at (0.76,-0.9) {\scalebox{1.2}{\textnormal{,}}};

\draw[fontc] (1,0) -- (1,-0.4);
\draw[fontc] (1,-0.8) -- (1,-0.4);

\gnode{1,0}{2}{llblue}{south}{\scalebox{0.5}{$(0,1)$}};
\gnode{1,-0.4}{2}{llblue}{west}{\scalebox{0.5}{$(1,0)$}};
\gnode{1,-0.8}{2}{llblue}{west}{\scalebox{0.5}{$(0,1)$}};

\draw[decorate,line width=0.6pt,decoration={brace,amplitude=2pt},fontc] ((1.1,-1.07) -- (-0.64,-1.07) node[midway,xshift=9pt,llblue] {};
\node[fontc] at (0.2,-1.3) {\scalebox{0.5}{$\cF^{(1)}(G_2,S_2)$}};

\node[fontc] at (1.67,-0.4) {\scalebox{3}{\textnormal{)}}};

\end{scope}
\end{scope}

\draw[fontc,-stealth] (4.1,-0.4) -- (4.9,-0.4);
\node[fontc] at (4.5,-0.22) {\scalebox{0.8}{$\rho$}};

\begin{scope}[shift={(6.35,0)}]
\begin{scope}[shift={(0,0)}]

\node[fontc] at (-1,-0.4) {\scalebox{3}{\textnormal{(}}};

\draw[decorate,line width=0.6pt,decoration={brace,amplitude=4pt},fontc] ((-0.7,-0.9) -- (-0.7,0.1) node[midway,xshift=9pt,llblue] {};

\draw[fontc] (0,0) -- (-0.36,-0.4);
\draw[fontc] (0,0) -- (0.36,-0.4);
\draw[fontc] (-0.36,-0.4) -- (-0.54,-0.8);
\draw[fontc] (-0.36,-0.4) -- (-0.18,-0.8);
\draw[fontc] (0.36,-0.4) -- (0.18,-0.8);
\draw[fontc] (0.36,-0.4) -- (0.54,-0.8);

\gnode{0,0}{2}{llblue}{south}{\scalebox{0.5}{$(1,0)$}};
\gnode{-0.36,-0.4}{2}{llblue}{north}{};
\gnode{0.36,-0.4}{2}{llblue}{north}{};
\gnode{-0.54,-0.8}{2}{llred}{north}{};
\gnode{-0.18,-0.8}{2}{llred}{north}{};
\gnode{0.18,-0.8}{2}{llblue}{north}{\scalebox{0.5}{$(1,1)$}};
\gnode{0.54,-0.8}{2}{llred}{north}{};

\node[fontc] at (-0.43,-0.25) {\scalebox{0.5}{\tiny $(0,1)$}};
\node[fontc] at (0.43,-0.25) {\scalebox{0.5}{\tiny $(1,0)$}};

\node[fontc] at (0.72,-0.9) {\scalebox{1.2}{\textnormal{,}}};

\end{scope}

\begin{scope}[shift={(1.7,0)}]

\draw[decorate,line width=0.6pt,decoration={brace,amplitude=4pt},fontc] ((0.7,0.1) -- (0.7,-0.9) node[midway,xshift=9pt,llblue] {};

\draw[fontc] (0,0) -- (-0.36,-0.4);
\draw[fontc] (0,0) -- (0.36,-0.4);
\draw[fontc] (-0.36,-0.4) -- (-0.54,-0.8);
\draw[fontc] (-0.36,-0.4) -- (-0.18,-0.8);
\draw[fontc] (0.36,-0.4) -- (0.18,-0.8);
\draw[fontc] (0.36,-0.4) -- (0.54,-0.8);

\gnode{0,0}{2}{llred}{south}{};
\gnode{-0.36,-0.4}{2}{llred}{north}{};
\gnode{0.36,-0.4}{2}{llred}{north}{};
\gnode{-0.54,-0.8}{2}{llred}{north}{};
\gnode{-0.18,-0.8}{2}{llred}{north}{};
\gnode{0.18,-0.8}{2}{llred}{north}{};
\gnode{0.54,-0.8}{2}{llred}{north}{};

\node[fontc] at (0.9,-0.9) {\scalebox{1.2}{\textnormal{,}}};

\end{scope}

\begin{scope}[shift={(3.65,0)}]

\node[fontc] at (2.65,-0.4) {\scalebox{3}{\textnormal{)}}};

\draw[decorate,line width=0.6pt,decoration={brace,amplitude=4pt},fontc] ((-0.7,-0.9) -- (-0.7,0.1) node[midway,xshift=9pt,llblue] {};

\draw[fontc] (0,0) -- (-0.36,-0.4);
\draw[fontc] (0,0) -- (0.36,-0.4);
\draw[fontc] (-0.36,-0.4) -- (-0.54,-0.8);
\draw[fontc] (-0.36,-0.4) -- (-0.18,-0.8);
\draw[fontc] (0.36,-0.4) -- (0.18,-0.8);
\draw[fontc] (0.36,-0.4) -- (0.54,-0.8);

\gnode{0,0}{2}{llblue}{south}{\scalebox{0.5}{$(1,1)$}};
\gnode{-0.36,-0.4}{2}{llblue}{north}{};
\gnode{0.36,-0.4}{2}{llblue}{north}{};
\gnode{-0.54,-0.8}{2}{llblue}{north}{\scalebox{0.5}{$(1,1)$}};
\gnode{-0.18,-0.8}{2}{llblue}{north}{\scalebox{0.5}{$(1,1)$}};
\gnode{0.18,-0.8}{2}{llblue}{north}{\scalebox{0.5}{$(1,0)$}};
\gnode{0.54,-0.8}{2}{llblue}{north}{\scalebox{0.5}{$(0,1)$}};

\node[fontc] at (-0.43,-0.25) {\scalebox{0.5}{\tiny $(0,1)$}};
\node[fontc] at (0.43,-0.25) {\scalebox{0.5}{\tiny $(1,0)$}};

\node[fontc] at (0.78,-0.9) {\scalebox{1.2}{\textnormal{,}}};

\end{scope}

\begin{scope}[shift={(5.3,0)}]

\draw[decorate,line width=0.6pt,decoration={brace,amplitude=4pt},fontc] ((0.7,0.1) -- (0.7,-0.9) node[midway,xshift=9pt,llblue] {};

\draw[fontc] (0,0) -- (-0.36,-0.4);
\draw[fontc] (0,0) -- (0.36,-0.4);
\draw[fontc] (-0.36,-0.4) -- (-0.54,-0.8);
\draw[fontc] (-0.36,-0.4) -- (-0.18,-0.8);
\draw[fontc] (0.36,-0.4) -- (0.18,-0.8);
\draw[fontc] (0.36,-0.4) -- (0.54,-0.8);

\gnode{0,0}{2}{llblue}{south}{\scalebox{0.5}{$(0,1)$}};
\gnode{-0.36,-0.4}{2}{llred}{north}{};
\gnode{0.36,-0.4}{2}{llblue}{north}{};
\gnode{-0.54,-0.8}{2}{llred}{north}{};
\gnode{-0.18,-0.8}{2}{llred}{north}{};
\gnode{0.18,-0.8}{2}{llblue}{north}{\scalebox{0.5}{$(0,1)$}};
\gnode{0.54,-0.8}{2}{llred}{north}{};

\node[fontc] at (0.43,-0.25) {\scalebox{0.5}{\tiny $(1,0)$}};

\end{scope}
\end{scope}

\end{tikzpicture}

%% file: fig_overview.tex

\begin{tikzpicture}[transform shape, scale=1]

\definecolor{lgreen}{HTML}{4DA84D}
\definecolor{fontc}{HTML}{403E30}
\definecolor{llred}{HTML}{FF7A87}
\definecolor{llblue}{HTML}{7EAFCC}
\definecolor{lviolet}{HTML}{756BB1}
\definecolor{lorange}{HTML}{FF7F0E}

\newcommand{\gnode}[5]{%
    \draw[#3, fill=#3!20] (#1) circle (#2pt);
    \node[anchor=#4,fontc] at (#1) {\tiny #5};
}

\coordinate (c1) at (0,0);
\coordinate (c2) at (0.3,-0.2);
\coordinate (c3) at (0.7,-0.5);
\coordinate (c4) at (0.7,-1.02);

\draw[fill=llblue!15,draw=none] (c1) circle [radius=0.5];
\draw[fill=llblue!15,draw=none] (c2) circle [radius=0.5];
\draw[fill=llblue!15,draw=none] (c4) circle [radius=0.8];

\draw[decorate,decoration={brace,amplitude=2pt,post length=0.3cm},llblue] (0.65,-1.14) -- (0.19,-0.13) node[pos=0.26,xshift=-9pt,llblue] {\footnotesize \scalebox{0.5}{$\leq 3c$}};

\begin{scope}[shift={(0.049,0.1)},rotate=0]
\draw[draw=none,fill=llblue!15,rotate=25] (0.2,0.2) rectangle (0,-1.2);
\end{scope}

\draw[draw=llblue, rounded corners=1pt] (0.32,-0.5) -- (0.19,-0.2) -- (0.21,-0.13);

\draw[draw=llblue,dashed, fill=none,dash pattern=on 2pt off 1.5pt] (c1) circle [radius=0.5];
\draw[draw=llblue,dashed, fill=none,dash pattern=on 2pt off 1.5pt] (c2) circle [radius=0.5];
\draw[draw=llblue,dashed, fill=none,dash pattern=on 2pt off 1.5pt] (c4) circle [radius=0.8];

\gnode{c1}{2}{lgreen}{north}{};
\gnode{c2}{2}{lviolet}{north}{};
\gnode{c3}{2}{lorange}{north}{};
\gnode{c4}{2}{llred}{north}{};

\draw[decorate,decoration={brace,amplitude=2pt},llblue] (0.78,-0.42) -- (0.78,-1.1) node[midway,xshift=9pt,llblue] {\footnotesize \scalebox{0.5}{$\leq 2c$}};

\node[fontc] at (1.2,0.5) {\scalebox{0.6}{$(\Rb^d, \|\cdot\|_2)$}};

\begin{scope}[shift={(0.15,0)}]
\draw[fontc,-stealth] (-2.5,0) to[bend left=20] (-1,0);
\node[fontc] at (-1.75,0.525) {\scalebox{0.6}{1-Layer}};
\node[fontc] at (-1.75,0.3) {\scalebox{0.6}{MPNN}};
\end{scope}

\begin{scope}[shift={(-3.1,0.3)}]

\node[fontc] at (-1.1,0.2) {\scalebox{0.6}{$(\cG,\mathrm{UD}_{\cT,\cV}^{(1)})$}};

\draw[draw=lviolet,dashed, fill=lviolet!10,dash pattern=on 1.5pt off 1pt,rounded corners=3pt] (-0.35,0.1) rectangle (0.35,-0.1);

\draw[draw=lorange,dashed, fill=lorange!10,dash pattern=on 1.5pt off 1pt,rounded corners=3pt] (-0.1,-0.35) rectangle (0.1,-0.95);

\draw[draw=llred,dashed, fill=llred!10,dash pattern=on 1.5pt off 1pt,rounded corners=3pt] (0.25,-1.25) circle[radius=3pt];

\draw[draw=lgreen,fill=lgreen!10,dash pattern=on 1.5pt off 1pt] (-0.13,-1.3) to[rounded corners=1.5pt] (-0.18,-1.15) to[rounded corners=2pt] (-0.31,-1.15) to[rounded corners=2pt] (-0.37,-1.27) to[rounded corners=7pt] (-0.23,-1.73) to[rounded corners=7pt] (0.25,-1.9) to[rounded corners=7pt] (0.73,-1.73) to[rounded corners=2pt] (0.87,-1.27) to[rounded corners=1.5pt] (0.81,-1.15) to[rounded corners=1.5pt] (0.69,-1.15) to[rounded corners=2pt](0.64,-1.25) to[rounded corners=4pt]  (0.57,-1.55) to[rounded corners=4pt]  (0.25,-1.685) to[rounded corners=4pt]  (-0.05,-1.55) -- cycle;

\draw[fontc] (-0.25,0) -- (0.25,0);
\draw[fontc] (-0.25,0) -- (0,-0.5);
\draw[fontc] (0,-0.45) -- (0.25,0);
\draw[fontc] (0,-0.45) -- (0,-0.85);
\draw[fontc] (0.25,-1.25) -- (0,-0.85);
\draw[fontc] (-0.25,-1.25) -- (0,-0.85);
\draw[fontc] (0.25,-1.25) -- (0.75,-1.25);
\draw[fontc] (0.25,-1.25) -- (0.62,-1.6);
\draw[fontc] (0.25,-1.25) -- (0.25,-1.75);
\draw[fontc] (0.25,-1.25) -- (-0.1,-1.6);

\gnode{-0.25,0}{2}{lviolet}{north}{};
\gnode{0.25,0}{2}{lviolet}{north}{};
\gnode{0,-0.45}{2}{lorange}{north}{};
\gnode{0,-0.85}{2}{lorange}{north}{};
\gnode{0.25,-1.25}{2}{llred}{north}{};
\gnode{0.75,-1.25}{2}{lgreen}{north}{};
\gnode{0.62,-1.6}{2}{lgreen}{north}{};
\gnode{0.25,-1.75}{2}{lgreen}{north}{};
\gnode{-0.1,-1.6}{2}{lgreen}{north}{};
\gnode{-0.25,-1.25}{2}{lgreen}{north}{};

\end{scope}

\begin{scope}[shift={(-3.6,-1.9)}]

\begin{scope}[shift={(0.37,0)}]
\draw[lviolet] (0,0) -- (-0.1,-0.15);
\draw[lviolet] (0,0) -- (0.1,-0.15);
\gnode{0,0}{1}{lviolet}{north}{};
\gnode{-0.1,-0.15}{1}{lviolet}{north}{};
\gnode{0.1,-0.15}{1}{lviolet}{north}{};

\node[fontc] at (0.2,-0.2) {\scalebox{0.65}{,}};
\end{scope}

\begin{scope}[shift={(-0.35,0)}]
\draw[lgreen] (0.4,0.05) -- (0.4,-0.15);
\gnode{0.4,0.05}{1}{lgreen}{north}{};
\gnode{0.4,-0.15}{1}{lgreen}{north}{};

\node[fontc] at (0.5,-0.2) {\scalebox{0.65}{,}};
\end{scope}

\begin{scope}[shift={(0.53,0)}]
\draw[llred] (0.91,0.05) -- (0.65,-0.15);
\draw[llred] (0.91,0.05) -- (0.78,-0.15);
\draw[llred] (0.91,0.05) -- (0.91,-0.15);
\draw[llred] (0.91,0.05) -- (1.04,-0.15);
\draw[llred] (0.91,0.05) -- (1.17,-0.15);
\gnode{0.65,-0.15}{1}{llred}{north}{};
\gnode{0.78,-0.15}{1}{llred}{north}{};
\gnode{0.91,-0.15}{1}{llred}{north}{};
\gnode{1.04,-0.15}{1}{llred}{north}{};
\gnode{1.17,-0.15}{1}{llred}{north}{};
\gnode{0.91,0.05}{1}{llred}{north}{};

\end{scope}

\begin{scope}[shift={(-0.72,0)}]
\draw[lorange] (1.55,0.05) -- (1.42,-0.15);
\draw[lorange] (1.55,0.05) -- (1.55,-0.15);
\draw[lorange] (1.55,0.05) -- (1.68,-0.15);
\gnode{1.42,-0.15}{1}{lorange}{north}{};
\gnode{1.55,-0.15}{1}{lorange}{north}{};
\gnode{1.68,-0.15}{1}{lorange}{north}{};
\gnode{1.55,0.05}{1}{lorange}{north}{};
\node[fontc] at (1.77,-0.2) {\scalebox{0.65}{,}};
\end{scope}
\end{scope}

\end{tikzpicture}